\documentclass[nohyperref]{article}

\usepackage{microtype}
\usepackage{graphicx}
\usepackage{subcaption}
\usepackage{booktabs} 

\usepackage{hyperref}
\usepackage{wrapfig}

\usepackage[accepted]{icml2023}

\usepackage{amsmath}
\usepackage{amssymb}
\usepackage{mathtools}
\usepackage{amsthm}
\usepackage{kbordermatrix}
\renewcommand{\kbldelim}{(}
\renewcommand{\kbrdelim}{)}

\usepackage[capitalize,noabbrev,nameinlink]{cleveref}
\crefname{assumption}{assumption}{Assumptions}
\theoremstyle{plain}
\newtheorem{theorem}{Theorem}[section]
\newtheorem{proposition}[theorem]{Proposition}
\newtheorem{lemma}[theorem]{Lemma}
\newtheorem{corollary}[theorem]{Corollary}
\theoremstyle{definition}
\newtheorem{definition}[theorem]{Definition}
\newtheorem{assumption}[theorem]{Assumption}
\theoremstyle{remark}
\newtheorem{remark}[theorem]{Remark}

\usepackage[textsize=tiny]{todonotes}

\usepackage{symbol_maosheng}

\begin{document}

\twocolumn[
\icmltitle{Convolutional Learning on Simplicial Complexes }



\icmlsetsymbol{equal}{*}

\begin{icmlauthorlist}

\icmlauthor{Maosheng Yang}{yyy}
\icmlauthor{Elvin Isufi}{yyy}

\end{icmlauthorlist}

\icmlaffiliation{yyy}{Department of Intelligient Systems, Delft University of Technology, Delft, The Netherlands}

\icmlcorrespondingauthor{Maosheng Yang}{m.yang-2@tudelft.nl}
\icmlcorrespondingauthor{Elvin Isufi}{e.isufi-1@tudelft.nl}

\icmlkeywords{Machine Learning, ICML}

\vskip 0.3in
]


\printAffiliations{}


\begin{abstract}
We propose a simplicial complex convolutional neural network (SCCNN) to learn data representations on simplicial complexes. 
It performs convolutions based on the multi-hop simplicial adjacencies via common faces and cofaces independently and captures the inter-simplicial couplings, generalizing state-of-the-art. 
Upon studying symmetries of the simplicial domain and the data space, it is shown to be permutation and orientation equivariant, thus, incorporating such inductive biases. 
Based on the Hodge theory, we perform a spectral analysis to understand how SCCNNs regulate data in different frequencies, showing that the convolutions via faces and cofaces operate in two orthogonal data spaces.
Lastly, we study the stability of SCCNNs to domain deformations and examine the effects of various factors.
Empirical results show the benefits of higher-order convolutions and inter-simplicial couplings in simplex prediction and trajectory prediction.

\end{abstract}

\section{Introduction}\label{sec:introduction}
Graphs are commonly used to represent the support of networked data as nodes and capture their pairwise relations as edges. 
Graph neural networks (GNNs) have emerged as a learning model that leverages this topology information as an inductive bias \citep{battaglia2018relational,bronstein2021geometric}. However, this bias can be erroneous when the topology structure of data involves with polyadic or higher-order relations, which often arises in real-world problems.  
For example, in social networks, people often interact in social groups, not just in pairs \citep{newman2002random}. In gene regularoty networks, a collection of molecular regulators interact with each other \citep{masoomy2021topological}. In coauthorship networks, collaborations form between several authors rather than just two \citep{benson2018simplicial,bick2021higher}.

Moreover, GNNs are often used to learn representations from data defined on nodes  \citep{kipf2017semi, defferrard2017,gilmer2017neural}. However, we also have data defined on higher-order structures in a network. 
For example, water flows in a water distribution system \citep{money2022online} and traffic flow in a road network \citep{jia2019graph}, such flow-type data are naturally supported on edges of a network. In coauthorship networks, data supported on a multi-set of $k$ nodes can be the frequency or citation of the collaboration between $k$ people \citep{benson2018simplicial}.

Capturing the coupling between the data and these higher-order network structures is key to overcome the limitations of GNNs.
As a higher-order network model, simplicial complexes (SCs) support an entity of multiple elements as a simplex and the relations between simplices can be mediated through their common faces and cofaces, referred to as lower and upper (simplicial) adjacencies.
In analogy to graph Laplacians, Hodge Laplacians provide an algebraic representation of an SC to encode such adjacencies. This allows for a principled extension of processing and learning techniques from graphs to SCs. 
For example, \citet{barbarossa2020} proposed a spectral signal processing framework in SCs, followed by simplicial convolutional filters (SCFs) \citep{yang2021finite, yang2022simplicial}. 
Neural networks on SCs include, among others, \citet{ebli2020simplicial,roddenberry2021principled,yang2021simplicial,bodnar2021weisfeiler,bunch2020simplicial}. But they either focus on simplices of the same order, not exploiting the couplings between different orders or apply a message passing scheme based on direct simplicial adjacencies. 

With a comprehensive framework to capture both higher-order simplicial adjacencies and inter-simplicial couplings, we conduct a convolution-based study for learning on SCs: 

1) \emph{Simplicial complex convolutional neural network:} we propose an SCCNN to propagate information across simplices of the same order via lower and upper adjacencies independently in a multi-hop way while leveraging the inter-simplicial couplings. It generalizes state-of-the-art and admits intra- and extended inter-simplicial localities with a linear computational complexity. 

2) \emph{Symmetries:} Based on group theory, we show that there exhibit a permutation symmetry in SCs and an orientation symmetry in the SC data. SCCNNs can be built equivariant to both symmetries to incorporate such inductive biases.

3) \emph{Spectral analysis:} Based on tools from \citet{barbarossa2020,yang2022simplicial}, we study how each component of the SCCNN regulates the data from the spectral perspective. This analysis generalizes to state-of-the-art.  


4) \emph{Stability analysis:} We prove that SCCNNs are stable to domain deformations when the convolutional filters are integral Lipschitz and show how the inter-simplicial couplings propagate the deformations across the SC.

\section{Background} \label{sec:background}
We briefly introduce the SC and its algebraic representation, together with the data defined on SCs.
 
\textbf{Simplicial Complex.} 
Given a finite set of vertices $\ccalV:=\{1,\dots,N_0\}$, a $k$-simplex $s^k$ is a subset of $\ccalV$ with cardinality $k+1$. 
A \emph{face} of $s^k$ is a subset with cardinality $k$. A \emph{coface} of $s^k$ is a $(k+1)$-simplex that has $s^k$ as a face. Nodes, edges and (filled) triangles are geometric realizations of 0-, 1- and 2-simplices. 
An SC $\ccalS$ of order $K$ is a collection of $k$-simplices $s^k$, $k=[K]:=0,\dots,K$, with the inclusion property: $s^{k-1}\in\ccalS$ if $s^{k-1}\subset s^k$ for $s^k\in\ccalS$, e.g., \cref{fig:sc_example}. A graph is also an SC of order one including nodes and edges. 
We collect the $k$-simplices in $\ccalS$ in set $\ccalS^k=\{s_i^k \vert i=1,\dots,N_k \}$ with $N_k=|\ccalS^k|$, therefore $\ccalS = \bigcup_{t=0}^{K} \ccalS^k$. 

To facilitate computations, an orientation of a simplex is chosen as an ordering of its vertices, which is an equivalence class that two orderings are equivalent if they differ by an even permutation; otherwise anti-aligned \citep{munkres2018elements,lim2015hodge}. We fix an orientation for a simplex according to the lexicographical ordering of its vertices, $s^k = [1,\dots,k+1]$, e.g., a triangle $s^2=\{i,j,k\}$ is oriented as $[i,j,k]$ with $i<j<k$ and a node has a trivial orientation.  

\textbf{Simplicial Adjacency.}
For $s_i^k$, we define its \emph{lower (upper) neighborhood} $\ccalN^k_{i,{\rmd}}$ ($\ccalN^k_{i,{\rmu}}$) as the set of $k$-simplices which share a common face (coface) with it. If $s_j^k\in \ccalN_{i,\rmd}^k (\ccalN_{i,\rmu}^k)$, we say $s_j^k$ is \emph{lower (upper) adjacent} to $s_i^k$. In \cref{fig:sc_example}, we have $\ccalN_{1,\rmd}^1=\{e_2,e_3,e_4,e_5\}$ and $\ccalN_{1,\rmu}^1=\{e_2,e_4\}$ for $e_1$. 

\textbf{Algebraic Representations.}
We use incidence matrices $\bbB_k, k=[K]$ to describe the incidence relations in an SC, where $\bbB_1$ and $\bbB_2$ are the node-to-edge and edge-to-triangle incidence matrices, respectively. Note that $\bbB_0$ is not defined. See \cref{app:background} for those of SC in \cref{fig:sc_example}.
By definition, we have $\bbB_k \bbB_{k+1} = \mathbf{0}$ \citep{lim2015hodge}.  

In an SC of order $K$, the Hodge Laplacians are defined as 
\begin{equation*}
    \bbL_k = \bbB_k^\top\bbB_k + \bbB_{k+1}\bbB_{k+1}^\top, k = [K]
\end{equation*}
with the \emph{lower Laplacian} $\bbL_{k,\rmd} = \bbB_k^\top \bbB_k$ and the \emph{upper Laplacian} $\bbL_{k,\rmu} = \bbB_{k+1}\bbB_{k+1}^\top$, and the graph Laplacian $\bbL_0 = \bbB_1\bbB_1^\top$ and $\bbL_K = \bbB_{K}^\top\bbB_K$. Matrices $\bbL_{k,\rmd}$ and $\bbL_{k,\rmu}$ encode the lower and upper adjacencies of $k$-simplices, respectively. In particular, $\bbL_{1,\rmd}$ and $\bbL_{1,\rmu}$ encode the edge-to-edge adjacencies through nodes and triangles, respectively.

\textbf{Simplicial Signals.} 
In an SC, we define $k$\emph{-simplicial signals (or data features)} $\bbx_k=[x_{k,1},\dots,x_{k,{N_k}}]^\top, k = [K]$ by an \emph{alternating} map $f_k:\ccalS^k\to\ccalX^{N_k}$ which assigns a signal $x_{k,i}$ to the $i$th simplex $s^k_i$. The alternating map restricts that if the orientation of the simplex is anti-aligned with the reference orientation, denoted by $\overline{s}_i^k = -s_i^k$, then the sign of the signal value will be changed, $f_k(\overline{s}_i^k) = -f_k(s_i^k)$. If the signal value $x_i^k$ is negative, then the signal is anti-aligned with the reference \citep{lim2015hodge,schaub2021}. An $F$-feature simplicial signal $\bbX_k \in \setR^{N_k \times F}$ can be defined. 

\begin{figure*}[t]
  \centering
  \begin{subfigure}{0.246\linewidth}
    \includegraphics[width=\linewidth]{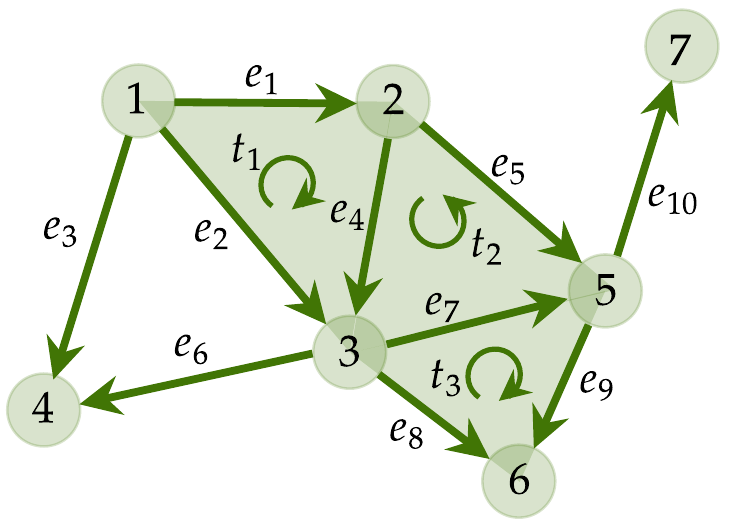}
    \caption{SC example}
    \label{fig:sc_example}
  \end{subfigure}
  \begin{subfigure}{0.246\linewidth}
    \includegraphics[width=\linewidth]{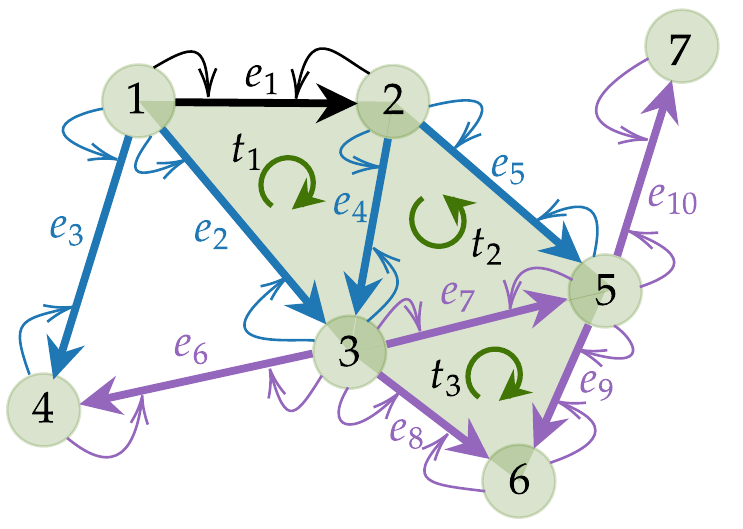}
    \caption{Lower edge convolutions}
    \label{fig:lower_convolution}
  \end{subfigure}
  \begin{subfigure}{0.246\linewidth}
    \includegraphics[width=\linewidth]{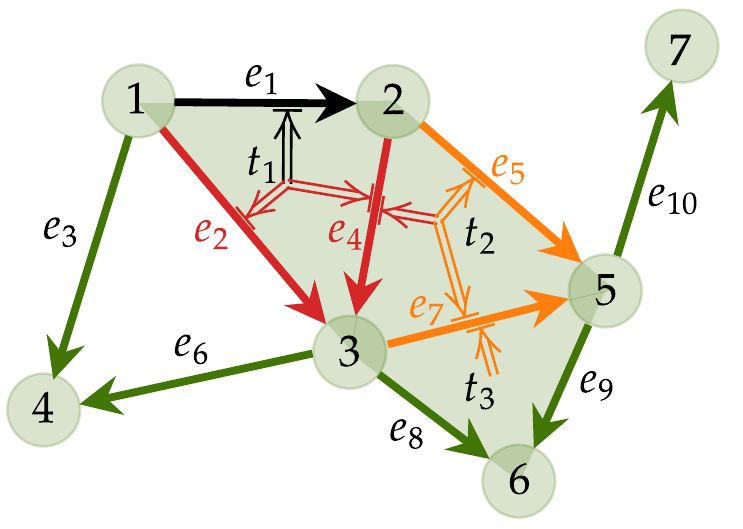}
    \caption{Upper edge convolutions}
    \label{fig:upper_convolution}
  \end{subfigure}
  \begin{subfigure}{0.246\linewidth}
    \includegraphics[width=\linewidth]{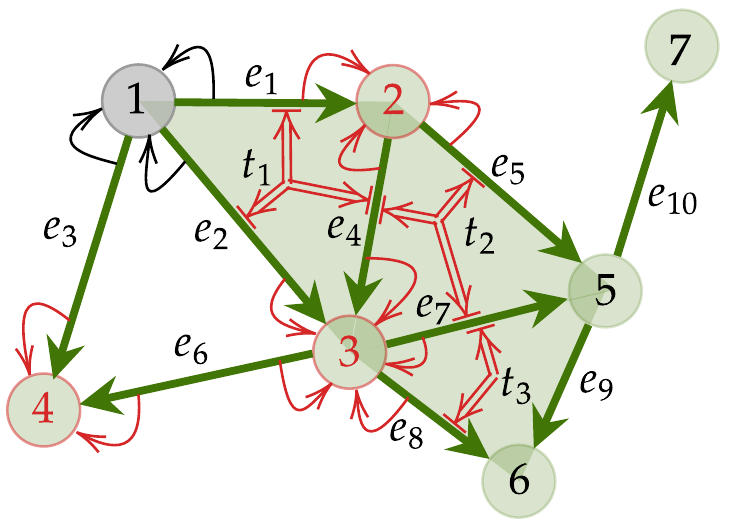}
    \caption{Inter-simplicial locality}
    \label{fig:extended_simplicial_locality}
  \end{subfigure}
  \caption{(a) An SC where arrows indicate the reference orientations of edges and triangles. 2-simplices are (filled) triangles shaded in green and open triangle $\{1,3,4\}$ is not in the SC. (b) Lower convolution via $\bbH_1$ and $\bbH_{1,\rmd}$ on edge $e_1$: SCF $\bbH_1$ aggregates the information from its direct lower neighbors (edges in blue) and two-hop lower neighbors (edges in purple) to $e_1$ (in black) if $T_\rmd=2$; and lower SCF $\bbH_{1,\rmd}$ aggregates the projected information from nodes to edges likewise (denoted by the arrows in blue and purple from nodes to edges). (c) Upper convolution via $\bbH_1$ and $\bbH_{1,\rmu}$ on $e_1$: $\bbH_1$ aggregates the information from direct upper neighbors (edges in red) and two-hop upper neighbors (edges in orange) to $e_1$ (in black); and upper SCF $\bbH_{1,\rmu}$ aggregates the projected information from triangles to edges likewise (denoted by double arrows in red and orange from trianlge centers to edges). (d) Node 1 (in black) contains information from its neighbors $\{2,3,4\}$ (nodes in red), and projected information from edges which contribute to these neighbors (denoted by arrows in red from edges to nodes), and from triangles $\{t_1,t_2,t_3\}$ which contribute to those edges (denoted by double arrows in red from triangle centers to edges). This interaction is the coupling between the intra- and the extended inter-simplicial locality.}
  \vskip -0.1in 
\end{figure*}

\section{SCCNNs}\label{sec:sccnns}
We introduce the SCCNN to learn from data defined on SCs. Then, we discuss its intra- and inter-simplicial localities, followed by its complexity, and related works. 

An $L$-layer SCCNN defined in an SC $\ccalS$ of order $K$ computes the output $\bbx^l_{k}$ at layer $l$ as a nonlinear function of the outputs $\bbx_{k-1}^{l-1}, \bbx^{l-1}_k$ and $\bbx^{l-1}_{k+1}$ at the previous layer $l-1$
\begin{equation*}
    \text{SCCNN}_k^l:\{\bbx_{k-1}^{l-1}, \bbx_k^{l-1}, \bbx_{k+1}^{l-1}\} \to \bbx_k^l,
\end{equation*}
for $k=[K]$ and $l=1,\dots,L$, and admits a detailed form 
\begin{equation}\label{eq.sccnn_layer}
    \bbx_k^{l} = \sigma(\bbH_{k,\rmd}^{l} \bbx_{k,\rmd}^{l-1} + \bbH_{k}^{l}\bbx_k^{l-1} + \bbH_{k,\rmu}^{l} \bbx_{k,\rmu}^{l-1} )
\end{equation}  
with $\bbx_{k,\rmd}^{l-1} = \bbB_k^\top \bbx_{k-1}^{l-1}$ and $\bbx_{k,\rmu}^{l-1} = \bbB_{k+1} \bbx_{k+1}^{l-1}$, which can be understood as follows:

1) The previous output $\bbx_{k}^{l-1}$ is passed through a simplicial convolution filter (SCF) $\bbH^{l}_{k}$ \citep{yang2022simplicial}, given by 
  \begin{equation*} 
       \bbH^{l}_{k} := \bbH^{l}_{k}(\bbL_{k,\rmd},\bbL_{k,\rmu})= \sum_{t=0}^{T_{\rmd}}w^l_{k,\rmd,t}\bbL_{k,\rmd}^t + \sum_{t=0}^{T_\rmu} w^l_{k,\rmu,t} \bbL_{k,\rmu}^t,
  \end{equation*}
which is a sum of two matrix polynomials of $\bbL_{k,\rmd}$ and $\bbL_{k,\rmu}$ with trainable filter coefficients $\{w_{k,\rmd,t}^{l}, w_{k,\rmu,t}^{l}\}$ and filter orders $T_\rmd, T_\rmu$. Operator $\bbH_{k}^{l}$ performs simplicial convolutions relying on the lower and upper adjacencies independently.

2) $\bbx_{k,\rmd}^{l-1} $ and $\bbx_{k,\rmu}^{l-1} $ are the lower and upper projections from the lower and upper adjacent simplices (i.e., faces and cofaces) to $k$-simplices via incidence structures $\bbB_k^\top$ and $\bbB_{k+1}$, respectively. $\bbx_{0,\rmd}^{l-1}$ and $\bbx_{K,\rmu}^{l-1}$ are not defined.
 
3) The lower projection $\bbx_{k,\rmd}^{l-1}$ is passed through another SCF, but it reduces to $\bbH_{k,\rmd}^l:=\sum_{t=0}^{T_{\rmd}}w^{\prime l}_{k,\rmd,t}\bbL_{k,\rmd}^t$ since $\bbL_{k,\rmu} \bbB_{k}^\top = \mathbf{0}$. That is, the lower projection cannot propagate via the upper adjacency. Likewise, the upper projection $\bbx_{k,\rmu}^{l-1}$ is passed through an upper SCF $\bbH_{k,\rmu}^l:=\sum_{t=0}^{T_\rmu} w^{\prime l}_{k,\rmu,t} \bbL_{k,\rmu}^t$, only accounting for the upper adjacency. 

4) The sum of the three SCF outputs is passed by an elementwise nonlinearity $\sigma(\cdot)$. 
A multi-feature variant of SCCNNs can also be defined (see \cref{app:b}).

\textbf{Localities.} Consider the output of an SCF on a $k$-simplicial signal, i.e., $\bbH_k\bbx_k$ (with layer index $l$ omitted). We have \emph{simplicial shiftings} $\bbL_{k,\rmd} \bbx_k $ and $\bbL_{k,\rmu}\bbx_k$ on simplex $s_i^k$ as 
\begin{equation} \label{eq.shifting_entry}
  \begin{aligned}
    [\bbL_{k,\rmd} \bbx_k]_i= \textstyle \sum_{j\in\ccalN^k_{i,\rmd}\cup \{i\}}[\bbL_{k,\rmd}]_{ij} [\bbx_k]_j,  \\
    [\bbL_{k,\rmu} \bbx_k]_i= \textstyle \sum_{j\in\ccalN^k_{i,\rmu} \cup \{i\}}[\bbL_{k,\rmu}]_{ij} [\bbx_k]_j,
  \end{aligned}
\end{equation}
where $s_i^k$ aggregates signals from its lower and upper neighbors in $\ccalN_{i,\rmd}^k$ and $\ccalN_{i,\rmu}^k$ based on the corresponding adjacencies. We can compute the $t$-step shifting recursively as $\bbL_{k,\rmd}^t \bbx_k  = \bbL_{\rmd}(\bbL_{k,\rmd}^{t-1}\bbx_k)$, a one-step shifting of the $(t-1)$-shift result; likewise for $\bbL_{k,\rmu}^t\bbx_k$. An SCF linearly combines such multi-step simplicial shiftings based on lower and upper adjacencies. Thus, the output $\bbH_k\bbx_k$ is localized in $T_{\rmd}$-hop lower and $T_{\rmu}$-hop upper $k$-simplicial neighborhoods \citep{yang2022simplicial}. SCCNNs preserve such \emph{intra-simplicial locality} as the elementwise nonlinearity does not alter the information locality, shown in \cref{fig:lower_convolution,fig:upper_convolution}. 

An SCCNN takes the data on $k$- and $(k\pm 1)$-simplices at layer $l-1$ to compute $\bbx^l_k$, causing interactions between $k$-simplices and their (co)faces when all SCFs are identity. 
In turn, $\bbx_{k-1}^{l-1}$ contains information on $(k-2)$-simplices from layer $l-2$. Likewise for $\bbx_{k+1}^{l-1}$, thus, $\bbx^l_k$ also contains information up to $(k\pm 2)$-simplices if $L\geq 2$, because $\bbB_{k}\sigma(\bbB_{k+1})\neq \mathbf{0}$ (see \cref{app:c}). Accordingly, this \emph{inter-simplicial locality} extends to the whole SC if $L\geq K$, unlike linear filters in an SC where the locality happens up to the adjacent simplices \citep{isufi2022convolutional,schaub2021}. This locality is further coupled with the intra-locality through three SCFs such that a node not only interacts with its cofaces (direct edges) and direct triangles including it, but also edges and triangles further hops away which contribute to the neighboring nodes, as shown in \cref{fig:extended_simplicial_locality}.


 
\textbf{Complexity.} For an SCCNN layer, the parameter complexity is of order $\ccalO(T_\rmd+T_\rmu)$. Denote the maximum of the number of neighbors for $k$-simplices by $M_k:=\max\{|\ccalN^k_{i,\rmd}|, |\ccalN^k_{i,\rmu|} \}_{i=1}^{N_k}$. The computational complexity is of order $\ccalO(k(N_k+N_{k+1}) + N_kM_k(T_\rmd+T_\rmu))$, discussed in \cref{app:complexity}, which is linear to the simplex dimensions.

\subsection{Related Works}
The related works of this paper concern the following.
 
\textbf{Signal Processing on SCs.}
Recent works on processing SC signals started on edge flows, which intrinsically follow properties like divergence-free, curl-free or harmonic \citep{jiang2011statistical,schaub2018flow,jia2019graph}. In \citet{barbarossa2020,schaub2021}, a better understanding of simplicial signals was approached via Hodge theory \citep{lim2015hodge}. \citet{yang2021finite,yang2022simplicial} proposed an SCF, providing a spectral analysis of simplicial signals based on the spectrum of Hodge Laplacians. SCF was further extended to a joint filtering of signals on simplices of different orders by \citet{isufi2022convolutional}. These concepts are key to understand the SCCNN spectrally.

\textbf{NNs on SCs.} \citet{roddenberry2019hodgenet} first used edge-Laplacian $\bbL_{1,\rmd}$ to build NNs where edge convolution only considers the lower adjacency. \citet{ebli2020simplicial} built an SNN based on a convolution via Hodge Laplacians, jointly relying on the lower and upper adjacencies. \citet{yang2021simplicial} discussed the limitations of this strategy and proposed separate simplicial convolutions based on the SCF. A one-step simplicial shifting separately by $\bbL_{k,\rmd}$ and $\bbL_{k,\rmu}$ was proposed by \citet{roddenberry2021principled}. An attention scheme was applied to the previous two by \citet{giusti2022simplicial,goh2022simplicial}. 
Information from simplices of adjacent orders was added by \citet{bunch2020simplicial} and \citet{yang2022efficient}. Instead, \citet{bodnar2021weisfeiler} and \citet{hajij2021simplicial} used a message passing scheme to collect such information, in analogy to the graph case \citep{gilmer2017neural}. \citet{chen2022bscnets} combined graph shifting of node features and simplicial shifting of edge features in link predictions. As listed in \cref{tab:related_works} and further discussed in \cref{app:related_works}, most of these solutions can be subsumed into the SCCNNs. 

\textbf{Graph Neural Networks.}
NNs on SCs return to GNNs when the SC is a graph. Most GNNs vary in terms of the graph convolutions, a shift-and-sum operation via graph shift operators such as graph adjacency and Laplacian matrices, e.g., a one-step graph shifting was performed in \citet{kipf2017semi} in contrast to a general graph convolution in \citep{defferrard2017,gama2019convolutional,gama2020graphs}, which can be obtained as $\bbx_0^l$ without the upper projection $\bbx_{0,\rmu}$. 

\begin{table}[t!]
  \caption{SCCNNs generalize several related works.}
  \label{tab:related_works}
  \vskip -0.1in
  \begin{center}
  \begin{small}
  \resizebox{\columnwidth}{!}{
  \begin{tabular}{lr}
  \toprule
  Methods & Parameters (n.d. denotes ``not defined'') \\
  \midrule
  \citet{ebli2020simplicial}  & $w_{k,\rmd,t}^l = w_{k,\rmu,t}^l,\bbH_{k,\rmd}^l, \bbH_{k,\rmu}^{l}$ n.d. \\ 
  \citet{roddenberry2021principled}  & $T_{\rmd} = T_{\rmu}=1, \bbH_{k,\rmd}^l, \bbH_{k,\rmu}^{l}$ n.d.    \\
  \citet{yang2021simplicial}  &  $\bbH_{k,\rmd}^l, \bbH_{k,\rmu}^{l}$ n.d.   \\
  \citet{bunch2020simplicial}  &  $T_{\rmd} = T_{\rmu}=1, \bbH_{k,\rmd}^l = \bbH_{k,\rmu}^{l} = \bbI$ \\
  \citet{bodnar2021weisfeiler} & $T_{\rmd} = T_{\rmu}=1, \bbH_{k,\rmd}^l = \bbH_{k,\rmu}^{l} = \bbI$  \\
  \bottomrule
  \end{tabular}
  }
  \end{small}
  \end{center}
  \vskip -0.2in
\end{table}

\section{Simplicial Complex Symmetry}\label{sec:sc_symmetry}
Machine learning models rely on \emph{symmetries} of the object domain, which are transformations that keep invariant certain object properties. Leveraging such symmetries of the data and the underlying domain imposes inductive biases, allowing the model to learn effective data representations \citep{bronstein2021geometric}. For example, 
GNNs leverage the permutation symmetry group to learn from graphs \citep{hamilton2020graph,ruiz2021graph}. Here we study the symmetries of the SC domain and the simplicial signal space, and show that SCCNNs preserve such symmetries, ultimately, extending the approach of \citet{bronstein2021geometric} to the SCs. 

\textbf{Permutation Symmetry.}
In an SC $\ccalS$, the labeling of the $k$-simplices in $\ccalS^k= \{s_{\frakp_k(1)}^k,\dots,s_{\frakp_k(N_k)}^k\}$ is a permutation $\frakp_k$ of the indices $\{1,\dots,N_k\}$. 
These permutations form a \emph{permutation group} $\frakP_k$ with $N_k!$ elements based on the group axioms (see   \cref{app:d}): they are associative, every permutation has an identity permutation and an inverse, and every two permutations form another permutation. 
A permutation $\frakp_k\in\frakP_k$ can be represented by an orthogonal permutation matrix $\bbP_k\in\{0,1\}^{N_k\times N_k}$ with entry $[\bbP_k]_{ij} = 1$ if $i=\frakp_k(j)$ and $[\bbP_k]_{ij}=0$ otherwise. 

Thus, the labeling of simplices in $\ccalS$ form a set $\{\frakP_k:k=[K]\}$ whose elements are permutation groups. This set can be represented by a set of matrices, 
$
    \setP = \{\setP_k:k=[K]\} 
$
with $\setP_k=\{\bbP_{k,i} \in \{0,1\}^{N_k\times N_k}: \bbP_{k,i}\bb1 = \bb1, \bbP_{k,i}^\top\bb1 = \bb1, i=1,\dots,N_k!\}$ representing the permutation group $\frakP_k$.

We then study how a permutation (labeling) of simplices affects the domain SC and the simplicial signals. 
\begin{proposition}[Permutation Symmetry]\label{prop:prop_permutation_symmetry}
  Consider an SC $\ccalS$ with $\bbB_k$ and $\bbL_k$ for $k=[K]$. Let $\{\bbP_k:k=[K]\}\in\setP$ represent a sequence of permutations $\{\frakp_k\in\frakP_k:k=[K]\}$. 
  Denote the permuted incidence matrices and Hodge Laplacians by $\overline{\bbB}_k$ and $\overline{\bbL}_k$, for $k=[K]$. 
  Then, we have {\rm{i)}} $\overline{\bbB}_k = \bbP_{k-1} \bbB_k \bbP_{k}^\top$ with entries $[\overline{\bbB}_k]_{i^\prime j^\prime} = [\bbB_k]_{ij}$ for $i^\prime=\frakp_{k-1}(i),j^\prime=\frakp_{k}(j)$; {\rm{ii)}} $\overline{\bbL}_k = \bbP_k \bbL_k \bbP_k^\top$ with entries $[\overline{\bbL}_k]_{i^\prime j^\prime} = [\bbL_k]_{ij}$ for $i^\prime=\frakp_k(i),j^\prime=\frakp_{k}(j)$; and {\rm{iii)}} the spectral property of the SC remains equivariant. 
\end{proposition}

See proof in  \cref{proof.prop_permutation_symmetry}. This states that an SC is unaffected by the labeling of simplices, as well as its simplicial adjacencies, illustrated in \cref{fig:permutation_symmetry}. Its algebraic representations remain equivariant to permutations, i.e., they are a rearrangement of the rows and columns of the original ones. The spectral property of algebraic representations remain equivariant as well. Furthermore, the permutation of $k$-simplices does not affect $\bbB_j$ for $j\neq k,k+1$, nor $\bbL_j$ for $j\neq k$. Lastly, a $k$-simplicial signal $\bbx^k$ changes into $\overline{\bbx}_k=\bbP_k\bbx_k$ according to the permutation of $k$-simplices.

\begin{figure}[t!]
  \centering
  \begin{subfigure}{0.495\linewidth}
    \includegraphics[width=\linewidth]{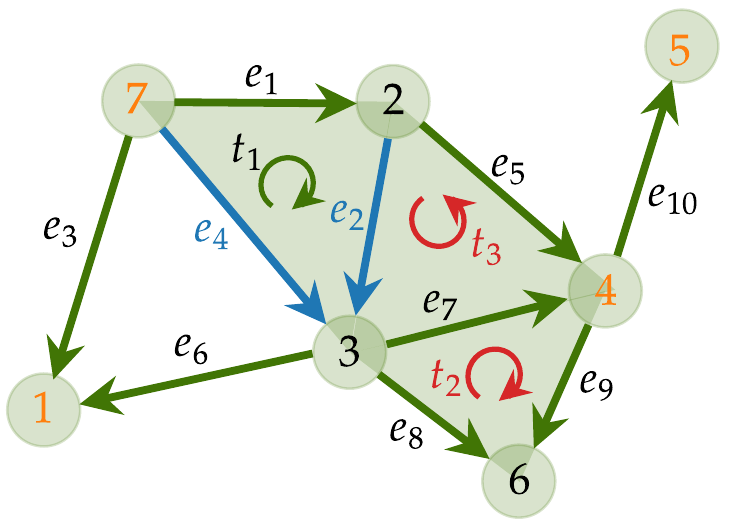}
    \caption{Permutation Symmetry}
    \label{fig:permutation_symmetry}
  \end{subfigure}
  \begin{subfigure}{0.495\linewidth}
    \includegraphics[width=\linewidth]{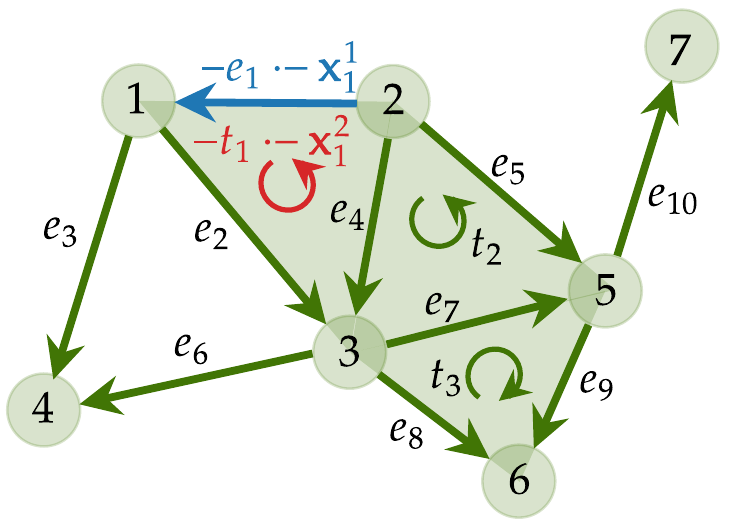}
    \caption{Orientation Symmetry}
    \label{fig:orientation_symmetry}
  \end{subfigure}
  \caption{(a) Simplex relabling (relabled nodes, edges and triangles indicated by orange, blue and red) does not alter the SC in \cref{fig:sc_example}. (b) Reorienting simplices (edges in blue, triangles in red) does not alter the chain (simplicial signal) on them.}
  \vskip -0.1in
\end{figure}

\textbf{Orientation Symmetry.}
In an oriented SC, the equivalence of an orientation defines that a simplex $s^k=[0,\dots,k]$ and its reoriented version $\overline{s}^k=[\pi(0),\dots,\pi(k)]$ have the same orientation if $\pi$ is an even permutation of $\{0,\dots,k\}$; and they are anti-aligned $\overline{s}^k=-s^k$, if $\pi$ is an odd permutation. These two orientations of a $k$-simplex form an \emph{orientation group} $\frakO_{k}=\{\frake,\frako_{-}\}$ with two elements. They can be obtained by a group homomorphism which maps all the even permutations of $\{0,\dots,k\}$ to the identity orientation $\frake$ and all the odd permutations of $[k]$ to  the reverse orientation $\frako_{-}$. Note that $\frakO_0=\{\frake\}$ since a node has one trivial orientation and we assume $k\neq 0$ here.

Thus, in an oriented SC $\ccalS$ we have a set whose elements are orientation groups $\{\frakO_{k,i}:k=[K], i=1,\dots,N_k\}$ with group $\frakO_{k,i}$ for the $i$th $k$-simplex $s_i^k$. The subset $\{\frakO_{k,i}:i=1,\dots,N_k\}$ admits a diagonal matrix representation $\bbD_k\in\{-1,1\}^{N_k\times N_k}$ with $[\bbD_k]_{ii}=1$ if $\frake$ is applied to $s_i^k$, and $[\bbD_k]_{ii}=-1$ if $\frako_-$ is applied. Then, we have a set of orientation matrices for $\ccalS$ as
$
  \setD = \{\bbD_{k}:k=[K]\}.
$  

A simplicial signal $\bbx_k$ by definition remains unchanged w.r.t. the underlying simplices after an orientation change. To see this conveniently, we introduce the \emph{$k$-chain} space $\ccalC_k$ with a chain $c_k=\sum_{i=1}^{N_k}x_{k,i}s_{i}^k$ that is a linear combination of $k$-simplices weighted by the supported signals $x_{k,i}$. With the basis $\{s_i^k:i=1,\dots,N_k\}$, a $k$-chain $c_k$ can be represented by the $k$-simplicial signal vector $\bbx_k = [x_{k,1},\dots,x_{k,N_k}]^\top$. By imposing the alternating property to $\ccalC_k$ that if $s_i^k$ is reversed, the weight $x_{k,i}$ changes its sign, a $k$-chain space $\ccalC_k$ is then isomorphic to the $k$-simplicial signal space $\ccalX_k$ \citep{munkres2018elements}. Unlike permutation groups, orientation groups do not form a symmetry group in an oriented SC but in the chain space as stated by the following proposition. 


\begin{proposition}[Orientation Symmetry]
  Consider an oriented SC $\ccalS$ with $\bbB_k$ and $\bbL_k$ for $k=[K]$.
  Let $\{\bbD_k:k=[K]\}\in\setD$ represent a sequence of orientation changes $\{\frakO^k_i:k=[K],i=1,\dots,N_k\}$ of simplices in $\ccalS$. In the reoriented SC, incidence matrices become $\overline{\bbB}_k= \bbD_k \bbB_k \bbD_{k+1}$ and Hodge Laplacians become $\overline{\bbL}_k= \bbD_k \bbL_k \bbD_k$. Moreover, a $k$-simplicial signal $\bbx_k$ becomes $\overline{\bbx}_k=\bbD_k\bbx_k$ and its underlying $k$-chain remains unchanged.
\end{proposition}


See the proof in \cref{proof.prop_orientation_symmetry}. This states that the incidence relations and the simplicial adjacencies in an oriented SC are altered when the orientations are reversed, whereas the $k$-chain remains invariant to this transformation and the $k$-simplicial signal $\bbx_k$ is equivariant in terms of the basis. 


\textbf{Equivariance of SCCNNs.}
Upon seeing that permutations form a symmetry group in an SC and orientations form a symmetry group in the simplicial signal space, we show that SCCNNs in \eqref{eq.sccnn_layer} are equivariant to the two symmetries.

\begin{proposition}[Permutation Equivariance]\label{prop:permutation_equivariance}
   $\text{SCCNN}_k^l:\{\bbx_{k-1}^{l-1}, \bbx_k^{l-1}, \bbx_{k+1}^{l-1}\} \to \bbx_k^l$ in \eqref{eq.sccnn_layer} is equivalent to permutations. In the $\{\frakP_k:k=[K]\}$-permuted SC, it follows  
  \begin{equation*}
    \{\bbP_{k-1}\bbx_{k-1}^{l-1}, \bbP_{k}\bbx_k^{l-1}, \bbP_{k+1}\bbx_{k+1}^{l-1}\} \to \bbP_{k}\bbx_k^l,
  \end{equation*}
  with matrix representation $\bbP_k$ of $\frakP_k$. Thus, permutations on the SC and the input affect the output in the same way. 
\end{proposition}

\begin{proposition}[Orientation Equivariance]\label{prop:orientation_equivariance}
  $\text{SCCNN}_k^l:\{\bbx_{k-1}^{l-1}, \bbx_k^{l-1}, \bbx_{k+1}^{l-1}\} \to \bbx_k^l$ in \eqref{eq.sccnn_layer} with an odd nonlinearity $\sigma(\cdot)$ is equivariant to orientations. In a reoriented SC by $\{\frakO_{k,i}:k=[K],i=1,\dots,N_k\}$, it follows 
  \begin{equation*}
    \{\bbD_{k-1}\bbx_{k-1}^{l-1}, \bbD_{k}\bbx_k^{l-1}, \bbD_{k+1}\bbx_{k+1}^{l-1}\} \to \bbD_{k}\bbx_k^l,
  \end{equation*}
  with matrix representation $\bbD_k$ of $\{\frakO_{k,i}:i=1,\dots,N_k\}$, i.e., orientations of the SC affect the output in the same way. 
\end{proposition}
\cref{prop:permutation_equivariance,prop:orientation_equivariance} shows that SCCNNs incorporate the inductive biases imposed by the symmetries of the SC and the signal space. If we relabel the SC, the output of an SCCNN on $k$-simplices will be relabled according to the labeling of $k$-simplices and remain unaffected by labeling of $j$-simplices with $j\neq k$. If the orientation of a simplex is reversed, the output of an SCCNN on this simplex changes its sign. The latter however requires an odd nonlinearity, thus, if ReLU or its alternatives are used, the SCCNN will not leverage the orientation symmetry of the data. 

\section{Spectral Analysis}\label{sec:spectral_analysis}
We use the spectrum of Hodge Laplacians and the Hodge decomposition \citep{hodge1989theory} to perform a spectral analysis for the SCCNN, which were also used to analyze SCFs and the NN in \citet{yang2021simplicial}. This analysis also reveals the spectral mechanisms of other NNs in \cref{tab:related_works}.
\begin{theorem}[Hodge Decomposition]\label{thm:hodge decomposition}
  In an SC $\ccalS$ with incidence matrices $\bbB_k$ and Hodge Laplacians $\bbL_k$, we have
  \begin{equation*}
    \ccalX_k = \im(\bbB_k^\top) \oplus \ker(\bbL_k) \oplus \im(\bbB_{k+1})
  \end{equation*}
 for $k=[K]$, where $\oplus$ is direct sum operation. Any $\bbx_k$ can be expressed as a sum of three orthogonal components $\bbx_k = \bbx_{k,\rm{G}} + \bbx_{k,\rm{H}} + \bbx_{k,\rm{C}}$ with $\bbx_{k,\rm{G}}=\bbB_{k}^\top \bbx_{k-1} $ and $\bbx_{k,\rm{C}}= \bbB_{k+1}\bbx_{k+1}$, for some $\bbx_{k-1}$ and $\bbx_{k+1}$, and $\bbL_{k}\bbx_{k,\rm{H}} = \mathbf{0}$.
\end{theorem}
For $k=1$, $\im(\bbB_1^\top)$ is the \emph{gradient} space collecting edge flows as the gradient of some node signal; $\im(\bbB_2)$ is the \emph{curl} space where flows circulate within triangles; and $\ker(\bbL_1)$ contains \emph{harmonic} flows which are divergence-free (zero netflow at nodes) and curl-free (zero circulation in triangles) \citep{barbarossa2020,schaub2021}. 

This decomposition implies that $\bbx_k$ is a sum of $\bbx_{k,\rm{G}}$ via lower incident relations from some $\bbx_{k-1}$, $\bbx_{k,\rm{C}}$ via upper incident relations from some $\bbx_{k+1}$, and $\bbx_{k,\rm{H}}$ which cannot be diffused to other simplices. This motivates the input $\bbx_{k-1}$ and $\bbx_{k+1}$ of an SCCNN layer as they also contain information that contributes to the $k$-simplicial signal space. 

\begin{definition}[Simplicial Fourier Transform]\label{def:sft}
  The SFT of $\bbx_k$ is $\tilde{\bbx}_k = \bbU_k^\top \bbx_k$ where SFT basis $\bbU_k$ is the eigenbasis of $\bbL_k=\bbU_k\bLambda_k\bbU_k^\top$, and the inverse SFT is $\bbx_k = \bbU_k\tilde{\bbx}_k$ \citep{barbarossa2020}.
\end{definition}

\begin{proposition}[\citet{yang2022simplicial}]\label{prop:sft}
  The SFT basis $\bbU_k$ can be found as $\bbU_k = [\bbU_{k,\rm{H}}\,\,\bbU_{k,\rm{G}}\,\,\bbU_{k,\rm{C}}]$ where \newline
  1) $\bbU_{k,\rm{H}}$, associated with $N_{k,\rm{H}}$ zero eigenvalues of $\bbL_k$, spans $\ker(\bbL_k)$, and $\dim(\ker(\bbL_k))=N_{k,\rm{H}}$; \newline
  2) $\bbU_{k,\rm{G}}$, associated with nonzero eigenvalues $\{\lambda_{k,{\rm{G}},i}\}_{i=1}^{N_{k,\rm{G}}}$ of $\bbL_{k,\rmd}$, referred to as gradient frequencies, spans $\im(\bbB_k^\top)$, and $\dim(\im(\bbB_k^\top)) = N_{k,\rm{G}}$; \newline
  3) $\bbU_{k,\rm{C}}$, associated with nonzero eigenvalues $\{\lambda_{k,{\rm{C}},i}\}_{i=1}^{N_{k,\rm{C}}}$ of $\bbL_{k,\rmu}$, referred to as curl frequencies, spans $\im(\bbB_{k+1})$, and $\dim(\im(\bbB_k^\top)) = N_{k,\rm{C}}$. 
\end{proposition}
The eigenvalues of $\bbL_k$ carry two types of simplicial frequencies, which measure the $k$-simplicial signal variations in terms of faces and cofaces. 
For $k=1$, gradient frequencies measure the edge flow ``smoothness'' in terms of nodal variations, i.e., the total divergence. Curl frequencies measure the smoothness in terms of rotational variations, i.e., the total curl. For $k=0$, curl frequencies are the graph frequencies in graph signal processing while gradient frequencies do not exist. We refer to \cref{app:sft} for more details. 

We can now analyze the input of an SCCNN layer in spectral domain. First, the SFT of $\bbx_k$ is given by 
\begin{equation} \label{eq.sft_embedding}
  \tilde{\bbx}_k = [\tilde{\bbx}_{k,\rm{H}}^\top,\tilde{\bbx}_{k,\rm{G}}^\top,\tilde{\bbx}_{k,\rm{C}}^\top ]^\top
\end{equation}
with the \emph{harmonic embedding} $\tilde{\bbx}_{k,\rm{H}} = \bbU_{k,\rm{H}}^\top \bbx_{k} = \bbU_{k,\rm{H}}^\top \bbx_{k,\rm{H}}$ in the zero frequencies, the \emph{gradient embedding} $\tilde{\bbx}_{k,\rm{G}}= \bbU_{k, \rm{G}}^\top \bbx_{k} = \bbU_{k, \rm{G}}^\top \bbx_{k,\rm{G}}$ in the gradient frequencies, and the \emph{curl embedding} $\tilde{\bbx}_{k,\rm{C}} = \bbU_{k,\rm{C}}^\top \bbx_k = \bbU_{k,\rm{C}}^\top \bbx_{k,\rm{C}}$ in the curl frequencies. Second, the lower projection $\bbx_{k,\rmd}\in\im(\bbB_{k}^\top)$ has only a nonzero gradient embedding $\tilde{\bbx}_{k,\rmd} = \bbU_{k,\rm{G}}^\top \bbx_{k,\rmd}$. The upper projection $\bbx_{k,\rmu}\in \im(\bbB_{k+1})$ contains only a nonzero curl embedding $\tilde{\bbx}_{k,\rmu} = \bbU_{k,\rm{C}}^\top \bbx_{k,\rmu}$. 

\begin{corollary}\label{cor:evd_Ld_Lu}
  $\bbL_{k,\rmd}$ and $\bbL_{k,\rmu}$ admit diagonalizations by $\bbU_k$. Thus, the simplicial shifting in \eqref{eq.shifting_entry} can be expressed as  
  \begin{equation}\label{eq.shifting_spectral}
    \begin{aligned}
      \bbL_{k,\rmd}\bbx_k & = \bbU_{k,\rm{G}} (\blambda_{k,\rm{G}} \odot \tilde{\bbx}_{k,\rm{G}}) \in \im(\bbB_k^\top)  \\
      \bbL_{k,\rmu}\bbx_k & = \bbU_{k,\rm{C}} (\blambda_{k,\rm{C}} \odot \tilde{\bbx}_{k,\rm{C}}) \in \im(\bbB_{k+1})
    \end{aligned}
  \end{equation}
  with the Hadamard product $\odot$ and column vectors $\blambda_{k,\rm{G}}$ and $\blambda_{k,\rm{C}}$  collecting gradient and curl frequencies, respectively.
\end{corollary}


See \cref{proof:cor:evd_Ld_Lu} for the proof. \eqref{eq.shifting_spectral}  implies that a lower shifting of $\bbx_k$ results a signal living in the gradient space and an upper one results in the curl space. This limits a linear relation between the output and input in terms of the corresponding frequencies as in \citet{roddenberry2021principled}. 

By diagonalizing an SCF $\bbH_k$ with $\bbU_k$, we can further express the simplicial convolution as 
\begin{equation} \label{eq.freq_response_scf}
  \bbH_k \bbx_k = \bbU_k \widetilde{\bbH}_k \bbU_k^\top \bbx_k = \bbU_k (\tilde{\bbh}_k \odot \tilde{\bbx}_k) \\ 
\end{equation}
where $\widetilde{\bbH}_k = \diag(\tilde{\bbh}_k)$. Here, $\tilde{\bbh}_k = [\tilde{\bbh}_{k,\rm{H}}^\top, \tilde{\bbh}_{k,\rm{G}}^\top, \tilde{\bbh}_{k,\rm{C}}^\top]^\top$ is the \emph{filter frequency response}, given by 
\begin{equation*}
  \begin{cases}
    \emph{harmonic response}: \tilde{\bbh}_{k,\rm{H}}  = (w_{k,\rmd,0} + w_{k,\rmu,0})\mathbf{1},\\ 
    \emph{gradient response}: \tilde{\bbh}_{k,\rm{G}}  = \sum_{t=0}^{T_\rmd} w_{k,\rmd,t} \blambda_{k,\rm{G}}^{\odot t} + w_{k,\rmu,0} \mathbf{1}, \\ 
    \emph{curl response}: \tilde{\bbh}_{k,\rm{C}}  = \sum_{t=0}^{T_\rmu} w_{k,\rmu,t} \blambda_{k,\rm{C}}^{\odot t} + w_{k,\rmd,0} \mathbf{1},
  \end{cases}
\end{equation*}
with $(\cdot)^{\odot t}$ the elementwise $t$th power of a vector. Furthermore, we can express $\tilde{\bbh}_k \odot \tilde{\bbx}_k$ as 
\begin{equation} \label{eq.y_output_sft}
   [(\tilde{\bbh}_{k,\rm{H}} \odot \tilde{\bbx}_{k,\rm{H}})^\top,(\tilde{\bbh}_{k,\rm{G}} \odot \tilde{\bbx}_{k,\rm{G}})^\top,(\tilde{\bbh}_{k,\rm{C}} \odot \tilde{\bbx}_{k,\rm{C}})^\top]^\top .
\end{equation}
Therefore, the simplicial convolution corresponds to a pointwise multiplication of the SFT of a simplicial signal by the filter frequency response in the spectral domain. Specifically, the frequency response $\tilde{\bbh}_{k,\rm{H}}$ at the zero frequency is determined by the coefficients of the SCF on the identity matrix. The coefficients $\{w_{k,\rmd,t}\}_{t=1}^{T_\rmd}$ on $\bbL_{k,\rmd}$ and its powers contribute to $\tilde{\bbh}_{k,\rm{G}}$, acting in the gradient frequencies and gradient space, while the coefficients $\{w_{k,\rmu,t}\}_{t=1}^{T_\rmu}$ on $\bbL_{k,\rmu}$ and its powers contribute to $\tilde{\bbh}_{k,\rm{C}}$, acting in the curl frequencies and curl space. This is a direct result of \cref{cor:evd_Ld_Lu}. Unlike \eqref{eq.y_output_sft}, the SCF in \citet{ebli2020simplicial} has the same gradient and curl responses which prohibits different processing in the gradient and curl spaces.

The lower SCF $\bbH_{k,\rmd}$ has $\tilde{\bbh}_{k,\rmd} = \sum_{t=0}^{T_\rmd} w_{k,\rmd,t}^{\prime} \blambda_{k,\rm{G}}^{\odot t}$ as the frequency response that modulates the gradient embedding of $\bbx_{k,\rmd}$ and the upper SCF $\bbH_{k,\rmu}$ has $\tilde{\bbh}_{k,\rmu} = \sum_{t=0}^{T_\rmu}w_{k,\rmu,t}^{\prime} \blambda_{k,\rm{C}}^{\odot t}$ as the frequency response that modulates the curl embedding of $\bbx_{k,\rmu}$. 

Now, consider the output after the linear operation in an SCCNN layer $\bby_k = \bbH_{k,\rmd} \bbx_{k,\rmd} + \bbH_{k}\bbx_k + \bbH_{k,\rmu} \bbx_{k,\rmu}$. Its three spectral embeddings are given by  
\begin{equation}
\begin{cases}
  \tilde{\bby}_{k,\rm{H}} = \tilde{\bbh}_{k,\rm{H}} \odot \tilde{\bbx}_{k,\rm{H}}, \\ 
  \tilde{\bby}_{k,\rm{G}} = \tilde{\bbh}_{k,\rmd} \odot \tilde{\bbx}_{k,\rmd}  + \tilde{\bbh}_{k,\rm{G}} \odot \tilde{\bbx}_{k,\rm{G}}, \\ 
  \tilde{\bby}_{k,\rm{C}} =  \tilde{\bbh}_{k,\rm{C}} \odot \tilde{\bbx}_{k,\rm{C}}  + \tilde{\bbh}_{k,\rmu} \odot \tilde{\bbx}_{k,\rmu}.\\ 
\end{cases}
\end{equation}
This spectral relation shows how SCCNNs regulate the three inputs coming from simplices of different order and enable a flexible processing of inputs in different signal spaces owing to that different coefficients are used in the SCFs. 

The nonlinearity induces an information spillage \citep{gama2020stability} such that one type of spectral embedding could be spread over other types of frequencies. That is, $\sigma(\tilde{\bby}_{k,\rm{G}})$ could contain information in zero or curl frequencies. For example, a gradient flow projected from a node input could have information spillage in curl frequencies after $\sigma(\cdot)$. This spilled information further contributes to a triangle signal via projection $\bbB_2^\top$. Thus, the triangle output of SCCNNs contains information from nodes. This is the spectral perspective of the extended inter-simplicial locality. 

\section{Stability Analysis}\label{sec:stability_analysis}


Characterizing the stability of NNs to domain perturbations is key to understand their learning abilities from data \citep{bruna2013invariant,bronstein2021geometric}. The analysis by \citet{gama2020stability} showed that GNNs could be both stable and selective in contrast to graph convolution filters. We here perform a stability analysis of SCCNNs to understand the effect of various factors on the output of different simplices, with a focus on the roles of lower and upper simplicial adjacencies and inter-simplicial couplings. 

Domain perturbations could occur in a weighted SC as a result of misestimated simplicial weights. Denote as $\bbM_{k}$ a weight matrix of $k$-simplices. A weighted lower Laplacian is defined as $\bbL_{k,\rmd}=f_{k,\rmd}(\bbB_k,\bbM_{k-1},\bbM_{k})$ a function of incidence matrix $\bbB_k$ and weights $\bbM_{k-1},\bbM_k$, and likewise for the upper one $\bbL_{k,\rmu}=f_{k,\rmu}(\bbB_{k+1},\bbM_{k},\bbM_{k+1})$. The projections in SCCNNs are performed by the lower and upper projection matrices in place of $\bbB_k^\top$ and $\bbB_{k+1}$, defined as $\bbR_{k,\rmd}=f^\prime_{k,\rmd}(\bbB_k,\bbM_{k-1},\bbM_{k})$ and $\bbR_{k,\rmu}=f^\prime_{k,\rmu}(\bbB_{k+1},\bbM_{k},\bbM_{k+1})$, functions of incidence matrices and weights. See \cref{app:weighted sccnn} for some explicit forms by \citet{grady2010discrete,schaub2020random}. The misestimations of these weights could be viewed as relative perturbations on Hodge Laplacians and projection matrices. 


\begin{definition}[Relative Perturbation] \label{def:relative_perturbation}
  Consider a weighted SC $\ccalS$ with projection matrices $\bbR_{k,\rmd}$, $\bbR_{k,\rmu}$ and Hodge Laplacians $\bbL_{k,\rmd}$, $\bbL_{k,\rmu}$, $k=[K]$. A relative perturbed SC $\widehat{\ccalS}$ has 
  \begin{equation*}
    \begin{aligned}
      & \widehat{\bbR}_{k,\rmd}=\bbR_{k,\rmd}+\bbJ_{k,\rmd}\bbR_{k,\rmd}, \,\, \widehat{\bbR}_{k,\rmu}=\bbR_{k,\rmu}+\bbJ_{k,\rmu}\bbR_{k,\rmu}, \\
      & \widehat{\bbL}_{k,\rmd} = \bbL_{k,\rmd}  + \bbE_{k,\rmd} \bbL_{k,\rmd} + \bbL_{k,\rmd} \bbE_{k,\rmd} ,  \\
      & \widehat{\bbL}_{k,\rmu} = \bbL_{k,\rmu}  + \bbE_{k,\rmu} \bbL_{k,\rmu} + \bbL_{k,\rmu} \bbE_{k,\rmu} 
    \end{aligned}
  \end{equation*}
  where small perturbation matrices follow that $\lVert \bbE_{k,\rmd} \lVert \leq \epsilon_{k,\rmd}$ and $\lVert \bbJ_{k,\rmd} \lVert \leq \vepsilon_{k,\rmd}$ $\lVert \bbE_{k,\rmu} \lVert \leq \epsilon_{k,\rmu}$ and $\lVert \bbJ_{k,\rmu} \lVert \leq \vepsilon_{k,\rmu}$ with the spectral radius $\lVert\cdot \rVert$.

  

\end{definition} 
This model generalizes the graph perturbation model in \citet{gama2019stability,gama2020stability,parada2022stability} and implies that the same degree of perturbations affect differently stronger and weaker simplicial adjacencies. We further describe an SCF by its integral Lipschitz property. 

\begin{definition}[Integral Lipschitz SCF]\label{def.integral_lip_def_1}
  An SCF $\bbH_k$ is integral Lipschitz with constants $C_{k,\rmd}$ and $C_{k,\rmu}$ if 
  if 
  \begin{equation} \label{eq.integral_lip_def_1}
    |\lambda \tilde{h}'_{k,\rm{G}}(\lambda)| \leq C_{k,\rmd} \text{ and } |\lambda \tilde{h}'_{k,\rm{C}}(\lambda)| \leq C_{k,\rmu},
  \end{equation}
  with $\tilde{h}^\prime_{k,\rm{G}}(\lambda)$ and $\tilde{h}^\prime_{k,\rm{C}}(\lambda)$ the derivatives of the gradient and curl frequency response functions [cf. \eqref{eq.freq_response_scf}], respectively.  
\end{definition}
Integral Lipschitz SCFs can have a large variability in low simplicial frequencies $\lambda\to 0$, thus, a good selectivity with a low stability, while in large frequencies, they tend to be flat with a better stability at the cost of selectivity. This tradeoff holds independently for the gradient and curl frequencies. See \cref{app:perturbation model} for more details. As of the polynomial nature of frequency responses, all SCFs of an SCCNN are integral Lipschitz. We also denote the constant for the lower SCFs $\bbH_{k,\rmd}$ by $C_{k,\rmd}$ and for the upper SCFs $\bbH_{k,\rmu}$ by $C_{k,\rmu}$.  


\begin{assumption}\label{assump:bounded_filters}
  The SCFs $\bbH_k$ of an SCCNN have a normalized bounded frequency response  $|\tilde{h}_{k}(\lambda)|\leq 1$, likewise for $\bbH_{k,\rmd}$ and $\bbH_{k,\rmu}$, where we assume one for simplicity. 
\end{assumption}


\begin{assumption}\label{assump:bounded_proj}
  The lower and upper projections are finite with bounded norms $\lVert\bbR_{k,\rmd}\rVert\leq r_{k,\rmd}$ and $\lVert\bbR_{k,\rmu}\rVert\leq r_{k,\rmu}$. 
\end{assumption}

\begin{assumption}\label{assump:bounded_input}
  The initial input $\bbx_k^0$ is finite a limited energy $\lVert\bbx_k^0\rVert\leq\beta_k, k=[K]$, collected in $\bbeta=[\beta_0,\dots,\beta_K]^\top$. 
\end{assumption}

\begin{assumption}\label{assump:lip_nonlinear}
  The nonlinearity $\sigma(\cdot)$ is $C_\sigma$-Lipschitz, i.e., $|\sigma(b)-\sigma(a)|\leq |b-a|$. 
\end{assumption}

\begin{figure}[t!]
  \vskip 0.1in 
  \centering
  \begin{subfigure}{0.495\linewidth}
    \includegraphics[width=\linewidth]{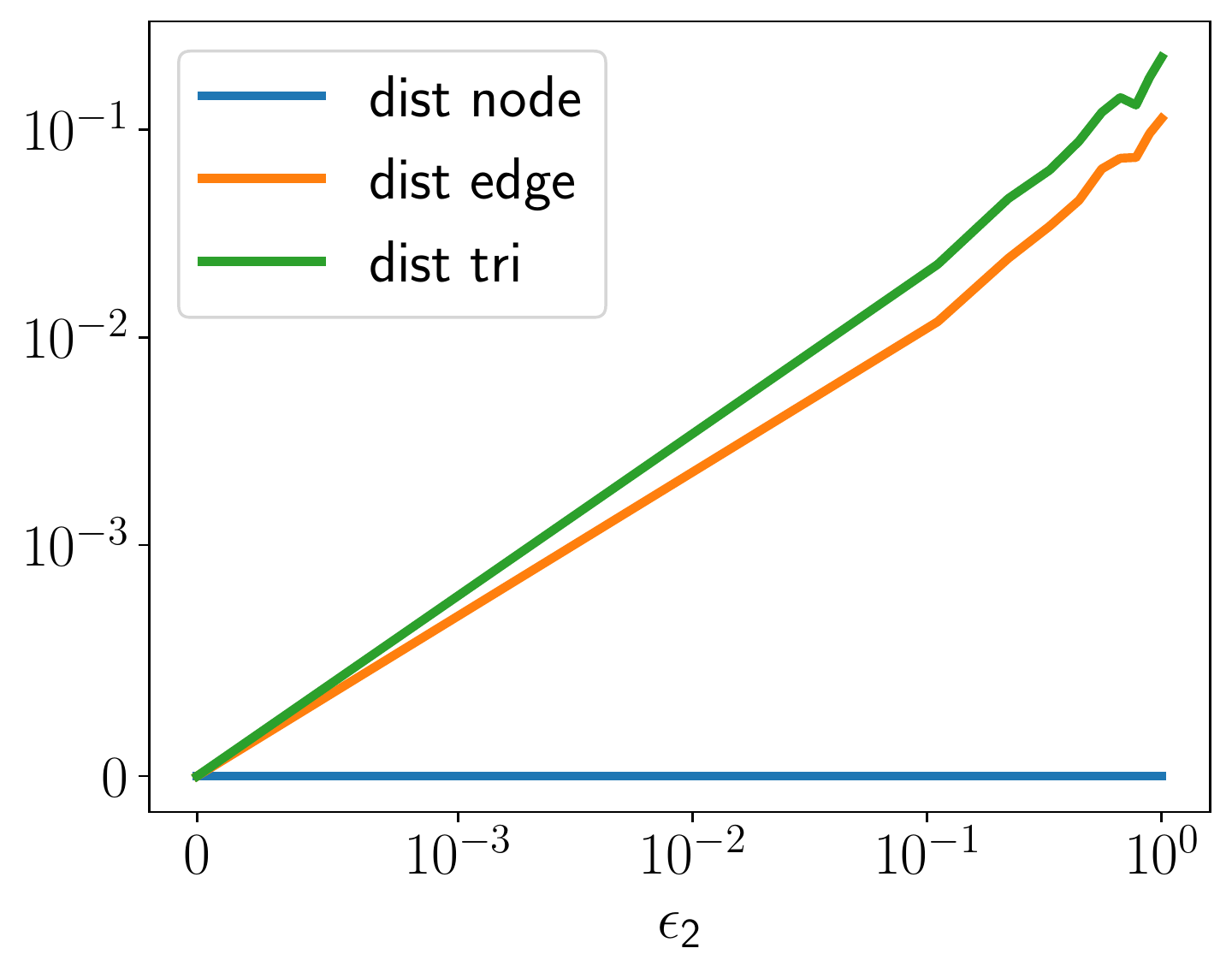}
  \end{subfigure}
  \begin{subfigure}{0.495\linewidth}
    \includegraphics[width=\linewidth]{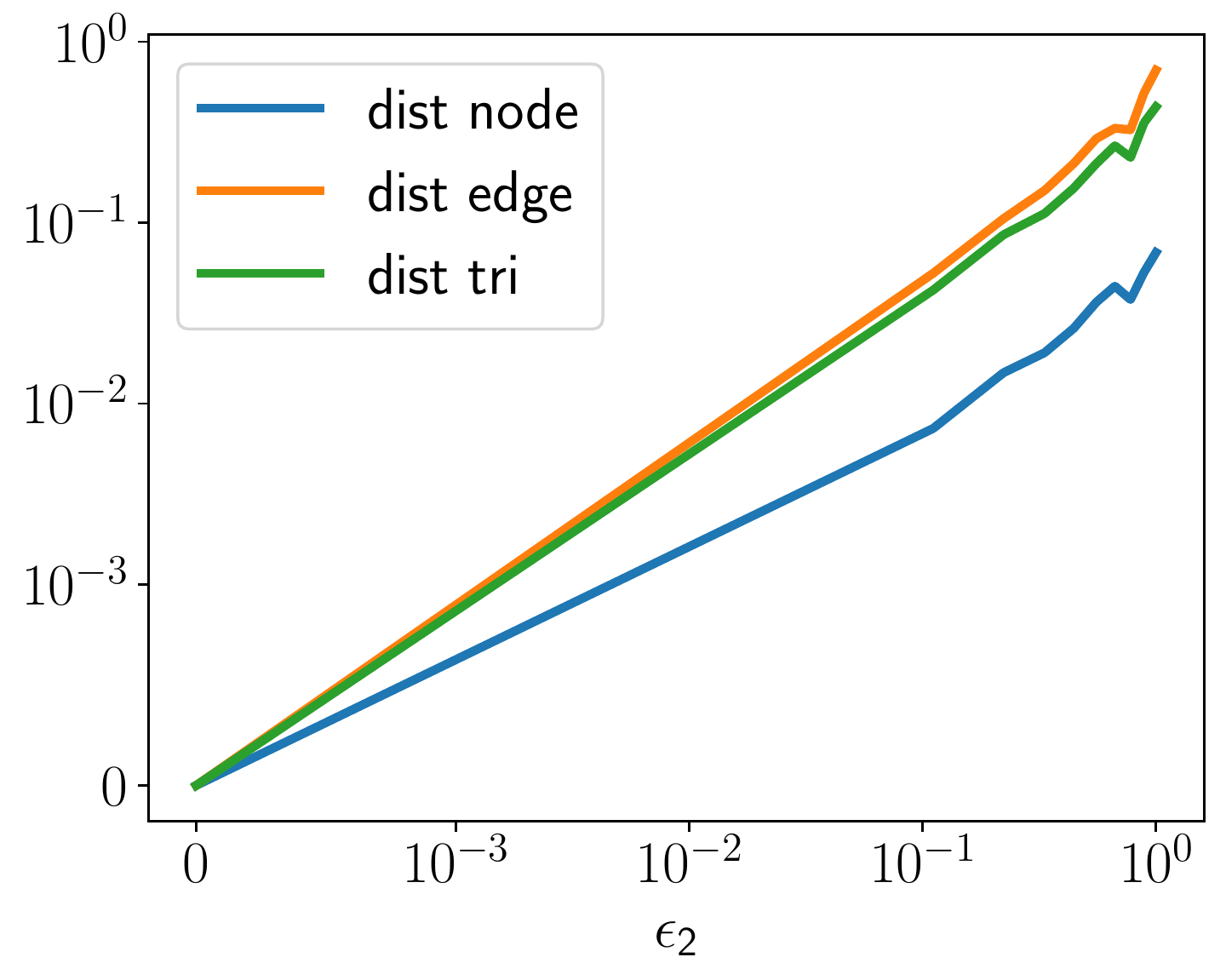}
  \end{subfigure}
  \vskip -0.05in
  \caption{Euclidean distances of node, edge and triangle outputs of an SCCNN with one \emph{(Left)} and two \emph{(Right)} layers under perturbations on triangle weights. Edge output is influenced after one layer, while node output is influenced after two layers.}
  \label{fig:stability_main_content}
  \vskip -0.15in 
\end{figure}

\begin{theorem}[Stability]\label{thm:stability}
Let $\bbx_k^L$ be the output of an $L$-layer SCCNN defined on a weighted SC. Let $\hat{\bbx}_k^L$ be the output of the same SCNN but on a relatively perturbed SC according to \cref{def:relative_perturbation}. Under \cref{assump:bounded_filters,assump:bounded_proj,assump:bounded_input,assump:lip_nonlinear}, the Euclidean distance between the two outputs is finite and upper-bounded by $ \lVert\hat{\bbx}_k^L-\bbx_k^L\rVert\leq [\bbd]_k$ with
\begin{equation} \label{thm.eq.stability}
    \textstyle \bbd = C_\sigma^L \sum_{l=1}^{L}\widehat{\bbZ}^{l-1}\bbT\bbZ^{L-l} \bbeta. 
  \end{equation} 
Here, matrices $\bbT$ and $\bbZ$ are tridiagonal, e.g., for $K = 2$,
  \begin{equation*}
    \bbT = 
    \begin{bmatrix}
      t_0 & t_{0,\rmu} &  \\ 
      t_{1,\rmd} & t_1 & t_{1,\rmu} \\
       & t_{2,\rmd} & t_2 \\ 
    \end{bmatrix} \text{ and }
    \bbZ = 
    \begin{bmatrix}
      1 & r_{0,\rmu} & \\ 
      r_{1,\rmd} & 1 & r_{1,\rmu} \\
        & r_{2,\rmd} & 1 \\ 
    \end{bmatrix},
  \end{equation*}
with constants 
$t_{k,\rmd}=r_{k,\rmd}\vepsilon_{k,\rmd} + C_{k,\rmd}\Delta_{k,\rmd}\epsilon_{k,\rmd}r_{k,\rmd}$, $t_{k,\rmu}=r_{k,\rmu}\vepsilon_{k,\rmu} + C_{k,\rmu}\Delta_{k,\rmu}\epsilon_{k,\rmu}r_{k,\rmu}$ and $t_k=C_{k,\rmd} \Delta_{k,\rmd} \epsilon_{k,\rmd} + C_{k,\rmu}\Delta_{k,\rmu}\epsilon_{k,\rmu}$, where $\Delta_{k,\rmd}$ and $\Delta_{k,\rmu}$ capture the eigenvector misalignment between the respective Hodge Laplacians and their perturbations, scaled by $\sqrt{N_k}$. Matrix $\widehat{\bbZ}$ is defined as ${\bbZ}$ but with off-diagonal entries $\hat{r}_{k,\rmd} = r_{k,\rmd} (1+\vepsilon_{k,\rmd})$ and $\hat{r} _{k,\rmu} = r_{k,\rmu} (1+\vepsilon_{k,\rmu})$.
\end{theorem}

The proof can be found in \cref{app:proof_stability}. 
This result shows that SCCNNs are stable to relative domain perturbations, which can be analyzed from two perspectives. First, the stability of $k$-simplicial output depends on not only factors of $k$-simplices, but also simplices of other orders due to the inter-simplicial couplings. When $L = 1$, the node output bound $d_0$ depends on $\beta_0$ via $t_0$, and $\beta_1$ via $t_{0,\rmu}$ where node perturbation $\bbE_{0,\rmu}$ (described by $\Delta_{0,\rmu}$ and $\epsilon_{0,\rmu}$), node SCFs (by $C_{0,\rmu}$) and projections from edges to nodes (by $r_{0,\rmu}$ and $\vepsilon_{0,\rmu}$) play a role. Likewise, $d_1$ depends on $\beta_0,\beta_1$ and $\beta_2$ via factors of perturbations, SCFs and projections in the edge space. When $L=2$, the bound $d_0$ is also affected by $t_{1,\rmd},t_1$ and $t_{1,\rmu}$ containing factors in the edge space. As $L$ increases, factors in the triangle space (and higher-order simplicial space) will appear in $d_0$, as illustrated in \cref{fig:stability_main_content}. Thus, while leveraging information from adjacent simplices may be beneficial, it may severely affect the stability when the SC is perturbed. This can be mitigated by using less layers and higher-order SCFs, imposed by stronger integral Lipschitz properties, to maintain the expressive power.  

Second, the integral Lipschitz property $C_{k,\rmd}$ in gradient frequencies plays no role in the stability against upper perturbations, and vice-versa. Thus, if there are only triangle perturbations SCFs in the edge space need not to be strictly integral Lipschitz in gradient frequencies where SCCNNs could be more selective while preserving stability. This is a direct benefit of using different parameter spaces in the gradient and curl spaces unlike in \citet{ebli2020simplicial}.

\section{Experiments}
\textbf{Simplex Prediction.}
We consider the task of simplex prediction: \emph{ given all the $(k-1)$-simplices in a set of $k+1$ nodes, to predict if this set will be closed to form a $k$-simplex}, which is an extension of link prediction in graphs \citep{zhang2018link}. 
Our approach is to first learn features of lower-order simplices and then use an MLP to identify if a simplex is closed or open.
With coauthorship data in \citet{ammar2018construction} we built an SC as \citet{ebli2020simplicial} where nodes are authors and collaborations between $k$-authors are modeled as $(k-1)$-simplices. The simplicial signals are the number of citations, e.g., $\bbx_1$ and $\bbx_2$ are the citations of dyadic and triadic collaborations.
Thus, 2-simplex predictions in this SC give rise to predict triadic collaborations given the pairwise collaborations in the triads. 


In an SC of order two, for an open triangle, we use an SCCNN to learn features of nodes and edges. Then, an MLP is used to predict if this triangle shall be closed or not based on its three node or edge features.
We also perform a 3-simplex prediction, which amounts to predicting tetradic collaborations. \cref{tab:simplex_prediction} reports the AUC results. We include the experiment details and an ablation study in \cref{app:experiments_simplex_pred}.


We see that SC solutions achieve better results than using only graphs, validating SCs as an inductive model. SCCNN performs best in both tasks, since it exploits both intra- and extended inter-simplicial localities. First, SCCNN leverages the lower and upper adjacencies individually in a multi-hop fashion to perform convolutions such that it performs better than Bunch. For the same reason, SCNN performs better than SNN and PSNN. 
Second, SCCNN considers information from adjacent simplices such that it gives better results than methods without the inter-simplicial locality, such as GNN and SCNN, or CF-SC with a limited locality.

\begin{table}[t!]
  \vskip -0.1in
  \caption{Simplex prediction AUC ($\%$, area under the curve) results for ten runs. The \emph{first} and \emph{second} best results are in \emph{\textcolor{red}{red}} and \emph{\textcolor{blue}{blue}}. }
  \label{tab:simplex_prediction}
  \vskip 0.1in
  \begin{center}
  \begin{small}
  \resizebox{\columnwidth}{!}{
  \begin{tabular}{lcr}
  \toprule
  Methods & 2-Simplex & 3-Simplex\\
  \midrule
  Harm. Mean \citep{benson2018simplicial} \citet{} &62.8$\pm$2.7  &63.6$\pm$1.6\\
  MLP &68.5$\pm$1.6 &69.0$\pm$2.2\\
  GF \citep{sandryhaila2013discrete} &78.7$\pm$1.2 &83.9$\pm$2.3\\
  SCF \citep{yang2022simplicial} & 92.6$\pm$1.8 & 94.9$\pm$1.0\\
  CF-SC \citep{isufi2022convolutional} & 96.9$\pm$0.8 &97.9$\pm$0.7\\
  GNN \citep{gama2020stability} &93.9$\pm$1.0 &96.6$\pm$0.5\\
  SNN \citep{ebli2020simplicial} &92.0$\pm$1.8 &95.1$\pm$1.2\\
  PSNN \citep{roddenberry2021principled} &95.6$\pm$1.3 &98.1$\pm$0.5\\
  SCNN \citep{yang2021simplicial} &96.5$\pm$1.5 &98.3$\pm$0.4\\
  Bunch \citep{bunch2020simplicial} &\textcolor{blue}{98.0$\pm$0.5} & \textcolor{blue}{98.5$\pm$0.5}\\
  \textbf{SCCNN (ours)} &\textcolor{red}{98.4$\pm$0.5} & \textcolor{red}{99.4$\pm$0.3}\\
  \bottomrule
  \end{tabular}
  }
  \end{small}
  \end{center}
  \vskip -0.15in
\end{table}

\begin{table}[t!]
  \vskip -0.1in 
  \caption{Prediction accuracy of the synthetic data in standard, reverse and generalization tasks, and ocean drifters (Last Column). }
  \label{tab:trajectory_prediction_both}
  \vskip 0.1in
  \begin{center}
  \begin{small}
  \resizebox{\columnwidth}{!}{
  \begin{tabular}{lcccr}
  \toprule
  Methods & Standard & Reverse & Gen. & Ocean \\
  \midrule
  PSNN & 63.1$\pm$3.1 & 58.4$\pm$3.9 & 55.3$\pm$2.5 & 49.0$\pm$8.0 \\
  SCNN & \textcolor{red}{67.7$\pm$1.7} & 55.3$\pm$5.3 & 61.2$\pm$3.2 & \textcolor{blue}{53.0$\pm$7.8} \\ 
  SNN & \textcolor{blue}{65.5$\pm$2.4} & 53.6$\pm$6.1 & 59.5$\pm$3.7 & \textbf{52.5}$\pm$\textbf{6.0} \\ 
  Bunch & 62.3$\pm$4.0 & 59.6$\pm$6.1 & 53.9$\pm$3.1 & 46.0$\pm$6.2 \\
  \textbf{SCCNN (ours)} & \textbf{65.2}$\pm$\textbf{4.1} & 58.9$\pm$4.1 & 56.8$\pm$2.4 & \textcolor{red}{54.5$\pm$7.9} \\
  \bottomrule
  \end{tabular}
  }
  \end{small}
  \end{center}
  \vskip -0.15in
\end{table}

\textbf{Trajectory Prediction.}  In the trajectory prediction task, a trajectory is represented as an edge flow and the goal is to predict the next node based on this representation \citep{roddenberry2021principled}. We consider trajectories in a synthetic SC and ocean drifter trajectories localized around Madagascar \citep{schaub2020random}. The experiment details can be found in \cref{app:traj_pred_results}. From the results of different methods in \cref{tab:trajectory_prediction_both}, we see that SCCNN performs better than Bunch due to the use of higher-order convolutions, and likewise, SCNN and SNN give better predictions than PSNN. Also, differentiating the parameter space for lower and upper convolutions improves the performance of SCNN compared to SNN.
As zero inputs on nodes and triangles are applied, SCCNN does not perform better than SCNN. Like other NNs, SCCNNs do not deteriorate in the reverse task owing to the orientation equivariance and they show good generalization ability to the unseen data as well.  

\section{Conclusion}
We proposed an SCCNN for learning on SCs, which performs a simplicial convolution with an intra-simplicial locality and multi-hop information from adjacent simplices with an extended inter-simplicial locality. 
We provide a through theoretical study of the proposed architecture from different viewpoints. 
First, we study the symmetries in an SC and simplicial signal space and show the SCCNN can be built equivariant to permutations and orientations of simplices. 
We then study its spectral behavior and understand how the learned convolutional filters perform in the different simplicial frequencies, i.e., in different simplicial signal spaces. 
Finally, we study the stability of the SCCNN, showing that it is stable to domain perturbations and how the inter-simplicial locality affects the performance. 
We corroborate these results with numerical experiments achieving a comparable performance with state-of-the-art alternatives.

\bibliography{refs}

\begin{thebibliography}{67}
\providecommand{\natexlab}[1]{#1}
\providecommand{\url}[1]{\texttt{#1}}
\expandafter\ifx\csname urlstyle\endcsname\relax
  \providecommand{\doi}[1]{doi: #1}\else
  \providecommand{\doi}{doi: \begingroup \urlstyle{rm}\Url}\fi

\bibitem[Ammar et~al.(2018)Ammar, Groeneveld, Bhagavatula, Beltagy, Crawford,
  Downey, Dunkelberger, Elgohary, Feldman, Ha, et~al.]{ammar2018construction}
Ammar, W., Groeneveld, D., Bhagavatula, C., Beltagy, I., Crawford, M., Downey,
  D., Dunkelberger, J., Elgohary, A., Feldman, S., Ha, V., et~al.
\newblock Construction of the literature graph in semantic scholar.
\newblock \emph{arXiv preprint arXiv:1805.02262}, 2018.

\bibitem[Barbarossa \& Sardellitti(2020)Barbarossa and
  Sardellitti]{barbarossa2020}
Barbarossa, S. and Sardellitti, S.
\newblock Topological signal processing over simplicial complexes.
\newblock \emph{IEEE Transactions on Signal Processing}, 68:\penalty0
  2992--3007, 2020.

\bibitem[Battaglia et~al.(2018)Battaglia, Hamrick, Bapst, Sanchez-Gonzalez,
  Zambaldi, Malinowski, Tacchetti, Raposo, Santoro, Faulkner,
  et~al.]{battaglia2018relational}
Battaglia, P.~W., Hamrick, J.~B., Bapst, V., Sanchez-Gonzalez, A., Zambaldi,
  V., Malinowski, M., Tacchetti, A., Raposo, D., Santoro, A., Faulkner, R.,
  et~al.
\newblock Relational inductive biases, deep learning, and graph networks.
\newblock \emph{arXiv preprint arXiv:1806.01261}, 2018.

\bibitem[Benson et~al.(2018)Benson, Abebe, Schaub, Jadbabaie, and
  Kleinberg]{benson2018simplicial}
Benson, A.~R., Abebe, R., Schaub, M.~T., Jadbabaie, A., and Kleinberg, J.
\newblock Simplicial closure and higher-order link prediction.
\newblock \emph{Proceedings of the National Academy of Sciences}, 115\penalty0
  (48):\penalty0 E11221--E11230, 2018.

\bibitem[Bick et~al.(2021)Bick, Gross, Harrington, and Schaub]{bick2021higher}
Bick, C., Gross, E., Harrington, H.~A., and Schaub, M.~T.
\newblock What are higher-order networks?
\newblock \emph{arXiv preprint arXiv:2104.11329}, 2021.

\bibitem[Bodnar et~al.(2021{\natexlab{a}})Bodnar, Frasca, Otter, Wang, Li{\`o},
  Montufar, and Bronstein]{bodnar2021weisfeiler2}
Bodnar, C., Frasca, F., Otter, N., Wang, Y.~G., Li{\`o}, P., Montufar, G.~F.,
  and Bronstein, M.
\newblock Weisfeiler and lehman go cellular: Cw networks.
\newblock \emph{Advances in Neural Information Processing Systems}, 34,
  2021{\natexlab{a}}.

\bibitem[Bodnar et~al.(2021{\natexlab{b}})Bodnar, Frasca, Wang, Otter,
  Montufar, Lio, and Bronstein]{bodnar2021weisfeiler}
Bodnar, C., Frasca, F., Wang, Y., Otter, N., Montufar, G.~F., Lio, P., and
  Bronstein, M.
\newblock Weisfeiler and lehman go topological: Message passing simplicial
  networks.
\newblock In \emph{International Conference on Machine Learning}, pp.\
  1026--1037. PMLR, 2021{\natexlab{b}}.

\bibitem[Bronstein et~al.(2021)Bronstein, Bruna, Cohen, and
  Veli{\v{c}}kovi{\'c}]{bronstein2021geometric}
Bronstein, M.~M., Bruna, J., Cohen, T., and Veli{\v{c}}kovi{\'c}, P.
\newblock Geometric deep learning: Grids, groups, graphs, geodesics, and
  gauges.
\newblock \emph{arXiv preprint arXiv:2104.13478}, 2021.

\bibitem[Bruna \& Mallat(2013)Bruna and Mallat]{bruna2013invariant}
Bruna, J. and Mallat, S.
\newblock Invariant scattering convolution networks.
\newblock \emph{IEEE transactions on pattern analysis and machine
  intelligence}, 35\penalty0 (8):\penalty0 1872--1886, 2013.

\bibitem[Bunch et~al.(2020)Bunch, You, Fung, and Singh]{bunch2020simplicial}
Bunch, E., You, Q., Fung, G., and Singh, V.
\newblock Simplicial 2-complex convolutional neural networks.
\newblock In \emph{TDA {\&} Beyond}, 2020.
\newblock URL \url{https://openreview.net/forum?id=TLbnsKrt6J-}.

\bibitem[Candogan et~al.(2011)Candogan, Menache, Ozdaglar, and
  Parrilo]{candogan2011flows}
Candogan, O., Menache, I., Ozdaglar, A., and Parrilo, P.~A.
\newblock Flows and decompositions of games: Harmonic and potential games.
\newblock \emph{Mathematics of Operations Research}, 36\penalty0 (3):\penalty0
  474--503, 2011.

\bibitem[Chen et~al.(2022{\natexlab{a}})Chen, Gel, and Poor]{chen2022time}
Chen, Y., Gel, Y., and Poor, H.~V.
\newblock Time-conditioned dances with simplicial complexes: Zigzag filtration
  curve based supra-hodge convolution networks for time-series forecasting.
\newblock In \emph{Advances in Neural Information Processing Systems},
  2022{\natexlab{a}}.

\bibitem[Chen et~al.(2022{\natexlab{b}})Chen, Gel, and Poor]{chen2022bscnets}
Chen, Y., Gel, Y.~R., and Poor, H.~V.
\newblock Bscnets: Block simplicial complex neural networks.
\newblock \emph{Proceedings of the AAAI Conference on Artificial Intelligence},
  36\penalty0 (6):\penalty0 6333--6341, 2022{\natexlab{b}}.
\newblock \doi{10.1609/aaai.v36i6.20583}.
\newblock URL \url{https://ojs.aaai.org/index.php/AAAI/article/view/20583}.

\bibitem[Cordonnier \& Loukas(2019)Cordonnier and
  Loukas]{cordonnier2019extrapolating}
Cordonnier, J.-B. and Loukas, A.
\newblock Extrapolating paths with graph neural networks.
\newblock In \emph{Proceedings of the Twenty-Eighth International Joint
  Conference on Artificial Intelligence, {IJCAI-19}}, pp.\  2187--2194.
  International Joint Conferences on Artificial Intelligence Organization, 7
  2019.
\newblock \doi{10.24963/ijcai.2019/303}.
\newblock URL \url{https://doi.org/10.24963/ijcai.2019/303}.

\bibitem[De~Longueville(2012)]{de2012course}
De~Longueville, M.
\newblock \emph{A course in topological combinatorics}.
\newblock Springer Science \& Business Media, 2012.

\bibitem[Defferrard et~al.(2016)Defferrard, Bresson, and
  Vandergheynst]{defferrard2017}
Defferrard, M., Bresson, X., and Vandergheynst, P.
\newblock Convolutional neural networks on graphs with fast localized spectral
  filtering.
\newblock In Lee, D., Sugiyama, M., Luxburg, U., Guyon, I., and Garnett, R.
  (eds.), \emph{Advances in Neural Information Processing Systems}, volume~29.
  Curran Associates, Inc., 2016.
\newblock URL
  \url{https://proceedings.neurips.cc/paper/2016/file/04df4d434d481c5bb723be1b6df1ee65-Paper.pdf}.

\bibitem[Ebli et~al.(2020)Ebli, Defferrard, and Spreemann]{ebli2020simplicial}
Ebli, S., Defferrard, M., and Spreemann, G.
\newblock Simplicial neural networks.
\newblock In \emph{NeurIPS 2020 Workshop on Topological Data Analysis and
  Beyond}, 2020.

\bibitem[Gama et~al.(2019{\natexlab{a}})Gama, Marques, Leus, and
  Ribeiro]{gama2019convolutional}
Gama, F., Marques, A.~G., Leus, G., and Ribeiro, A.
\newblock Convolutional graph neural networks.
\newblock In \emph{2019 53rd Asilomar Conference on Signals, Systems, and
  Computers}, pp.\  452--456. IEEE, 2019{\natexlab{a}}.

\bibitem[Gama et~al.(2019{\natexlab{b}})Gama, Ribeiro, and
  Bruna]{gama2019stability}
Gama, F., Ribeiro, A., and Bruna, J.
\newblock Stability of graph scattering transforms.
\newblock \emph{Advances in Neural Information Processing Systems}, 32,
  2019{\natexlab{b}}.

\bibitem[Gama et~al.(2020{\natexlab{a}})Gama, Bruna, and
  Ribeiro]{gama2020stability}
Gama, F., Bruna, J., and Ribeiro, A.
\newblock Stability properties of graph neural networks.
\newblock \emph{IEEE Transactions on Signal Processing}, 68:\penalty0
  5680--5695, 2020{\natexlab{a}}.

\bibitem[Gama et~al.(2020{\natexlab{b}})Gama, Isufi, Leus, and
  Ribeiro]{gama2020graphs}
Gama, F., Isufi, E., Leus, G., and Ribeiro, A.
\newblock Graphs, convolutions, and neural networks: From graph filters to
  graph neural networks.
\newblock \emph{IEEE Signal Processing Magazine}, 37\penalty0 (6):\penalty0
  128--138, 2020{\natexlab{b}}.

\bibitem[Ghosh et~al.(2018)Ghosh, Rozemberczki, Ramamoorthy, and
  Sarkar]{ghosh2018topological}
Ghosh, A., Rozemberczki, B., Ramamoorthy, S., and Sarkar, R.
\newblock Topological signatures for fast mobility analysis.
\newblock In \emph{Proceedings of the 26th ACM SIGSPATIAL International
  Conference on Advances in Geographic Information Systems}, pp.\  159--168,
  2018.

\bibitem[Gilmer et~al.(2017)Gilmer, Schoenholz, Riley, Vinyals, and
  Dahl]{gilmer2017neural}
Gilmer, J., Schoenholz, S.~S., Riley, P.~F., Vinyals, O., and Dahl, G.~E.
\newblock Neural message passing for quantum chemistry.
\newblock In \emph{International conference on machine learning}, pp.\
  1263--1272. PMLR, 2017.

\bibitem[Giusti et~al.(2022)Giusti, Battiloro, Di~Lorenzo, Sardellitti, and
  Barbarossa]{giusti2022simplicial}
Giusti, L., Battiloro, C., Di~Lorenzo, P., Sardellitti, S., and Barbarossa, S.
\newblock Simplicial attention networks.
\newblock \emph{arXiv preprint arXiv:2203.07485}, 2022.

\bibitem[Goh et~al.(2022)Goh, Bodnar, and Lio]{goh2022simplicial}
Goh, C. W.~J., Bodnar, C., and Lio, P.
\newblock Simplicial attention networks.
\newblock In \emph{ICLR 2022 Workshop on Geometrical and Topological
  Representation Learning}, 2022.

\bibitem[Grady \& Polimeni(2010)Grady and Polimeni]{grady2010discrete}
Grady, L.~J. and Polimeni, J.~R.
\newblock \emph{Discrete calculus: Applied analysis on graphs for computational
  science}, volume~3.
\newblock Springer, 2010.

\bibitem[Hajij et~al.(2020)Hajij, Istvan, and Zamzmi]{hajij2020cell}
Hajij, M., Istvan, K., and Zamzmi, G.
\newblock Cell complex neural networks.
\newblock In \emph{NeurIPS 2020 Workshop on Topological Data Analysis and
  Beyond}, 2020.

\bibitem[Hajij et~al.(2021)Hajij, Zamzmi, Papamarkou, Maroulas, and
  Cai]{hajij2021simplicial}
Hajij, M., Zamzmi, G., Papamarkou, T., Maroulas, V., and Cai, X.
\newblock Simplicial complex representation learning.
\newblock \emph{arXiv preprint arXiv:2103.04046}, 2021.

\bibitem[Hajij et~al.(2022)Hajij, Zamzmi, Papamarkou, Miolane,
  Guzm{\'a}n-S{\'a}enz, and Ramamurthy]{hajij2022higher}
Hajij, M., Zamzmi, G., Papamarkou, T., Miolane, N., Guzm{\'a}n-S{\'a}enz, A.,
  and Ramamurthy, K.~N.
\newblock Higher-order attention networks.
\newblock \emph{arXiv preprint arXiv:2206.00606}, 2022.

\bibitem[Hamilton(2020)]{hamilton2020graph}
Hamilton, W.~L.
\newblock \emph{Graph representation learning}.
\newblock Morgan \& Claypool Publishers, 2020.

\bibitem[Hansen \& Ghrist(2019)Hansen and Ghrist]{hansen2019toward}
Hansen, J. and Ghrist, R.
\newblock Toward a spectral theory of cellular sheaves.
\newblock \emph{Journal of Applied and Computational Topology}, 3\penalty0
  (4):\penalty0 315--358, 2019.

\bibitem[Hodge(1989)]{hodge1989theory}
Hodge, W. V.~D.
\newblock \emph{The theory and applications of harmonic integrals}.
\newblock CUP Archive, 1989.

\bibitem[Horak \& Jost(2013)Horak and Jost]{horak2013spectra}
Horak, D. and Jost, J.
\newblock Spectra of combinatorial laplace operators on simplicial complexes.
\newblock \emph{Advances in Mathematics}, 244:\penalty0 303--336, 2013.

\bibitem[Isufi \& Yang(2022)Isufi and Yang]{isufi2022convolutional}
Isufi, E. and Yang, M.
\newblock Convolutional filtering in simplicial complexes.
\newblock In \emph{ICASSP 2022 - 2022 IEEE International Conference on
  Acoustics, Speech and Signal Processing (ICASSP)}, pp.\  5578--5582, 2022.
\newblock \doi{10.1109/ICASSP43922.2022.9746349}.

\bibitem[Jia et~al.(2019)Jia, Schaub, Segarra, and Benson]{jia2019graph}
Jia, J., Schaub, M.~T., Segarra, S., and Benson, A.~R.
\newblock Graph-based semi-supervised \& active learning for edge flows.
\newblock In \emph{Proceedings of the 25th ACM SIGKDD International Conference
  on Knowledge Discovery \& Data Mining}, pp.\  761--771, 2019.

\bibitem[Jiang et~al.(2011)Jiang, Lim, Yao, and Ye]{jiang2011statistical}
Jiang, X., Lim, L.-H., Yao, Y., and Ye, Y.
\newblock Statistical ranking and combinatorial hodge theory.
\newblock \emph{Mathematical Programming}, 127\penalty0 (1):\penalty0 203--244,
  2011.

\bibitem[Keros et~al.(2022)Keros, Nanda, and Subr]{keros2022dist2cycle}
Keros, A.~D., Nanda, V., and Subr, K.
\newblock Dist2cycle: A simplicial neural network for homology localization.
\newblock \emph{Proceedings of the AAAI Conference on Artificial Intelligence},
  36\penalty0 (7):\penalty0 7133--7142, 2022.
\newblock \doi{10.1609/aaai.v36i7.20673}.
\newblock URL \url{https://ojs.aaai.org/index.php/AAAI/article/view/20673}.

\bibitem[Kipf \& Welling(2016)Kipf and Welling]{kipf2016variational}
Kipf, T.~N. and Welling, M.
\newblock Variational graph auto-encoders.
\newblock \emph{arXiv preprint arXiv:1611.07308}, 2016.

\bibitem[Kipf \& Welling(2017)Kipf and Welling]{kipf2017semi}
Kipf, T.~N. and Welling, M.
\newblock Semi-supervised classification with graph convolutional networks.
\newblock In \emph{International Conference on Learning Representations
  (ICLR)}, 2017.

\bibitem[Lim(2020)]{lim2015hodge}
Lim, L.-H.
\newblock Hodge laplacians on graphs.
\newblock \emph{SIAM Review}, 62\penalty0 (3):\penalty0 685--715, 2020.

\bibitem[Masoomy et~al.(2021)Masoomy, Askari, Tajik, Rizi, and
  Jafari]{masoomy2021topological}
Masoomy, H., Askari, B., Tajik, S., Rizi, A.~K., and Jafari, G.~R.
\newblock Topological analysis of interaction patterns in cancer-specific gene
  regulatory network: persistent homology approach.
\newblock \emph{Scientific Reports}, 11\penalty0 (1):\penalty0 1--11, 2021.

\bibitem[Money et~al.(2022)Money, Krishnan, Beferull-Lozano, and
  Isufi]{money2022online}
Money, R., Krishnan, J., Beferull-Lozano, B., and Isufi, E.
\newblock Online edge flow imputation on networks.
\newblock \emph{IEEE Signal Processing Letters}, 2022.

\bibitem[Munkres(2018)]{munkres2018elements}
Munkres, J.~R.
\newblock \emph{Elements of algebraic topology}.
\newblock CRC press, 2018.

\bibitem[Newman et~al.(2002)Newman, Watts, and Strogatz]{newman2002random}
Newman, M.~E., Watts, D.~J., and Strogatz, S.~H.
\newblock Random graph models of social networks.
\newblock \emph{Proceedings of the national academy of sciences}, 99\penalty0
  (suppl\_1):\penalty0 2566--2572, 2002.

\bibitem[Nikhil \& Tran~Morris(2018)Nikhil and
  Tran~Morris]{nikhil2018convolutional}
Nikhil, N. and Tran~Morris, B.
\newblock Convolutional neural network for trajectory prediction.
\newblock In \emph{Proceedings of the European Conference on Computer Vision
  (ECCV) Workshops}, pp.\  0--0, 2018.

\bibitem[Parada-Mayorga et~al.(2022)Parada-Mayorga, Wang, Gama, and
  Ribeiro]{parada2022stability}
Parada-Mayorga, A., Wang, Z., Gama, F., and Ribeiro, A.
\newblock Stability of aggregation graph neural networks.
\newblock \emph{arXiv preprint arXiv:2207.03678}, 2022.

\bibitem[Roddenberry \& Segarra(2019)Roddenberry and
  Segarra]{roddenberry2019hodgenet}
Roddenberry, T.~M. and Segarra, S.
\newblock Hodgenet: Graph neural networks for edge data.
\newblock In \emph{2019 53rd Asilomar Conference on Signals, Systems, and
  Computers}, pp.\  220--224. IEEE, 2019.

\bibitem[Roddenberry et~al.(2021)Roddenberry, Glaze, and
  Segarra]{roddenberry2021principled}
Roddenberry, T.~M., Glaze, N., and Segarra, S.
\newblock Principled simplicial neural networks for trajectory prediction.
\newblock In \emph{International Conference on Machine Learning}, pp.\
  9020--9029. PMLR, 2021.

\bibitem[Roddenberry et~al.(2022)Roddenberry, Schaub, and
  Hajij]{roddenberry2022signal}
Roddenberry, T.~M., Schaub, M.~T., and Hajij, M.
\newblock Signal processing on cell complexes.
\newblock In \emph{ICASSP 2022-2022 IEEE International Conference on Acoustics,
  Speech and Signal Processing (ICASSP)}, pp.\  8852--8856. IEEE, 2022.

\bibitem[Rudenko et~al.(2020)Rudenko, Palmieri, Herman, Kitani, Gavrila, and
  Arras]{rudenko2020human}
Rudenko, A., Palmieri, L., Herman, M., Kitani, K.~M., Gavrila, D.~M., and
  Arras, K.~O.
\newblock Human motion trajectory prediction: A survey.
\newblock \emph{The International Journal of Robotics Research}, 39\penalty0
  (8):\penalty0 895--935, 2020.

\bibitem[Ruiz et~al.(2021)Ruiz, Gama, and Ribeiro]{ruiz2021graph}
Ruiz, L., Gama, F., and Ribeiro, A.
\newblock Graph neural networks: architectures, stability, and transferability.
\newblock \emph{Proceedings of the IEEE}, 109\penalty0 (5):\penalty0 660--682,
  2021.

\bibitem[Sandryhaila \& Moura(2013)Sandryhaila and
  Moura]{sandryhaila2013discrete}
Sandryhaila, A. and Moura, J.~M.
\newblock Discrete signal processing on graphs.
\newblock \emph{IEEE Transactions on Signal Processing}, 61\penalty0
  (7):\penalty0 1644--1656, 2013.

\bibitem[Sandryhaila \& Moura(2014)Sandryhaila and
  Moura]{sandryhaila2014discrete}
Sandryhaila, A. and Moura, J.~M.
\newblock Discrete signal processing on graphs: Frequency analysis.
\newblock \emph{IEEE Transactions on Signal Processing}, 62\penalty0
  (12):\penalty0 3042--3054, 2014.

\bibitem[Sardellitti et~al.(2021)Sardellitti, Barbarossa, and
  Testa]{sardellitti2021topological}
Sardellitti, S., Barbarossa, S., and Testa, L.
\newblock Topological signal processing over cell complexes.
\newblock In \emph{2021 55th Asilomar Conference on Signals, Systems, and
  Computers}, pp.\  1558--1562. IEEE, 2021.

\bibitem[Schaub \& Segarra(2018)Schaub and Segarra]{schaub2018flow}
Schaub, M.~T. and Segarra, S.
\newblock Flow smoothing and denoising: Graph signal processing in the
  edge-space.
\newblock In \emph{2018 IEEE Global Conference on Signal and Information
  Processing (GlobalSIP)}, pp.\  735--739. IEEE, 2018.

\bibitem[Schaub et~al.(2020)Schaub, Benson, Horn, Lippner, and
  Jadbabaie]{schaub2020random}
Schaub, M.~T., Benson, A.~R., Horn, P., Lippner, G., and Jadbabaie, A.
\newblock Random walks on simplicial complexes and the normalized hodge
  1-laplacian.
\newblock \emph{SIAM Review}, 62\penalty0 (2):\penalty0 353--391, 2020.

\bibitem[Schaub et~al.(2021)Schaub, Zhu, Seby, Roddenberry, and
  Segarra]{schaub2021}
Schaub, M.~T., Zhu, Y., Seby, J.-B., Roddenberry, T.~M., and Segarra, S.
\newblock Signal processing on higher-order networks: Livin’on the edge...
  and beyond.
\newblock \emph{Signal Processing}, 187:\penalty0 108149, 2021.

\bibitem[Shuman et~al.(2013)Shuman, Narang, Frossard, Ortega, and
  Vandergheynst]{shuman2013emerging}
Shuman, D.~I., Narang, S.~K., Frossard, P., Ortega, A., and Vandergheynst, P.
\newblock The emerging field of signal processing on graphs: Extending
  high-dimensional data analysis to networks and other irregular domains.
\newblock \emph{IEEE Signal Processing Magazine}, 30\penalty0 (3):\penalty0
  83--98, 2013.

\bibitem[Sleijpen \& Van~der Vorst(2000)Sleijpen and Van~der
  Vorst]{sleijpen2000jacobi}
Sleijpen, G.~L. and Van~der Vorst, H.~A.
\newblock A jacobi--davidson iteration method for linear eigenvalue problems.
\newblock \emph{SIAM review}, 42\penalty0 (2):\penalty0 267--293, 2000.

\bibitem[Watkins(2007)]{watkins2007matrix}
Watkins, D.~S.
\newblock \emph{The matrix eigenvalue problem: GR and Krylov subspace methods}.
\newblock SIAM, 2007.

\bibitem[Wu et~al.(2019)Wu, Souza, Zhang, Fifty, Yu, and
  Weinberger]{wu2019simplifying}
Wu, F., Souza, A., Zhang, T., Fifty, C., Yu, T., and Weinberger, K.
\newblock Simplifying graph convolutional networks.
\newblock In \emph{International conference on machine learning}, pp.\
  6861--6871. PMLR, 2019.

\bibitem[Wu et~al.(2017)Wu, Chen, Sun, Zheng, and Wang]{wu2017modeling}
Wu, H., Chen, Z., Sun, W., Zheng, B., and Wang, W.
\newblock Modeling trajectories with recurrent neural networks.
\newblock In \emph{Proceedings of the Twenty-Sixth International Joint
  Conference on Artificial Intelligence, {IJCAI-17}}, pp.\  3083--3090, 2017.
\newblock \doi{10.24963/ijcai.2017/430}.
\newblock URL \url{https://doi.org/10.24963/ijcai.2017/430}.

\bibitem[Yang et~al.(2021)Yang, Isufi, Schaub, and Leus]{yang2021finite}
Yang, M., Isufi, E., Schaub, M.~T., and Leus, G.
\newblock Finite {Impulse} {Response} {Filters} for {Simplicial} {Complexes}.
\newblock In \emph{2021 29th {European} {Signal} {Processing} {Conference}
  ({EUSIPCO})}, pp.\  2005--2009, August 2021.
\newblock \doi{10.23919/EUSIPCO54536.2021.9616185}.
\newblock ISSN: 2076-1465.

\bibitem[Yang et~al.(2022{\natexlab{a}})Yang, Isufi, and
  Leus]{yang2021simplicial}
Yang, M., Isufi, E., and Leus, G.
\newblock Simplicial convolutional neural networks.
\newblock In \emph{ICASSP 2022 - 2022 IEEE International Conference on
  Acoustics, Speech and Signal Processing (ICASSP)}, pp.\  8847--8851,
  2022{\natexlab{a}}.
\newblock \doi{10.1109/ICASSP43922.2022.9746017}.

\bibitem[Yang et~al.(2022{\natexlab{b}})Yang, Isufi, Schaub, and
  Leus]{yang2022simplicial}
Yang, M., Isufi, E., Schaub, M.~T., and Leus, G.
\newblock Simplicial convolutional filters.
\newblock \emph{IEEE Transactions on Signal Processing}, 70:\penalty0
  4633--4648, 2022{\natexlab{b}}.
\newblock \doi{10.1109/TSP.2022.3207045}.

\bibitem[Yang et~al.(2022{\natexlab{c}})Yang, Sala, and
  Bogdan]{yang2022efficient}
Yang, R., Sala, F., and Bogdan, P.
\newblock Efficient representation learning for higher-order data with
  simplicial complexes.
\newblock In \emph{The First Learning on Graphs Conference},
  2022{\natexlab{c}}.
\newblock URL \url{https://openreview.net/forum?id=nGqJY4DODN}.

\bibitem[Zhang \& Chen(2018)Zhang and Chen]{zhang2018link}
Zhang, M. and Chen, Y.
\newblock Link prediction based on graph neural networks.
\newblock \emph{Advances in neural information processing systems}, 31, 2018.

\end{thebibliography}
\bibliographystyle{icml2023}

\newpage
\onecolumn
\appendix\newpage\markboth{Appendix}{Appendix}
\renewcommand{\thesection}{\Alph{section}}
\numberwithin{equation}{section}
\numberwithin{figure}{section}
\numberwithin{table}{section}

\section*{Supplementary Material}
In this supplement, we discuss necessary materials to aid the exposition of this paper. We organize them according to the corresponding sections.

\section{Background} \label{app:background}
In  \cref{sec:background} we give the definition of an abstract SC and think of it in terms of a geometric realization, i.e., nodes, edges, triangle faces and so on. While this is helpful to understand, it might not be trivial. Thus, we discuss the geometric SC and show that every abstract SC has a geometric realization. 
\subsection{Abstract Simplicial Complex and its Geometric Realizations} \label{app:a}
If $\bbv_0,\dots,\bbv_k \in \setR^{m}$ are affinely independent, then \emph{the (affine) simplex} $s^k$ spanned by $\bbv_0,\dots,\bbv_k$ is defined to be their convex hull 
\begin{equation*}
  s_{\rm{g}}^k = {\rm{conv}}(\bbv_0,\dots,\bbv_k) = \Bigg\{ \sum_{i=0}^{k}\lambda_i \bbv_i: \lambda_i\geq 0, \sum_{i=0}^k\lambda_i=1 \Bigg\}.
\end{equation*}
  \cref{fig:geometric_simplex} illustrates some simplices. The prototype of a $k$-dimensional simplex is given by the \emph{standard $k$-simplex} $s^k$ defined by the standard basis vectors $e_1,\dots,e_{k_1}\in\setR^{k+1}$: $s_{\rm{g}}^k={\rm{conv}}(e_1,\dots,e_{k+1})\subseteq\setR^{k+1}$.

\begin{figure}[htp!]
  \centering
  \includegraphics[width=0.5\linewidth]{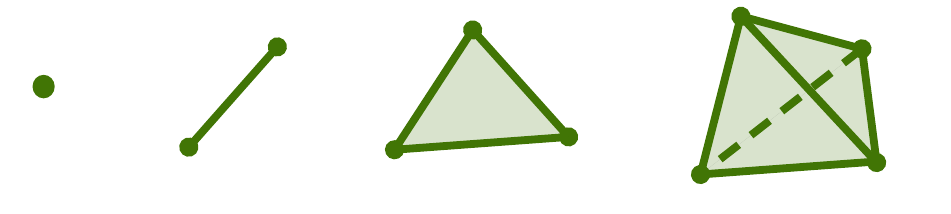}
  \caption{A $0$-simplex is a node (point), a $1$-simplex is a line segment (edge), a $2$-simplex is a (filled) triangle and a $3$-simplex is a (filled) tetrahedron}
  \label{fig:geometric_simplex}
\end{figure}

\begin{definition}
  A \emph{(finite geometric) simplicial complex} $\ccalS_{\rm{g}}$ is a finite set of (affine) simplices in $\setR^{m}$ such that 1) $s\in\ccalS_{\rm{g}}$ and $t\subset s$ implies $t\in\ccalS_{\rm{g}}$ (inclusion property); and 2) if $s_1,s_2\in\ccalS_{\rm{g}}$, then $s_1\cap s_2$ is a face of $s_1$ and a face of $s_2$.

\end{definition}
Any geometric SC $\ccalS_{\rm{g}}$ gives rise to an associated abstract SC. Recall the definition of an abstract SC in  \cref{sec:background}. 
\begin{definition}
  A \emph{simplicial map} of abstract simplicial complexes $\ccalS_1$ and $\ccalS_2$ is defined to be a map between their vertex sets $f:\ccalS_1^0\to \ccalS_2^0$ preserving simplices, i.e., $f(s)\in \ccalS_2$ for each $s\in\ccalS_1$. 
\end{definition} 
\begin{definition}
  A geometric simplicial complex $\ccalS_{\rm{g}}$ is a \emph{geometric realizaiton} of an abstract simplicial complex $\ccalS$ if there exists a bijection $\varphi:\ccalS^0\to\ccalS_{\rm{g}}^0$ such that for any set $\bbv_0,\dots,\bbv_k\in\ccalS^0$ of vertices of $\ccalS$, 
  \begin{equation*}
    \{\bbv_0,\dots,\bbv_k\}\in\ccalS \Longleftrightarrow {\rm{conv}}(\varphi(\bbv_0),\dots,\varphi(\bbv_k))\in\ccalS_{\rm{g}}.
  \end{equation*}
  In other words, $\ccalS_{\rm{g}}$ is a geometric realization of $\ccalS$ if and only if its associated abstract SC is simplicially isomorphic to $\ccalS$. 
\end{definition}
\begin{proposition}
  Any abstract simplicial complex $\ccalS$ of dimension $K$ admits a geometric realization in $\setR^{2K+1}$. 
\end{proposition}
\begin{proposition}
  Any two geometric simplicial complexes are simplicially homeomorphic if and only if the associated abstract simplicial complexes are simpliciallly isomorphic. 
\end{proposition}
The proofs can be found in \citet[P. 177-178]{de2012course} and \citet{munkres2018elements}. They imply that a geometric realization of an abstract SC is unique up to simplicial homeomorphisms.

\subsection{Algebraic Representations of an SC}
Given an oriented SC $\ccalS$ with vertex set $\ccalV=\{1,\dots,N_0\}$, the entries of $\bbB_1\in\setR^{N_0\times N_1}$ and $\bbB_2\in\setR^{N_1\times N_2}$ are given by  
\begin{equation}
  [\bbB_1]_{ie} = 
  \begin{cases}
    -1, &\text{for } e=[i,\cdot]\\
    1, &\text{for } e=[\cdot,i] \\
    0, &\text{otherwise. }
  \end{cases}
  \quad 
  [\bbB_2]_{et} = 
  \begin{cases}
    1, &\text{for } e=[i,j], t=[i,j,k]\\
    -1, &\text{for } e=[i,k], t=[i,j,k]\\
    1, &\text{for } e=[j,k], t=[i,j,k]\\
    0, &\text{otherwise. }
  \end{cases}
\end{equation}
In the following, we show the two incidence matrices of the SC in   \cref{fig:sc_example}. 
\begin{equation}
  \bbB_1 = \kbordermatrix{
      & e_1 & e_2 & e_3 & e_4 & e_5 & e_6 & e_7 & e_8 & e_9 & e_{10} \\
    1 & -1 & -1 & -1 & 0 & 0 & 0 & 0 & 0 & 0 & 0 \\
    2 & 1 & 0 & 0 & -1 & -1 & 0 & 0 & 0 & 0 & 0 \\
    3 & 0 & 1 & 0 & 1 & 0 & -1 & -1 & -1 & 0 & 0 \\
    4 & 0 & 0 & 1 & 0 & 0 & 1 & 0 & 0 & 0 & 0 \\
    5 & 0 & 0 & 0 & 0 & 1 & 0 & 1 & 0 & -1 & -1 \\
    6 & 0 & 0 & 0 & 0 & 0 & 0 & 0 & 1 & 1 & 0 \\
    7 & 0 & 0 & 0 & 0 & 0 & 0 & 0 & 0 & 0 & 1 \\
  }, \quad 
  \bbB_2 = 
  \kbordermatrix{
      & t_1 & t_2 & t_3 \\
    e_1 & 1 & 0 & 0 \\
    e_2 & -1 & 0 & 0 \\
    e_3 & 0 & 0 & 0 \\
    e_4 & 1 & 1 & 0 \\
    e_5 & 0 & -1 & 0 \\
    e_6 & 0 & 0 & 0 \\
    e_7 & 0 & 1 & 1 \\
    e_8 & 0 & 0 & -1 \\
    e_9 & 0 & 0 & 1 \\
    e_{10} & 0 & 0 & 0 \\
  }
\end{equation}

\section{Simplicial Complex Convolutional Neural Networks} \label{app:sccnns}
\subsection{Multi-Feature SCCNN} \label{app:b}
A multi-feature SCCNN at layer $l$ takes $\{\bbX_{k-1}^{l-1},\bbX_{k}^{l-1},\bbX_{k+1}^{l-1}\}$ as inputs, each of which has $F_{l-1}$ features, and generates an output $\bbX_{k}^l$ with $F_l$ features as
\begin{equation}\label{eq.multi_feature_sccnn}
  \bbX_k^l = \sigma \Bigg( \sum_{t=0}^{T_\rmd} \bbL_{k,\rmd}^t \bbB_k^\top \bbX_{k-1}^{l-1} \bbW_{k,\rmd,t}^{\prime l } + \sum_{t=0}^{T_\rmd} \bbL_{k,\rmd}^t \bbX_{k}^{l-1} \bbW_{k,\rmd,t}^{l} + \sum_{t=0}^{T_\rmu} \bbL_{k,\rmu}^t \bbX_{k}^{l-1} \bbW_{k,\rmu,t}^{l} +  \sum_{t=0}^{T_\rmu} \bbL_{k,\rmu}^t \bbB_{k+1}\bbX_{k+1}^{l-1} \bbW_{k,\rmu,t}^{\prime l}\Bigg)
\end{equation}
where $\bbL^t$ indicates the matrix $t$-power of $\bbL$, while superscript $l$ indicates the layer index. 

\subsection{Complexity}\label{app:complexity}
In an SCCNN layer for computing $\bbx_k^l$, there are $2+T_\rmd+T_\rmu$ filter coefficients for the SCF $\bbH_k^l$, and $1+T_\rmd$ and $1+T_\rmu$ for $\bbH_{k,\rmd}^l$ and $\bbH_{k,\rmu}^l$, respectively, which gives the parameter complexity of order $\ccalO(T_\rmd+T_\rmu)$. This complexity will increase by $F_lF_{l-1}$ fold for the multi-feature case, and likewise for the computational complexity. Given the inputs $\{\bbx_{k-1}^{l-1},\bbx_k^{l-1},\bbx_{k+1}^{l-1}\}$, we discuss the computation complexity of $\bbx_k^l$ in \eqref{eq.sccnn_layer}.

First, consider the SCF operation $\bbH_k^l\bbx_k^{l-1}$. As discussed in the localities, it is a composition of $T_\rmd$-step lower and $T_\rmu$-step upper simplicial shiftings. Each simplicial shifting has a computational complexity of order $\ccalO(N_k M_k)$ dependent on the number of neighbors. Thus, this operation has a complexity of order $\ccalO(N_k M_k (T_\rmd+T_\rmu))$.

Second, consider the lower SCF operation $\bbH_{k,\rmd}^l\bbB_k^\top\bbx_{k-1}^{l-1}$. As incidence matrix $\bbB_k$ is sparse, it has $N_k(k+1)$ nonzero entries as each $k$-simplex has $k+1$ faces. This leads to a complexity of order $O(N_k k)$ for operation $\bbB_k^\top\bbx_{k-1}^{l-1}$. Followed by a lower SCF operation, i.e., a $T_\rmd$-step lower simplicial shifting, thus, a complexity of order $\ccalO(k N_k + N_k M_k T_\rmd)$ is needed. 

Third, consider the upper SCF operation $\bbH_{k,\rmu}^l\bbB_{k+1}\bbx_{k+1}^{l-1}$. Likewise, incidence matrix $\bbB_{k+1}$ has $N_{k+1} (k+2)$ nonzero entries. This leads to a complexity of order $\ccalO(N_{k+1}k)$ for the projection operation $\bbB_{k+1}\bbx_{k+1}^{l-1}$. Followed by an upper SCF operation, i.e., a $T_\rmu$-step upper simplicial shifting, thus, a complexity of order $\ccalO(kN_{k+1} + N_k M_kT_\rmu)$ is needed. 

Finally, we have a computational complexity of order $\ccalO(k(N_k+N_{k+1})+N_kM_k(T_{\rmd}+T_{\rmu}))$ in total.

\begin{remark}
  The lower SCF operation $\bbH_{k,\rmd}^l\bbB_k^\top\bbx_{k-1}^{l-1}$ can be further reduced if $N_{k-1}\ll N_k$.  Note that we have 
  \begin{equation}\label{eq.app.complexity}
    \bbH_{k,\rmd}^l\bbB_k^\top\bbx_{k-1}^{l-1} = \sum_{t=0}^{T_\rmd} w_{k,\rmd,t}^{\prime l} \bbL_{k,\rmd}^t \bbB_k^\top\bbx_{k-1}^{l-1}  = \bbB_{k}^\top \sum_{t=0}^{T_\rmd} w_{k,\rmd,t}^{\prime l} \bbL_{k-1,\rmu}^t \bbx_{k-1}^{l-1} ,
  \end{equation}
  where the second equality comes from that $\bbL_{k,\rmd}\bbB_k^\top = \bbB_k^\top \bbB_k \bbB_k^\top = \bbB_k^\top \bbL_{k-1,\rmu}$, $\bbL_{k,\rmd}^2\bbB_k^\top=(\bbB_k^\top\bbB_k)(\bbB_k^\top\bbB_k)\bbB_k^\top = \bbB_k^\top(\bbB_k\bbB_k^\top)(\bbB_k\bbB_k^\top)=\bbB_k^\top\bbL_{k-1,\rmu}$ and likewise for general $t$. Using the RHS of \eqref{eq.app.complexity} where the simplicial shifting is performed in the $(k-1)$-simplicial space, we have a complexity of order $\ccalO(kN_k + N_{k-1}M_{k-1}T_\rmd)$. 
  Similarly, we have 
  \begin{equation}\label{eq.app.complexity_2}
    \bbH_{k,\rmu}^l\bbB_{k+1}\bbx_{k+1}^{l-1} = \sum_{t=0}^{T_\rmu} w_{k,\rmu,t}^{\prime l} \bbL_{k,\rmu}^t \bbB_{k+1}\bbx_{k+1}^{l-1} = \bbB_{k+1} \sum_{t=0}^{T_\rmu} w_{k,\rmu,t}^{\prime l} \bbL_{k+1,\rmd}^t \bbx_{k+1}^{l-1} 
  \end{equation}
  where the simplicial shifting is performed in the $(k+1)$-simplicial space. If it follows that $N_{k+1}\ll N_k$, we have a smaller complexity of $\ccalO(kN_{k+1} + N_{k+1}M_{k+1}T_\rmu)$ by using the RHS of \eqref{eq.app.complexity_2}. 
\end{remark} 

\subsection{Extended Simplicial Locality for SCCNN of order $K=2$} \label{app:c}
We use an SCCNN of order $K=2$ to illustrate the extended simplicial locality in detail. At layer $l$, we have 
\begin{equation}
  \begin{aligned}
    &\bbx_0^l = \sigma(\bbH_0^l\bbx_0^{l-1} + \bbH_{0,\rmu}^{l}\bbB_1\bbx_1^{l-1}) \\ 
    &\bbx_1^l = \sigma(\bbH_{1,\rmd}^l \bbB_1^\top\bbx_{0}^{l-1} + \bbH_{1}^l \bbx_1^{l-1} + \bbH_{1,\rmu}^l \bbB_2 \bbx_{2}^{l-1})\\ 
    &\bbx_2^l = \sigma(\bbH_{2,\rmd}^l \bbB_2^\top \bbx_1^{l-1} + \bbH_2^l \bbx_2^{l-1})
  \end{aligned}
\end{equation}
from which we can see the inter-simplicial locality that each $k$-simplicial signal contain information from its adjacent $(k\pm 1)$-simplices and itself. To see the extended inter-simplicial locality, we focus on the node output, which admits the following expression 
\begin{equation}
  \bbx_0^l = \sigma(\bbH_0\bbx_0^{l-1} + \bbH_{0,\rmu}^{l} \bbB_1 \sigma(\bbH_{1,\rmd}^{l-1} \bbB_1^\top\bbx_{0}^{l-2} + \bbH_{1}^{l-1} \bbx_1^{l-2} + \bbH_{1,\rmu}^{l-1} \bbB_2 \bbx_{2}^{l-2} ) )
\end{equation} 
where the contribution from triangle output $\bbx_2^{l-2}$ at layer $(l-2)$ is contained in $\bbx_0^l$, owing to the fact that $\bbB_1\sigma(\bbH_{1,\rmu}^{l-1}\bbB_2)\neq \mathbf{0}$. This is also illustrated in   \cref{fig:extended_simplicial_locality} where node 1 contains information not only from its neighbors including nodes $\{2,3,4\}$ and edges (cofaces) $\{e_1,e_2,e_3\}$ that contribute to those neighbors, but also from direct triangle $t_1$ and triangles $\{t_2,t_3\}$ further away, which project information to those edges.

\subsection{Related Works in Detail} \label{app:related_works}
We describe how the SCCNN in \eqref{eq.multi_feature_sccnn} generalize other NNs on graphs and SCs in   \cref{tab:related_works}. For simplicity, we use $\bbY$ and $\bbX$ to denote the output and input, respectively, without the index $l$. Note that for GNNs, $\bbL_{0,\rmd}$ is not defined.

\citet{gama2020stability} proposed to build a GNN layer with the form 
\begin{equation}\label{eq.gnn}
  \bbY_0 = \sigma\Bigg(\sum_{t=0}^{T_\rmu}\bbL_{0}^t\bbX_0\bbW_{0,\rmu,t}\Bigg)
\end{equation}
where the convolution step is performed via a graph filter \citep{sandryhaila2013discrete,sandryhaila2014discrete,gama2019convolutional,gama2020graphs}. This GNN can be easily built as a special SCCNN without contributions from edges. Furthermore, \citet{defferrard2017} considered a fast implementation of this GNN via a Chebyshev polynomial, while \citet{wu2019simplifying} simplified this by setting $\bbW_{0,t,\rmu}$ as zeros for $t<T_\rmu$. \citet{kipf2017semi} further simplified this by setting $T_\rmu=1$, namely, GCN. 

\citet{yang2021simplicial} proposed a simplicial convolutional neural network (SCNN) to learn from $k$-simplicial signals
\begin{equation}\label{eq.scnn_yang}
  \bbY_k = \sigma \Bigg( \sum_{t=0}^{T_\rmd} \bbL_{k,\rmd}^t \bbX_{k} \bbW_{k,\rmd,t} + \sum_{t=0}^{T_\rmu} \bbL_{k,\rmu}^t \bbX_{k} \bbW_{k,\rmu,t} \Bigg)
\end{equation}
where the linear operation is also defined as a simplicial convolution filter in \citet{yang2022simplicial}. This is a special SCCNN with a focus on one simplex level without taking into the lower and upper contributions consideration. The simplicial neural network (SNN) of \citet{ebli2020simplicial} did not differentiate the lower and the upper convolutions with a form of 
$
  \bbY_k = \sigma ( \sum_{t=0}^{T}\bbL_{k}^t\bbX_k\bbW_{k,t} ),
$
which leads to a joint processing in the gradient and curl subspaces as analyzed in  \cref{sec:spectral_analysis}. 

While \citet{roddenberry2021principled} proposed an architecture (referred to as PSNN)of a particular form of \eqref{eq.scnn_yang} with $\bbT_\rmd = \bbT_\rmu=1$, performing only a one-step simplicial shifting \eqref{eq.shifting_entry}. \citet{keros2022dist2cycle} also performs a one-step simplicial shifting but with an inverted Hodge Laplacian to localize the homology group in an SC. An attention mechanism was added to both SCNNs and PSNNs by \citet{giusti2022simplicial} and \citet{goh2022simplicial}, respectively. 

To account for the information from adjacent simplices, \citet{bunch2020simplicial} proposed a simplicial 2-complex CNN (S2CCNN)
\begin{equation}\label{eq.bunch_model}
  \begin{aligned}
    &\bbY_0 = \sigma \big( \bbL_0\bbX_0\bbW_{0,\rmu,1} + \bbB_1\bbX_1\bbW_{0,\rmu,0}^\prime \big) \\
    &\bbY_1 = \sigma \big( \bbB_1^\top\bbX_0\bbW_{1,\rmd,0} + \bbL_1\bbX_1\bbW_{1,1} + \bbB_2\bbX_2\bbW_{1,\rmu,0}^\prime  \big) \\
    &\bbY_2 = \sigma \big( \bbB_2^\top\bbX_1\bbW_{2,\rmd,0} + \bbL_{2,\rmu}\bbX_2\bbW_{2,\rmu,1} \big)
  \end{aligned}
\end{equation}
which is limited to SCs of order two. Note that instead of Hodge Laplacians, simplicial adjacency matrices with self-loops are used in \citet{bunch2020simplicial}, which encode equivalent information as setting all filter orders in SCCNNs as one. It is a particular form of the SCCNN where the SCF is a one-step simplicial shifting operation without differentiating the lower and upper shifting, and the lower and upper contributions are simply added, not convolved or shifted by lower and upper SCFs. That is, \citet{bunch2020simplicial} can be obtained from \eqref{eq.sccnn_layer} by setting lower and upper SCFs as identity, $\bbH_{k,\rmd}=\bbH_{k,\rmu}=\bbI$, and setting $w_{k,\rmd,t}=w_{k,\rmu,t}$ and $T_\rmd=T_{\rmu}=1$ for the SCF $\bbH_k$.

The convolution in \citet[eq. 3]{yang2022efficient} is the same as \citet{bunch2020simplicial} though it was performed in a block matrix fashion. The combination of graph shifting and edge shifting in \citet{chen2022bscnets} can be again seen as a special S2CCNN, where the implementation was performed in a block matrix fashion. 
\citet{bodnar2021weisfeiler} proposed a message passing scheme which collects information from one-hop simplicial neighbors and direct faces and cofaces as \citet{bunch2020simplicial} and \citet{yang2022efficient}, but replacing the one-step shifting and projections from (co)faces by some learnable functions. The same message passing was applied for simplicial representation learning by \citet{hajij2021simplicial}. 

Lastly, there are works on signal processing and NNs on cell complexes. For example, \citet{sardellitti2021topological,roddenberry2022signal} generalized the signal processing techniques from SCs to cell complexes, \citet{bodnar2021weisfeiler2,hajij2020cell} performed message passing on cell complexes as in SCs and \citet{hajij2022higher} added the attention mechanism. Cell complexes are a more general model compared to SCs, where $k$-cells compared to $k$-simplices contain any shapes homeomorphic to a $k$-dimensional closed balls in Euclidean space, e.g., a filled polygon is a 2-cell while only triangles are 2-simplices. We refer to \citet{hansen2019toward} for a more formal definition of cell complexes. Despite cell complexes are more powerful to model real-world higher-order structures, SCCNNs can be easily generalized to cell complexes by considering any $k$-cells instead of only $k$-simplices in the algebraic representations, and the theoretical analysis in this paper can be adapted to cell complexes as well. 

\section{Simplicial Complex Symmetries} \label{app:sc_symmetry}

\subsection{Background on Groups}\label{app:d}
In the following, we briefly introduce some definitions in group theory, following from \citet{bronstein2021geometric}.
\begin{definition}[Groups]
  A group is a set $\frakG$ along with a binary operation $\circ:\frakG \times \frakG \to \frakG$ called \emph{composition}, e.g., addition and multiplication of numbers or matrices, or union and intersection of sets, satisfying the following axioms:

1) Associativity: $(\frakg\circ\frakh)\circ\frakk = \frakg\circ(\frakh\circ\frakk)$ for all $\frakg,\frakh,\frakk\in\frakP$.

2) Identity: there exists a unique identity $\frake\in\frakP$ satisfying $\frake\circ\frakg=\frakg\circ\frake=\frakg$ for all $\frakg\in\frakP_k$.

3) Inverse: for each $\frakg\in\frakP$ there is a unique inverse $\frakg^{-1}\in\frakP_k$ such that $\frakg\circ\frakg^{-1}=\frakg^{-1}\circ\frakg=\frake$.

4) Closure: the group is closed under composition, i.e., for every $\frakg,\frakh\in\frakP$, we have $\frakg\circ\frakh\in\frakP$.   
\end{definition}

Consider a signal space $\ccalX(\Omega)$ supported on some underlying domain $\Omega$. A \emph{group action} of $\frakG$ on set $\Omega$ is defined as a mapping $(\frakg,u)\mapsto \frakg.u$ associating a group element $\frakg\in\frakG$ and a point $u\in\Omega$ with some other point $\frakg.u$ on $\Omega$. Group actions admit associativity. Moreover, if we have a group $\frakG$ acting on $\Omega$, we also obtain an action of $\frakG$ on the signal space $\ccalX(\Omega)$: $(\frakg.x)(u) = x(\frakg^{-1}u)$ respecting associativity as well. This group action is linear, which can be described by inducing a map $\rho:\frakG\to \setR^{n\times n}$ that assigns each group element $\frakg$ an (invertible) matrix $\rho(\frakg)$, satisfying the condition $\rho(\frakg \frakh)=\rho(\frakg)\rho(\frakh)$ for all $\frakg,\frakh\in\frakG$. A representation is called unitary or orthogonal if the matrix $\rho(\frakg)$ is unitary or orthogonal for all $\frakg\in\frakG$.

\begin{definition}[Equivariant Functions]
  A function $f:\ccalX(\Omega)\to\ccalY$ is $\frakG$-equivariant if $f(\rho(\frakg)x) = \rho(\frakg)f(x)$ for all $\frakg\in\frakG$ and $x\in\ccalX(\Omega)$, i.e.,  group action on the input affects the output in the same way. 
\end{definition}

\begin{definition}[Invariant Functions]
  We say $f$ is $\frakG$-invariant if $f(\rho(\frakg)x) = f(x)$ for all $\frakg\in\frakG$ and $x\in\ccalX(\Omega)$, i.e., its output is unaffected by the group action on the input. 
\end{definition}

\begin{definition}[Permutation Equivalence and Similarity] \label{def:matrix_similarity}
  Let $\bbA,\bbB\in\setR^{m \times n}$ be two matrices of dimension $m \times n$. If $\bbA = \bbS \bbB \bbT$ with square and nonsingular matrices $\bbS$ and $\bbT$, we say that $\bbA$ is equivalent to $\bbB$. When $\bbS$ and $\bbT$ are unitary, we say $\bbA$ and $\bbB$ are unitarily equivalent. If $\bbS$ and $\bbT$ are two permutation matrices, we say $\bbA$ and $\bbB$ are permutation equivalent. 
  When $m=n$ and $\bbB = \bbS^{-1} \bbA \bbS$ with a nonsingular $\bbS\in\setR^{n\times n}$, we say $\bbB$ is similar to $\bbA$. When $\bbS$ is unitary (or real orthogonal), we say $\bbB$ is unitarily (or real orthogonally) similar to $\bbA$. We say $\bbB$ is permutation similar to $\bbA$ if there is a permutation matrix $\bbP$ such that $\bbB = \bbP \bbA \bbP^\top$.
\end{definition}


\subsection{Proof of   \cref{prop:prop_permutation_symmetry}}
\begin{proof}\label{proof.prop_permutation_symmetry}
  The permuted incidence matrices and Hodge Laplacians can be found by the definition of the permutations $\frakp_{k}$ and $\frakp_{k\pm 1}$, given by $\overline{\bbB}_k = \bbP_k \bbB_k \bbP_{k+1}^\top$ and $\overline{\bbL}_k = \bbP_k \bbL_k \bbP_k^\top$. We mainly show that the algebraic property encoded in $\bbB_k$ and $\bbL_k$ of the SC is preserved after this sequence of permutations. 
  
  Based on   \cref{def:matrix_similarity}, the permuted incidence matrix $\overline{\bbB}_k$ is permutation equivalent to $\bbB_k$, since all permutation matrices $\bbP_{k}$ are real orthogonal. Consider an SVD of $\bbB_k=\bbU_k\bSigma_k\bbV_k^*$ with unitray matrices $\bbU_k$ and $\bbV_k$ of dimensions $N_k\times N_k$ and $N_{k+1}\times N_{k+1}$. We have then
  $
    \overline{\bbB}_k = \bbP_{k-1} \bbU_k \bSigma_k \bbV_k^* \bbP_{k}^\top 
  $
  which is an SVD of $\overline{\bbB}_k$ with unitary matrices $\bbP_{k-1} \bbU_k $  and $ \bbP_{k} \bbV_k $ collecting the left and the right singular vectors. Note that $\bbP_{k-1} \bbU_k $  and $ \bbP_{k} \bbV_k $ are row-permuted versions of $\bbU_k$ and $\bbV_k$ according to the permutations of $k-1$- and $k$-simplices. That is, the singular vectors and singular values of the incidence matrices of the SC remain equivariant before and after permutations. 

  We then consider the permuted Hodge Laplacian $\overline{\bbL}_k$. Matrix $\overline{\bbL}_k$ has same eigenvalues as $\bbL_k$, as they are permutation similar to each other. Consider an eigen-decomposition $\bbL_k=\bbW_k \bLambda_k \bbW_k^\top$. We have then $\overline{\bbL}_k = \bbP_k \bbW_k \bLambda_k \bbW_k^\top \bbP_k^\top$ where matrix $\bbP_k\bbW_k$ collects the eigenvectors in the permuted SC with the same eigenvalues. These eigenvectors are row-permuted versions of $\bbW_k$ according to the permutations of $k$-simplices. That is, the eigenvectors and eigenvalues of the Hodge Laplacians of the SC remain equivariant before and after permutations. 
  Thus, the spectral properties of the SC remain. 
\end{proof}

\subsection{Proof of   \cref{prop:orientation_equivariance}}
\begin{proof}\label{proof.prop_orientation_symmetry}
  Given the representation of the reoriented incidence matrix $\overline{\bbB}_k$ and the reoriented Hodge Laplacian $\overline{\bbL}_k$, we can easily see that they differ from the original ones, as reversing the orientation of $s_i^k$ results in multiplying the $i$th row of $\bbB_k$ and the $i$th column of $\bbB_{k-1}$ by $-1$, as well as the $i$th row and column of $\bbL_k$ by $-1$. Likewise, the corresponding simplicial signal $x_{k,i}$ becomes $-x_{k,i}$. However, the underlying $k$-chain remains unchanged. That is, we have $\overline{c}_k = c_k = \sum_{i=1}^{N_k} x_{k,i}s_{k,i}$ because for each $\overline{s}_i^k=-s_i^k$, we have $\overline{x}_{k,i}=-x_{k,i}$. This can be understood as the simplicial signal remains unchanged in terms of the underlying simplices.  
\end{proof}

\subsection{Proof of   \cref{prop:permutation_equivariance}}
\begin{proof}
  Prior to the permutations by $\{\frakP_k:k=[K]\}$, we have the output of an SCCNN layer on $k$-simplices as 
  \begin{equation}\label{eq:output_xkl}
    \bbx_k^{l} = \sigma(\bbH_{k,\rmd}^{l} \bbB_{k}^\top \bbx_{k-1}^{l-1} + \bbH_{k}^{l}\bbx_k^{l-1} + \bbH_{k,\rmu}^{l} \bbB_{k+1}\bbx_{k+1}^{l-1} )
\end{equation}
with $\bbH^{l}_{k} := \sum_{t=0}^{T_{\rmd}}w^l_{k,\rmd,t}(\bbL_{k,\rmd})^t + \sum_{t=0}^{T_\rmu} w^l_{k,\rmu,t} (\bbL_{k,\rmu})^t, \bbH_{k,\rmd}^l:=\sum_{t=0}^{T_{\rmd}^\prime}w^l_{k,\rmd,t}(\bbL_{k,\rmd})^t$ and $\bbH_{k,\rmu}^l:=\sum_{t=0}^{T_\rmu} w^l_{k,\rmu,t} (\bbL_{k,\rmu})^t$.

When applying $\{\frakP_k:k=[K]\}$ to the SC, we have the input tuple $\{\bbP_{k-1}\bbx_{k-1}^{l-1}, \bbP_{k}\bbx_k^{l-1}, \bbP_{k+1}\bbx_{k+1}^{l-1}\}$. The incidence matrices become $\overline{\bbB}_k = \bbP_{k-1} \bbB_k \bbP_{k}^\top$ and the Hodge Laplacians become $\overline{\bbL}_k = \bbP_k \bbL_k \bbP_{k}^\top$ for all $k$, likewise for their lower and upper counterparts. The SCF $\overline{\bbH}_{k,\rmd}^{l}$ is given by 
\begin{equation} \label{eq:SCF_permuted}
  \begin{aligned}
    \overline{\bbH}_{k,\rmd}^{l} & = \sum_{t=0}^{T_{\rmd}} w^l_{k,\rmd,t}(\bbP_k\bbL_{k,\rmd}\bbP_k^\top)^t + \sum_{t=0}^{T_\rmu} w^l_{k,\rmu,t} (\bbP_k\bbL_{k,\rmu}\bbP_k^\top)^t \\
    & = \sum_{t=0}^{T_{\rmd}}w^l_{k,\rmd,t} \bbP_k(\bbL_{k,\rmd})^t\bbP_k^\top + \sum_{t=0}^{T_\rmu} w^l_{k,\rmu,t} \bbP_k(\bbL_{k,\rmu})^t\bbP_k^\top  = \bbP_k \bbH_{k,\rmd}^l \bbP_k^\top
  \end{aligned}
\end{equation}
owing to fact that $\bbP_k$ is orthogonal. Following the same procedure, we have $\overline{\bbH}_{k,\rmd}^{l} = \bbP_k \bbH_{k,\rmd}^l \bbP_k^\top $ and $\overline{\bbH}_{k,\rmu}^l = \bbP_k \bbH_{k,\rmu}^l \bbP_k^\top$. We can now express the output on the permuted SC as 
\begin{equation} \label{eq.permutation_equivariant}
  \begin{aligned}
    \overline{\bbx}_{k}^{l} & = \sigma(\overline{\bbH}_{k,\rmd}^{l} \overline{\bbB}_{k}^\top \bbP_{k-1} \bbx_{k-1}^{l-1} + \overline{\bbH}_{k}^{l} \bbP_{k} \bbx_k^{l-1} + \overline{\bbH}_{k,\rmu}^{l} \overline{\bbB}_{k+1} \bbP_{k+1}\bbx_{k+1}^{l-1} ) \\ 
    & = \sigma(\bbP_k \bbH_{k,\rmd}^l \bbP_k^\top \bbP_k \bbB_k^\top \bbP_{k-1}^\top \bbP_{k-1} \bbx_{k-1}^{l-1} +  \bbP_k \bbH_{k,\rmd}^l \bbP_k^\top \bbP_k \bbx_k^l +  \bbP_k \bbH_{k,\rmu}^l \bbP_k^\top \bbP_k \bbB_{k+1}\bbP_{k+1}^\top \bbP_{k+1} \bbx_{k+1}^{l-1}) \\
    & = \sigma(\bbP_k \bbH_{k,\rmd}^l \bbB_k^\top \bbx_{k-1}^{l-1} +  \bbP_k \bbH_{k,\rmd}^l  \bbx_k^l +  \bbP_k \bbH_{k,\rmu}^l \bbB_{k+1} \bbx_{k+1}^{l-1}) \\ 
    & = \bbP_k \sigma(\bbH_{k,\rmd}^{l} \bbB_{k}^\top \bbx_{k-1}^{l-1} + \bbH_{k}^{l}\bbx_k^{l-1} + \bbH_{k,\rmu}^{l} \bbB_{k+1}\bbx_{k+1}^{l-1} ) = \bbP_k \bbx_{k}^{l},
  \end{aligned}
\end{equation}
where we use the fact $\bbP_k$ for all $k$ are orthogonal and the elementwise nonlinearity does no affect permutations.
\end{proof}

\subsection{Proof of \cref{prop:orientation_equivariance}}\label{proof:orientation_equivariance}
\begin{proof}
  In an oriented SC, the output of an SCCNN layer on $k$-simplices is given in \eqref{eq:output_xkl}. If the underlying simplices change their orientations according to $\{\frakO_{k,i}:k=[K],i=1,\dots,N_k\}$, then we have the input tuple become $\{\bbD_{k-1}\bbx_{k-1}^{l-1}, \bbD_{k}\bbx_k^{l-1}, \bbD_{k+1}\bbx_{k+1}^{l-1}\}$. The incidence matrices become $\overline{\bbB}_k = \bbD_{k-1} \bbB_k \bbD_{k}^\top$ and the Hodge Laplacians become $\overline{\bbL}_k = \bbD_k \bbL_k \bbD_{k}^\top$ for all $k$, likewise for their lower and upper counterparts. Since $\bbD_k$ is diagonal and orthogonal, we can follow the similar procedure in \eqref{eq.permutation_equivariant}. Thus, we have $\overline{\bbH}_{k,\rmd}^{l} = \bbD_k \bbH_{k,\rmd}^l \bbD_k^\top $ and $\overline{\bbH}_{k,\rmu}^l = \bbD_k \bbH_{k,\rmu}^l \bbD_k^\top$. We can now express the output on the reoriented SC as 
  \begin{equation}
    \begin{aligned}
      \overline{\bbx}_{k}^{l} & = \sigma(\overline{\bbH}_{k,\rmd}^{l} \overline{\bbB}_{k}^\top \bbD_{k-1} \bbx_{k-1}^{l-1} + \overline{\bbH}_{k}^{l} \bbD_{k} \bbx_k^{l-1} + \overline{\bbH}_{k,\rmu}^{l} \overline{\bbB}_{k+1} \bbD_{k+1}\bbx_{k+1}^{l-1} ) \\ 
      & =  \sigma( \bbD_k \bbH_{k,\rmd}^{l} \bbB_{k}^\top \bbx_{k-1}^{l-1} + \bbD_k \bbH_{k}^{l}\bbx_k^{l-1} + \bbD_k \bbH_{k,\rmu}^{l} \bbB_{k+1}\bbx_{k+1}^{l-1} ) \\
      & = \bbD_k \sigma( \bbH_{k,\rmd}^{l} \bbB_{k}^\top \bbx_{k-1}^{l-1} +  \bbH_{k}^{l}\bbx_k^{l-1} + \bbH_{k,\rmu}^{l} \bbB_{k+1}\bbx_{k+1}^{l-1} ) = \bbD_k \bbx_{k}^{l},
    \end{aligned}
  \end{equation}
  where we use the fact that $\bbD_k$ for all $k$ are diagonal and orthogonal for all $k$ and that $\sigma(\cdot)$ is an odd function, $\sigma(-\bbx) = -(\sigma(\bbx))$, such that $\sigma(\bbD_k\bbx_k^l) = \bbD_{k} \sigma(\bbx_k^l)$.
\end{proof}

\section{Spectral Analysis} \label{app:spectral_analysis}
As tools from simplicial signal processing based on Hodge theory are rather unfamiliar, which have been used to analyze simplicial signals by \citet{barbarossa2020} and SCFs by \citet{yang2022simplicial}, Thus, we append the necessary background to assist the exposition of \cref{sec:spectral_analysis}, together with some illustrations.

\subsection{Divergence, Gradient, and Curl Operations on SCs}
We first show how incidence matrices $\bbB_1$ and $\bbB_2$ relate to the physical divergence and curl operations \citep{barbarossa2020}. Consider an edge flow $\bbx_1$. By applying $\bbB_1$ to $\bbx_1$, we obtain a node signal whose $i$th entry on node $i$ can be expressed as 
\begin{equation} \label{eq.div_def}
  [\bbB_1\bbx_1]_i = \sum_{j<i}[\bbx_1]_{[j,i]} - \sum_{i< k}[\bbx_1]_{[i,k]},
\end{equation}
which is the total inflow minus the total outflow at node $i$, i.e., the netflow at node $i$. This operation $\bbB_1\bbx_1$ is also known as the \emph{divergence} operation, reflecting the \emph{irrotational} property of the edge flow. 

Consider the adjoint $\bbB_1^\top$ applied to a node signal $\bbx_0$. We can obtain an edge flow which is the difference between two adjacent node signals, and specifically, on edge $e=[i,j]$ we have 
\begin{equation} \label{eq.grad_def}
  [\bbB_1^\top \bbx_0]_{[i,j]} = [\bbx_0]_{j} - [\bbx_0]_{i}. 
\end{equation}
This is the \emph{gradient} operation and it also explains that $\im(\bbB_1^\top)$ is named as gradient space since it contains edge flows which can be expressed as a gradient of some node signal. 

Likewise, by applying $\bbB_2^\top$ to an edge flow $\bbx_1$, we compute a resulting triangle signal which circulates around triangles. Specifically, at triangle $t=[i,j,k]$, we have
\begin{equation} \label{eq.curl_def}
  [\bbB_2^\top\bbx_1]_{[i,j,k]} = [\bbx_1]_{[i,j]} + [\bbx_1]_{[j,k]} - [\bbx_1]_{[i,k]},
\end{equation}
which is the \emph{curl} operation, reflecting the \emph{solenoidal} property of the edge flow.

Lastly, by applying $\bbB_2$ to a triangle signal $\bbx_2$, we obtain an edge flow $\bbB_2\bbx_2$, which is in the curl space $\im(\bbB_2)$.  

The above operations hold for general $k$-simplicial signal space by using $\bbB_k$ and $\bbB_{k+1}$. Although the physical interpretations like divergence and curl do no generalize, the operation $\bbB_{k}\bbx_{k}$ can be used to measure how a $k$-simplicial signal varies in terms of the lower adjacent simplices (faces), and the operation $\bbB_{k+1}^\top\bbx_k$ reflects how a $k$-simplicial signal varies in terms of the upper adjacent simplices (cofaces). For $k=0$, we only have $\bbB_1^\top\bbx_0$, which measures how a graph signal varies along the edges \citep{sandryhaila2013discrete,sandryhaila2014discrete,shuman2013emerging}. 

\subsection{Decomposition of an Edge Flow}\label{app:hodge_decomposition_illustration}
Given an edge flow $\bbx_1$, the Hodge decomposition in   \cref{thm:hodge decomposition} gives 
\begin{equation}
  \bbx_1 = \bbx_{1,\rm{G}} + \bbx_{1,\rm{C}} + \bbx_{1,\rm{H}}
\end{equation}
with $\bbx_{1,\rm{G}}=\bbB_1^\top\bbx_0$ for some node signal $\bbx_0$ and $\bbx_{1,\rm{C}}=\bbB_2\bbx_2$ for some triangle signal $\bbx_2$. Some examples of physical edge flows can be electric currents, electromagnetic waves, or water flows. We also refer to \citet{jiang2011statistical,candogan2011flows,jia2019graph} for some edge flows generated from real-world problems.

\begin{figure}[htp]
  \centering
  \begin{subfigure}{0.246\linewidth}
    \includegraphics[width=\linewidth]{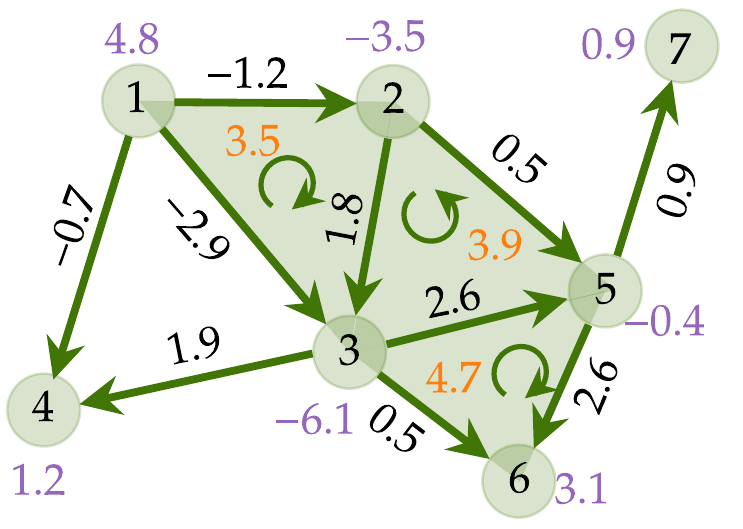}
    \caption{Edge flow}
  \end{subfigure}
  \begin{subfigure}{0.246\linewidth}
    \includegraphics[width=\linewidth]{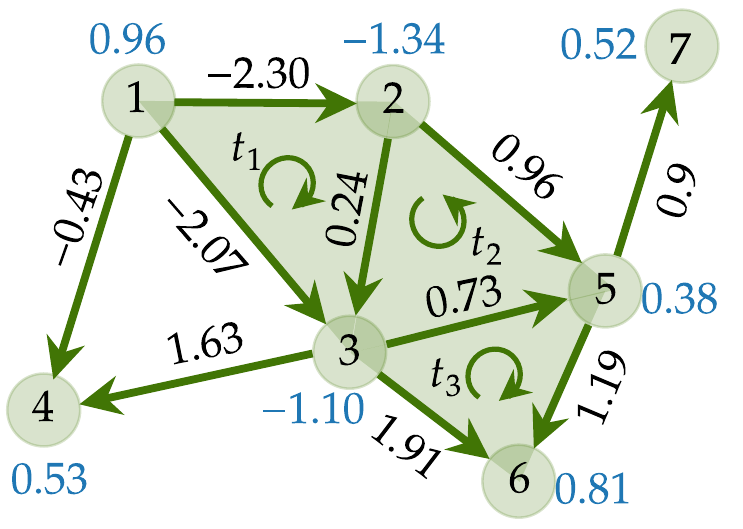}
    \caption{Gradient flow}
    \label{fig:gradient_flow_illustration}
  \end{subfigure}
  \begin{subfigure}{0.246\linewidth}
    \includegraphics[width=\linewidth]{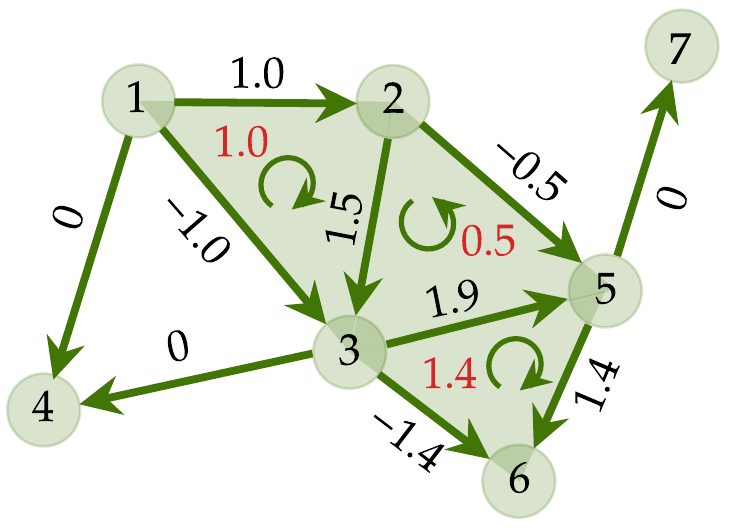}
    \caption{Curl flow}
  \end{subfigure}
  \begin{subfigure}{0.246\linewidth}
    \includegraphics[width=\linewidth]{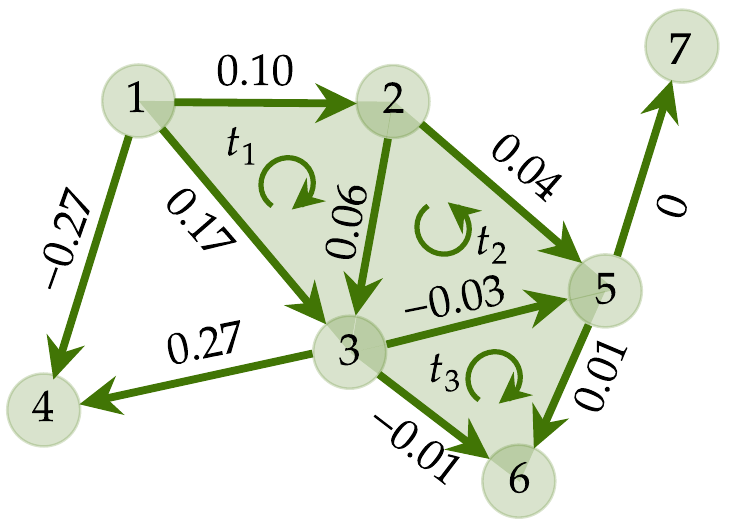}
    \caption{Harmonic flow}
  \end{subfigure}
  \caption{Hodge decomposition of an edge flow. (b)-(d) are the Hodge decomposition of the example edge flow in (a) (we denote its divergence and curl in purple and orange, respectively). The gradient flow is the gradient of some node signal (in blue) and is curl-free. The curl flow can be obtained from some triangle flow (in red), and is divergence-free. The harmonic flow has zero divergence and zero curl, which is circulating around the hole $\{1,3,4\}$. Note that in this figure and   \cref{fig:sft_illustration}, the flow numbers are rounded up to two decimal places. Thus, at some nodes or triangles with zero-divergence or zero-curl, the divergence or curl might not be exactly zero.}
\end{figure}

\subsection{SFT and Simplicial Frequency} \label{app:sft}
\begin{proof}[Proof of   \cref{prop:sft}]
  The proof can be found in \citet{yang2022simplicial}.
\end{proof}
Here we show how the eigenvalues of $\bbL_k$ carry the notion of simplicial frequency. To ease composition, we consider the case of $k=1$. The three sets of eigenvalues follow that: 
\begin{itemize}
  \item \emph{Gradient Frequency:} the nonzero eigenvalues associated with the eigenvectors $\bbU_{1,\rm{G}}$ of $\bbL_{1,\rmd}$, which span the gradient space $\im(\bbB_1^\top)$, admit $\bbL_{1,\rmd}\bbu_{1,\rm{G}} = \lambda_{1,\rm{G}}\bbu_{1,\rm{G}}$ for any eigenpair $\bbu_{1,\rm{G}}$ and $\lambda_{1,\rm{G}}$. From the definition of $\bbL_{1,\rmd}$, we have $\bbB_1^\top\bbB_1 \bbu_{1,\rm{G}} = \lambda_{1,\rm{G}}\bbu_{1,\rm{G}}$, which translates to $\lambda_{1,\rm{G}} = \bbu_{1,\rm{G}}^\top \bbB_1^\top \bbB_1 \bbu_{1,\rm{G}} = \lVert\bbB_1\bbu_{1,\rm{G}}\rVert_2^2$. This is an Euclidean norm of the divergence, i.e., the total nodal variation of $\bbu_{1,\rm{G}}$. If an eigenvector has a larger eigenvalue, it has a larger total divergence. For the SFT of an edge flow, if the gradient embedding $\tilde{\bbx}_{1,\rm{G}}$ has a large weight on such an eigenvector, it contains components with a large divergence, and we say it has a large gradient frequency. Thus, we call such eigenvalues associated with $\bbU_{1,\rm{G}}$ gradient frequencies, measuring the extent of the total divergence of an edge flow. 
  \item \emph{Curl Frequency:} the nonzero eigenvalues associated with the eigenvectors $\bbU_{1,\rm{C}}$ of $\bbL_{1,\rmu}$, which span the curl space $\im(\bbB_2)$, admit  $\bbL_{1,\rmu}\bbu_{1,\rm{C}} = \lambda_{1,\rm{C}}\bbu_{1,\rm{C}}$ for any eigenpair $\bbu_{1,\rm{C}}$ and $\lambda_{1,\rm{C}}$. From the definition of $\bbL_{1,\rmu}$, we have $\bbB_2\bbB_2^\top \bbu_{1,\rm{C}} = \lambda_{1,\rm{C}}\bbu_{1,\rm{C}}$ which translates to $\lambda_{1,\rm{C}} = \bbu_{1,\rm{C}}^\top\bbB_2\bbB_2^\top\bbu_{1,\rm{C}} = \lVert \bbB_2^\top\bbu_{1,\rm{C}} \rVert_2^2$. Eigenvalue $\lambda_{1,\rm{C}}$ then is an Euclidean norm of the curl, i.e., the total rotational variation, of the eigenvector $\bbu_{1,\rm{C}}$. If an eigenvector has a larger eigenvalue, it has a larger total curl. For the SFT of an edge flow, if the curl embedding $\tilde{\bbx}_{1,\rm{C}}$ has a large weight on such an eigenvector, it contains components with a large curl, and we say it has a large curl frequency. Thus. we call such eigenvalues associated with $\bbU_{1,\rm{C}}$ curl frequencies, measuring the extent of the total curl of an edge flow. 
  \item \emph{Harmonic Frequency:} the zero eigenvalues associated with the eigenvectors $\bbU_{1,\rm{H}}$, which span the harmonic space $\ker(\bbL_1)$, admit  $\bbL_1\bbu_{1,\rm{H}} = \mathbf{0}$ for any eigenpair $\bbu_{1,\rm{H}}$ and $\lambda_{1,\rm{H}}=0$. From the definition of $\bbL_1$, we have $\bbB_1\bbu_{1,\rm{H}} = \bbB_2^\top \bbu_{1,\rm{H}}=\mathbf{0}$. That is, the eigenvector $\bbu_{1,\rm{H}}$ has a zero-divergence and a zero-curl, or equivalently, divergence- and curl-free. We also say such an eigenvector has zero signal variation in terms of the nodes and triangles. This resembles the constant graph signal in the node space. We call such zero eigenvalues as harmonic frequencies. 
\end{itemize}

  \cref{fig:sft_illustration} shows the simplicial Fourier basis and the corresponding simplicial frequencies of the SC in  \cref{fig:sc_example}, from which we see how the eigenvalues of $\bbL_1$ can be interpreted as the simplicial frequencies. 

\begin{figure}[htp]
  \centering
  \begin{subfigure}{0.19\linewidth}
    \includegraphics[width=\linewidth]{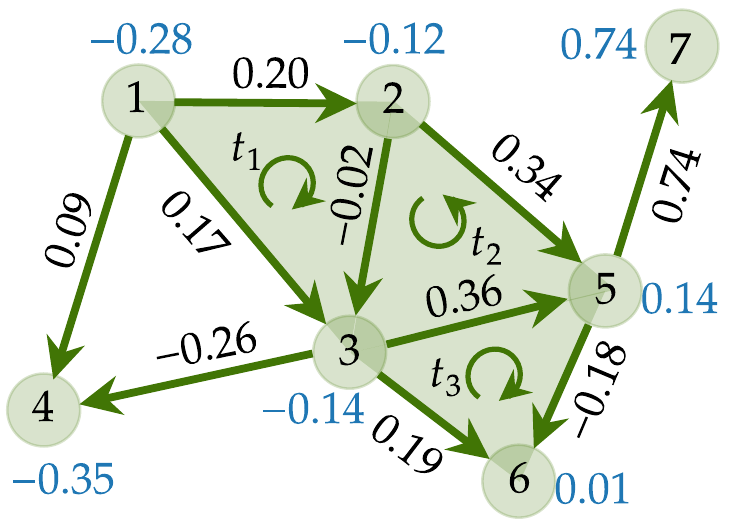}
    \caption{$\bbu_{\rm{G},1},\lambda_{\rm{G},1}=0.80$}
  \end{subfigure}
  \begin{subfigure}{0.19\linewidth}
    \includegraphics[width=\linewidth]{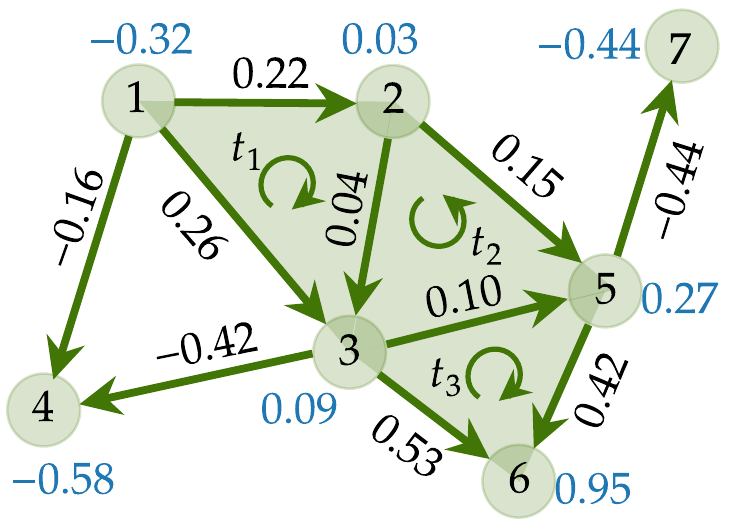}
    \caption{$\bbu_{\rm{G},2},\lambda_{\rm{G},2}=1.61$}
  \end{subfigure}
  \begin{subfigure}{0.19\linewidth}
    \includegraphics[width=\linewidth]{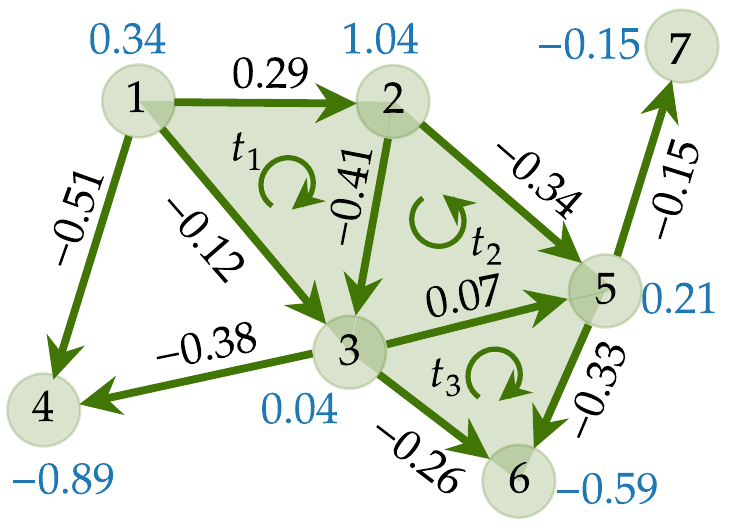}
    \caption{$\bbu_{\rm{G},3},\lambda_{\rm{G},3}=2.43$}
  \end{subfigure}
  \begin{subfigure}{0.19\linewidth}
    \includegraphics[width=\linewidth]{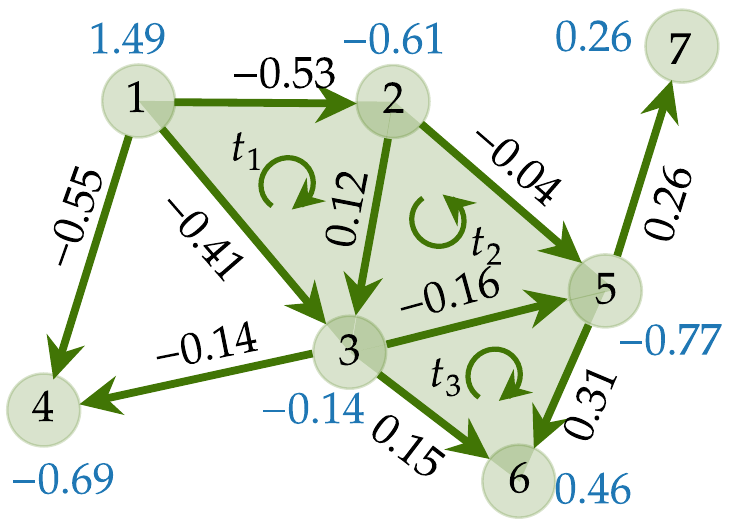}
    \caption{$\bbu_{\rm{G},4},\lambda_{\rm{G},4}=3.96$}
  \end{subfigure}
  \begin{subfigure}{0.19\linewidth}
    \includegraphics[width=\linewidth]{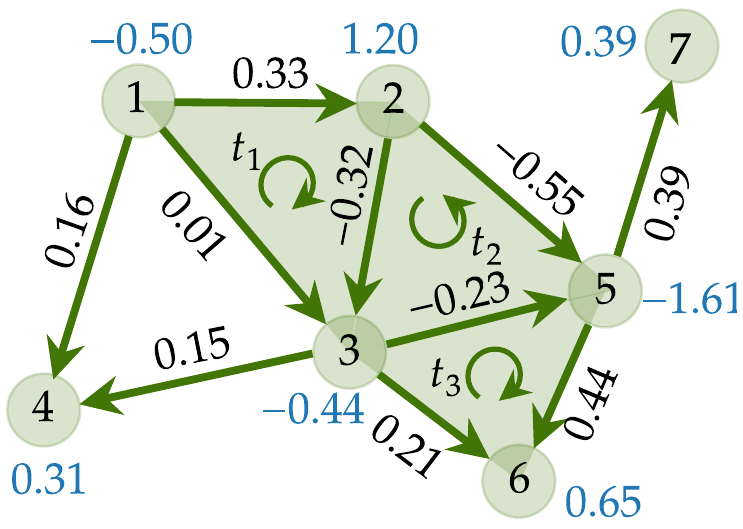}
    \caption{$\bbu_{\rm{G},5},\lambda_{\rm{G},5}=5.12$}
  \end{subfigure}

  \begin{subfigure}{0.19\linewidth}
    \includegraphics[width=\linewidth]{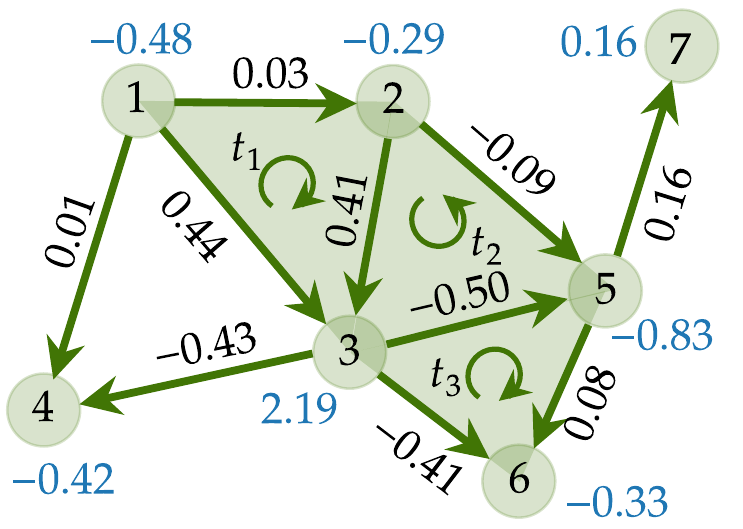}
    \caption{$\bbu_{\rm{G},6},\lambda_{\rm{G},6}=6.08$}
  \end{subfigure}
  \begin{subfigure}{0.19\linewidth}
    \includegraphics[width=\linewidth]{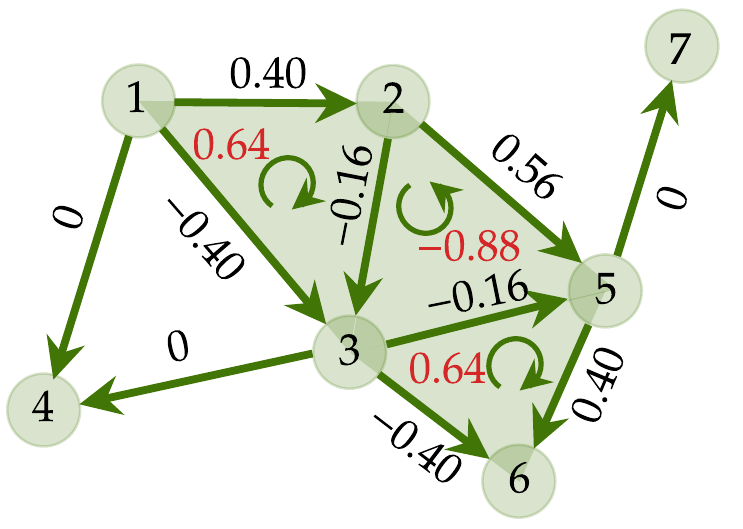}
    \caption{$\bbu_{\rm{C},1},\lambda_{\rm{C},1}=1.59$}
  \end{subfigure}
  \begin{subfigure}{0.19\linewidth}
    \includegraphics[width=\linewidth]{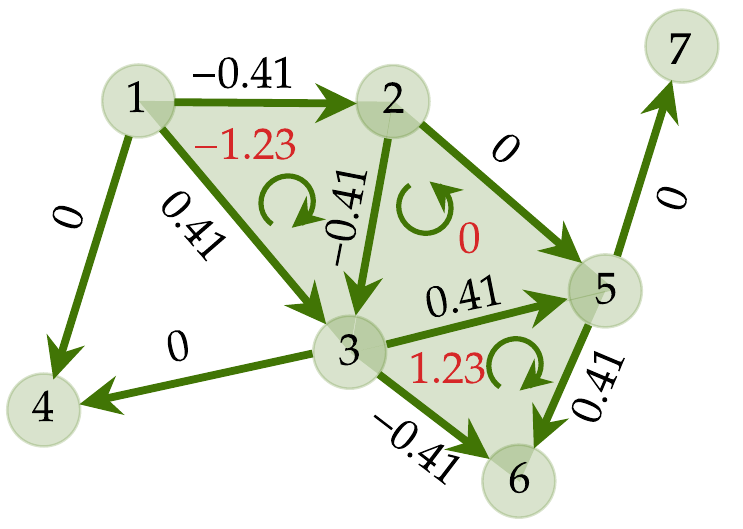}
    \caption{$\bbu_{\rm{C},2},\lambda_{\rm{C},2}=3.00$}
  \end{subfigure}
  \begin{subfigure}{0.19\linewidth}
    \includegraphics[width=\linewidth]{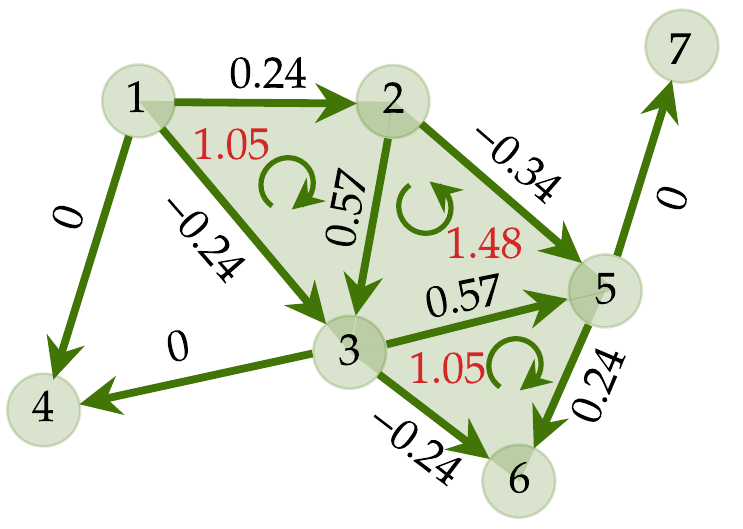}
    \caption{$\bbu_{\rm{C},3},\lambda_{\rm{C},3}=4.41$}
  \end{subfigure}
  \begin{subfigure}{0.19\linewidth}
    \includegraphics[width=\linewidth]{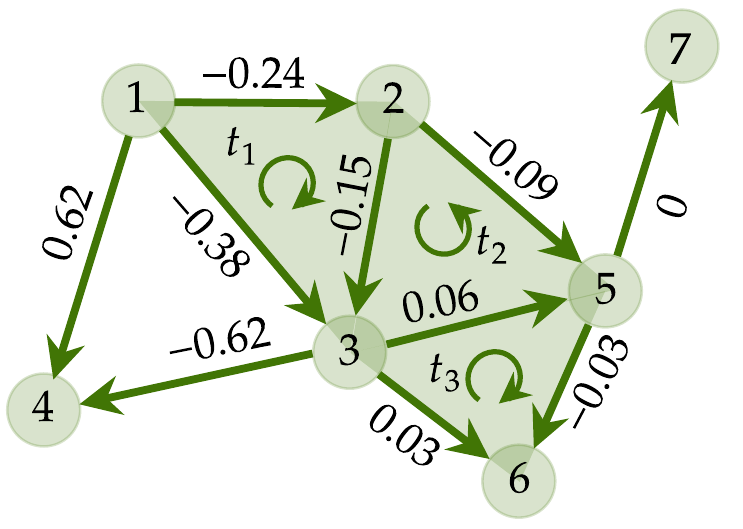}
    \caption{$\bbu_{\rm{H}},\lambda_{\rm{H}}=0$}
  \end{subfigure}
  \caption{Eigenvalues of $\bbL_1$ carry the notion of simplicial frequency. (a)-(f) Six gradient frequencies and the corresponding Fourier basis. We also annotate their divergences, and we see that these eigenvectors with a small eigenvalue have a small magnitude of total divergence, i.e., the edge flow variation in terms of the nodes. Gradient frequencies reflect the nodal variations. (g)-(i) Three curl frequencies and the corresponding Fourier basis. We annotate their curls and we see that these eigenvectors with a small eigenvalue have a small magnitude of total curl, i.e., the edge flow variation in terms of the triangles. Curl frequencies reflect the rotational variations. (j) Harmonic basis with a zero frequency, which has a zero nodal and zero rotational variation. }
  \label{fig:sft_illustration}
\end{figure}

For $k=0$, the eigenvalues of $\bbL_0$ carry the notion of graph frequency, which measures the graph (node) signal smoothness w.r.t. the upper adjacent simplices, i.e., edges. Thus, the curl frequency of $k=0$ coincides with the graph frequency and a constant graph signal has only harmonic frequency component, and there is no divergence frequency. For a more general $k$, there exist these three types of simplicial frequencies, which measure the $k$-simplicial signal total variations in terms of faces and cofaces. 

\subsection{Proof of  \cref{cor:evd_Ld_Lu}}\label{proof:cor:evd_Ld_Lu}

\begin{proof}
Diagonalizing the lower and upper Hodge Laplacians by $\bbU_k = [\bbU_{k,\rm{H}}\,\,\bbU_{k,\rm{G}}\,\,\bbU_{k,\rm{C}}]$, we have 
\begin{equation} \label{eq.decomposition_Ld_Lu}
    \bbL_{k,\rmd} = \bbU_k \diag([\mathbf{0} \,\, \blambda_{k,\rm{G}}^\top \,\, \mathbf{0}]^\top) \bbU_k^\top , \quad  
    \bbL_{k,\rmu} = \bbU_k \diag([\mathbf{0} \,\, \mathbf{0}\,\, \blambda_{k,\rm{C}}^\top ]^\top) \bbU_k^\top 
\end{equation}
where $\diag(\cdot)$ is a diagonal matrix operator of a column vector , $\blambda_{k,\rm{G}}$ and $\blambda_{k,\rm{C}}$ are column vectors collecting the gradient and curl frequencies, i.e., the nonzero eigenvalues, and $\mathbf{0}$ is all-zero vector of appropriate dimension. We can then express the lower and upper simplicial shifting as
\begin{equation*}
  \begin{aligned}
    \bbL_{k,\rmd}\bbx_k & = \bbU_k \diag([\mathbf{0} \,\, \blambda_{k,\rm{G}}^\top \,\, \mathbf{0}]^\top)\bbU_k^\top \bbx_k = \bbU_k \diag([\mathbf{0} \,\, \blambda_{k,\rm{G}}^\top \,\, \mathbf{0}]^\top )\tilde{\bbx}_k,  \\
    \bbL_{k,\rmu}\bbx_k & = \bbU_k \diag([\mathbf{0} \,\, \mathbf{0}\,\, \blambda_{k,\rm{C}}^\top ]^\top)  \bbU_k^\top \bbx_k = \bbU_k \diag([\mathbf{0} \,\, \mathbf{0}\,\, \blambda_{k,\rm{C}}^\top ]^\top)    \tilde{\bbx}_k.
  \end{aligned}
\end{equation*}
By substituting the spectral embedding $\tilde{\bbx}_k = [\tilde{\bbx}_{k,\rm{H}}^\top,\tilde{\bbx}_{k,\rm{G}}^\top,\tilde{\bbx}_{k,\rm{C}}^\top ]^\top $, we have 
\begin{equation*}
  \begin{aligned}
    \bbL_{k,\rmd}\bbx_k & = \bbU_k \diag([\mathbf{0} \,\, \blambda_{k,\rm{G}}^\top \,\, \mathbf{0}]^\top ) [\tilde{\bbx}_{k,\rm{H}}^\top,\tilde{\bbx}_{k,\rm{G}}^\top,\tilde{\bbx}_{k,\rm{C}}^\top ]^\top = \bbU_k \diag(\blambda_{k,\rm{G}}) \tilde{\bbx}_{k,\rm{G}}  = \bbU_{k,\rm{G}} (\blambda_{k,\rm{G}} \odot \tilde{\bbx}_{k,\rm{G}}),  \\
    \bbL_{k,\rmu} \bbx_k & = \bbU_k \diag([\mathbf{0} \,\, \mathbf{0}\,\, \blambda_{k,\rm{C}}^\top ]^\top)[\tilde{\bbx}_{k,\rm{H}}^\top,\tilde{\bbx}_{k,\rm{G}}^\top,\tilde{\bbx}_{k,\rm{C}}^\top ]^\top = \bbU_k \diag(\blambda_{k,\rm{C}}) \tilde{\bbx}_{k,\rm{C}}  = \bbU_{k,\rm{C}} (\blambda_{k,\rm{C}} \odot \tilde{\bbx}_{k,\rm{C}}),
  \end{aligned}
\end{equation*}
where $\odot$ is the Hadamard (elementwise) product.  Combined with  \cref{prop:sft}, the proof is completed. 

\end{proof}

\subsection{Effect of Nonlinearity}\label{app:information_spill_nonlinear}
We use a simple example to illustrate the information spillage effect of a nonlinearity $\sigma(\cdot)$. \cref{fig:information_spill} shows the SFT embeddings of a gradient flow in \cref{fig:gradient_flow_illustration} which contains only nonzero gradient embedding $\tilde{\bbx}_{\rm{G}}$ at the gradient frequencies shown in \cref{fig:sft_illustration}, and harmonic and curl embeddings are zeros at zero and curl frequencies. We pass this flow through a nonlinearity $\rm{tanh}(\cdot)$. We then see that at both zero and curl frequencies contain nonzero embeddings, i.e., the information spillage effect. This allows the information projected to an edge flow from nodes (contained in the gradient embedding) to be propagated to triangles, which is the spectral perspective of the extended inter-simplicial locality.

\begin{figure}[t!]
  \centering
  \includegraphics[width=\linewidth]{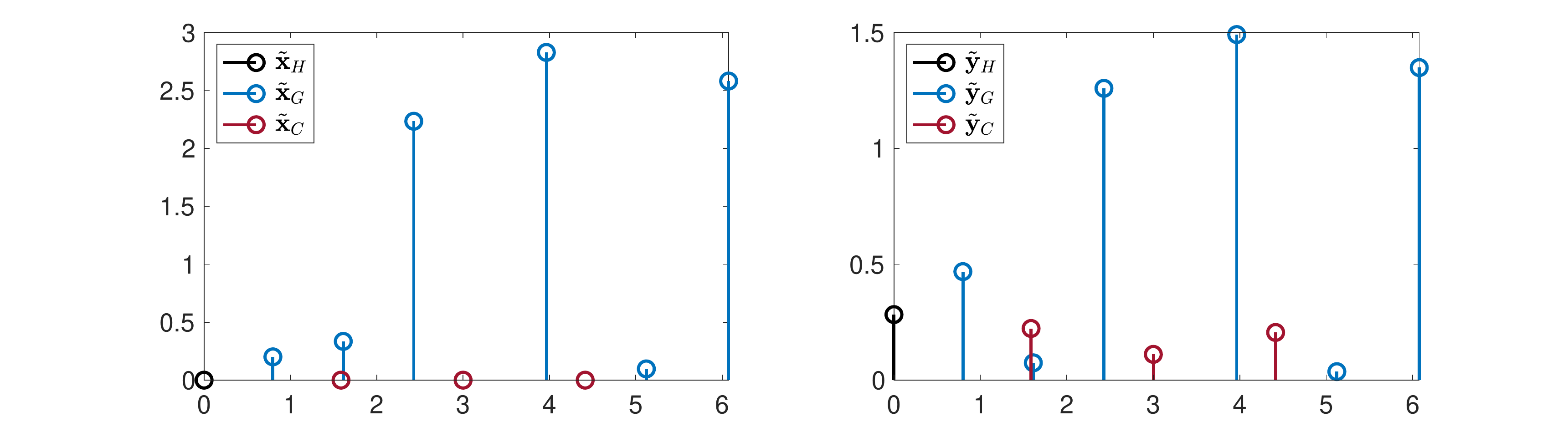}
  \caption{The information spillage effect of a nonlinearity $\sigma(\cdot)$. \emph{(Left)}: the SFT embeddings of the gradient flow ${\bbx}_{\rm{G}}$ in \cref{fig:gradient_flow_illustration}. \emph{(Right)}: the SFT embeddings of $\bby={\rm{tanh}}(\bbx_{\rm{G}})$. We see that information in a gradient flow after passed by a nonlinearity contains information in the harmonic and curl frequencies.}
  \label{fig:information_spill}
\end{figure}

\section{Stability Analysis} \label{app:stability_analysis}
In this section, we discuss the details in  \cref{sec:stability_analysis} to serve the stability of SCCNNs in  \cref{thm:stability}. 

\subsection{SCCNNs in Weighted SCs}\label{app:weighted sccnn}
A weighted SC can be defined through specifying the weights of simplices. We give the definition of a commonly used weighted SC with weighted Hodge Laplacians in \citet{grady2010discrete,horak2013spectra}. 

\begin{definition}[Weighted SC and Hodge Laplacians]\label{def:weighted_laplacian}
  In an oriented and weighted SC, we have diagonal weighting matrices $\bbM_k$ with $[\bbM]_{ii}$ measuring the weight of $i$th $k$-simplex. A weighted $k$th Hodge Laplacian is given by 
  \begin{equation}\label{eq.weighted_laplacian}
    \bbL_k = \bbL_{k,\rmd} + \bbL_{k,\rmu} = \bbM_k \bbB_k^\top \bbM_{k-1}^{-1} \bbB_k +  \bbB_{k+1} \bbM_{k+1} \bbB_{k+1}^\top \bbM_k^{-1}.
  \end{equation}
  where $\bbL_{k,\rmd}$ and $\bbL_{k,\rmu}$ are the weighted lower and upper Laplacians. A symmetric version follows $\bbL_k^s = \bbM_k^{-1/2}\bbL_k\bbM_k^{1/2}$, and likewise, we have $\bbL_{k,\rmd}^s = \bbM_k^{1/2} \bbB_k^\top \bbM_{k-1}^{-1} \bbB_k \bbM_k^{1/2} $ and $\bbL_{k,\rmu}^s = \bbM_k^{-1/2} \bbB_{k+1} \bbM_{k+1} \bbB_{k+1}^\top \bbM_k^{-1/2}$. 
\end{definition}

\textbf{SCCNNs in weighted SC.} The SCCNN layer defined in a weighted SC is of form 
\begin{equation}\label{eq.weighted_sccnn}
  \bbx_k^{l} = \sigma(\bbH_{k,\rmd}^{l} \bbR_{k,\rmd} \bbx_{k-1}^{l-1} + \bbH_{k}^{l}\bbx_k^{l-1} + \bbH_{k,\rmu}^{l} \bbR_{k,\rmu} \bbx_{k+1}^{l-1} )
\end{equation} 
where the three SCFs are defined based on the weighted Laplacians \eqref{eq.weighted_laplacian}, and the lower and upper contributions $\bbx_{k,\rmd}^l$ and $\bbx_{k,\rmu}^l$ are obtained via projection matrices $\bbR_{k,\rmd}\in\setR^{N_k\times N_{k-1}}$ and $\bbR_{k,\rmu}\in\setR^{N_k\times N_{k+1}}$, instead of $\bbB_k^\top$ and $\bbB_{k+1}$.
Projection operator $\bbR_{k,\rmd}$ can be defined via the incidence matrix $\bbB_k^\top$ and weight matrices $\bbM_{k-1}$ and $\bbM_k$. For example, \citet{bunch2020simplicial} considered $\bbR_{1,\rmd} = \bbM_1\bbB_1^\top\bbM_0^{-1}$ and $\bbR_{1,\rmu}=\bbB_2\bbM_2$. 

Although we focus on the unweighted SC to emphasize the mechanism of the SCCNN, the discussions on localities [cf. \cref{sec:sccnns}] and symmetries [cf. \cref{sec:sc_symmetry}] can be extended to a weighted SC, as well as the spectral analysis [cf. \cref{sec:spectral_analysis}] based on a weighted Hodge decomposition 
\begin{equation}
  \ccalX_k = \im(\bbM_k^{1/2}\bbB_k^\top) \oplus \ker(\bbL_k^s) \oplus \im(\bbM_k^{-1/2}\bbB_{k+1}). 
\end{equation}
We refer to \citet{schaub2020random} details. Moreover, the SFT can be likewise defined with the basis given by the eigenvectors of $\bbL_k^s$.   \cref{prop:sft} can be generalized based on the weighted Hodge decomposition and the eigenspace of $\bbL_k^s$. Lastly, the eigenvalues of $\bbL_k^s$ carry the meaning of simplicial frequency, which are the same as those of $\bbL_k$ since they admit a similarity transformation. 

\subsection{Integral Lipschitz SCFs}\label{app:perturbation model}

\begin{definition}[Equivalent Definition of Integral Lipschitz SCF]\label{def.integral_lip_def_2}
  The integral Lipschitz SCF in \cref{def.integral_lip_def_1} is equivalent to that there exist constants $C_{k,\rmd}, C_{k,\rmu}>0$ such that, for $\lambda_1,\lambda_2\geq 0$
  \begin{equation} \label{eq.integral_lip_def_2}
    \begin{aligned}
      |\tilde{h}_{k,\rm{G}}(\lambda_2) - \tilde{h}_{k,\rm{G}}(\lambda_1)|  \leq C_{k,\rmd} 2{|\lambda_2-\lambda_1|}/{|\lambda_2+\lambda_1|}, \quad 
      |\tilde{h}_{k,\rm{C}}(\lambda_2) - \tilde{h}_{k,\rm{C}}(\lambda_1)|  \leq C_{k,\rmu} 2{|\lambda_2-\lambda_1|}/{|\lambda_2+\lambda_1|}.
    \end{aligned}
  \end{equation}
\end{definition}

\begin{figure*}[t!]
  \centering
  \begin{subfigure}{0.45\linewidth}
    \includegraphics[width=\linewidth]{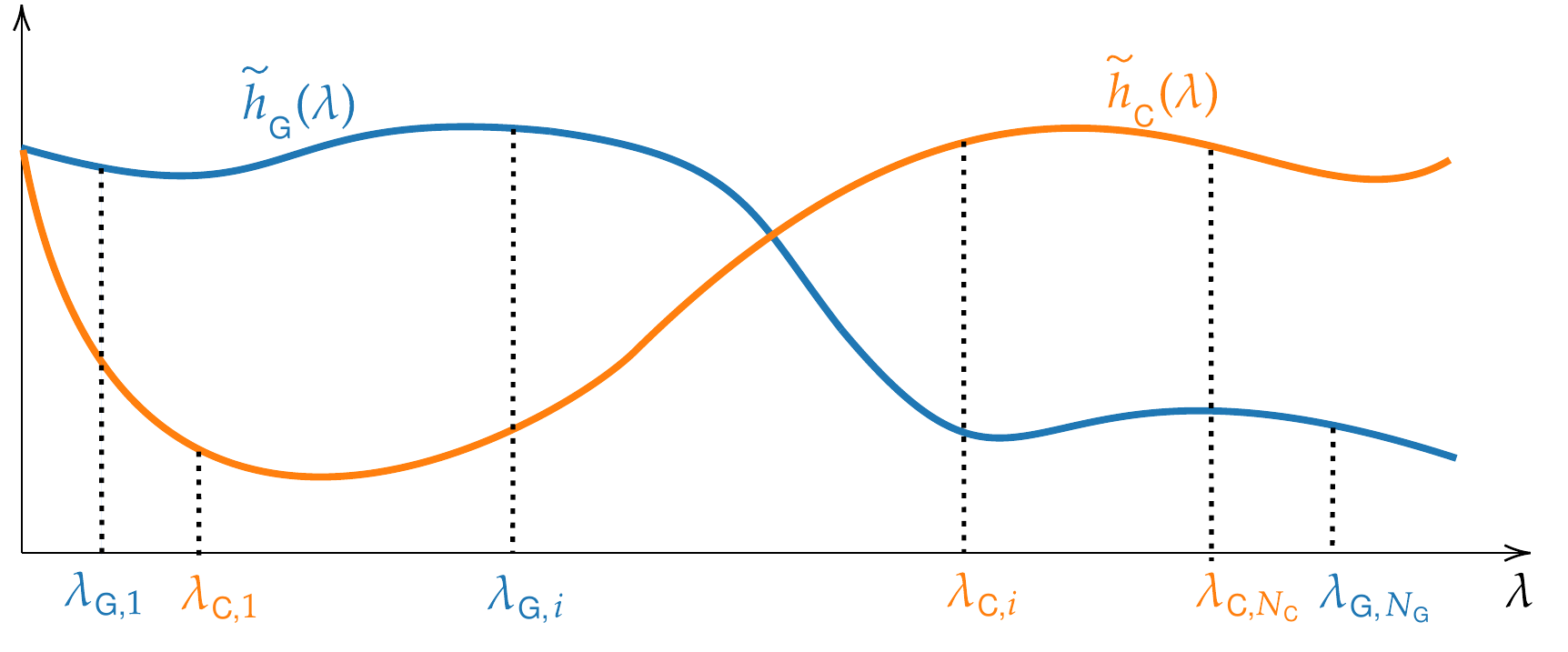}
    \caption{$\bbH$ is an instantiation of functions $\tilde{h}_{\rm{G}}(\lambda)$ and $\tilde{h}_{\rm{C}}(\lambda)$}
    \label{fig:il_illustration1}
  \end{subfigure}
  \begin{subfigure}{0.45\linewidth}
    \includegraphics[width=\linewidth]{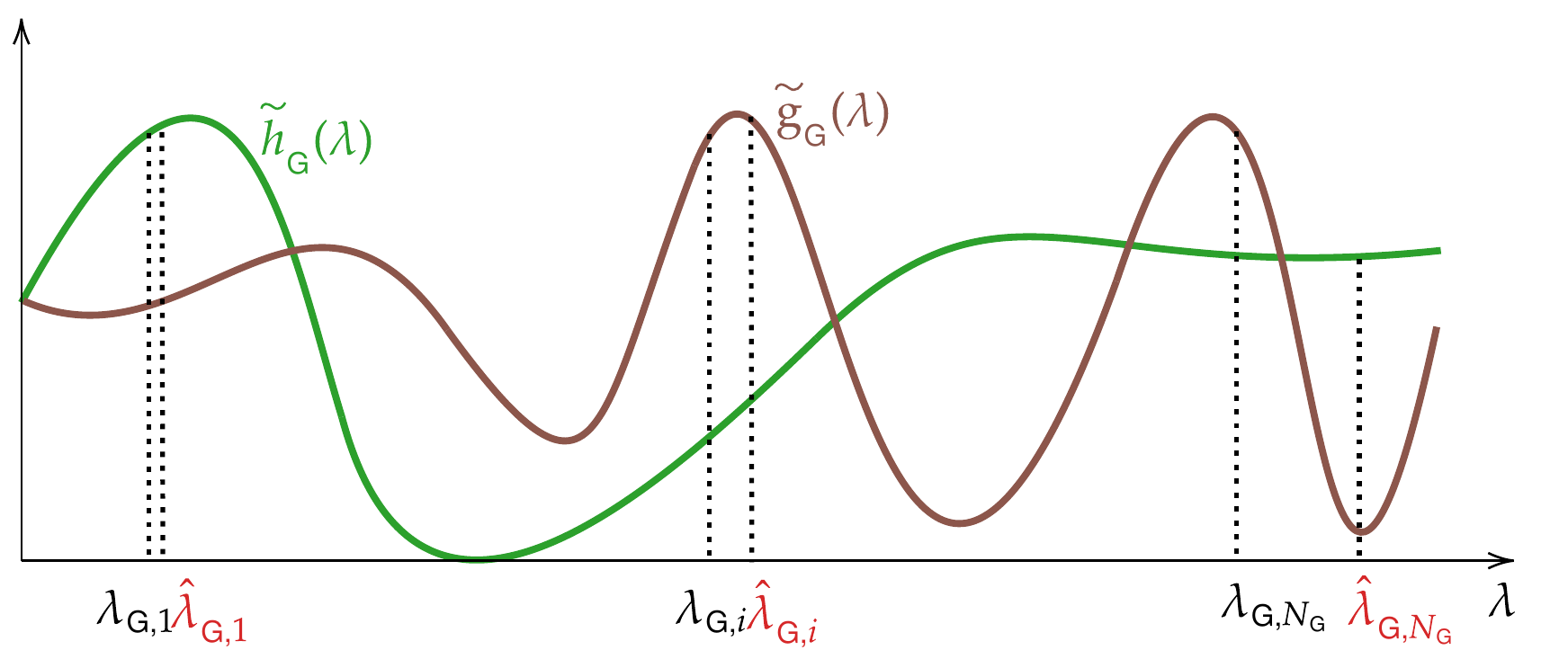}
    \caption{$\tilde{h}_{k,\rm{G}}(\lambda)$ and $\tilde{g}_{k,\rm{C}}(\lambda)$ with different integral Lipschitz}
    \label{fig:il_illustration2}
  \end{subfigure}
  \caption{(a) Given the continuous functions $\tilde{h}_{\rm{G}}(\lambda)$ and $\tilde{h}_{\rm{C}}(\lambda)$, the gradient frequency response $\tilde{\bbh}_{\rm{G}}(\blambda)$ of $\bbH_k$ is an instantiation of $\tilde{h}_{\rm{G}}(\lambda)$ at the discrete gradient frequencies $\blambda=[\lambda_{{\rm{G}},1},\dots,\lambda_{{\rm{G}},N_{\rm{G}}}]$ given by the SC, and likewise for the curl frequency response $\tilde{\bbh}_{\rm{C}}(\blambda)$. (b) Gradient frequency response $\tilde{h}_{\rm{G}}(\lambda)$ has a better integral Lipschitz property than $\tilde{g}_{\rm{G}}(\lambda)$ since the former has a smaller integral Lipschitz constant such that it has smaller variability at larger simplicial frequencies. This results in a better stability of $\tilde{h}_{\rm{G}}(\lambda)$ since $\tilde{h}_{\rm{G}}(\lambda)$ is closer to $\tilde{h}_{\rm{G}}(\hat{\lambda})$ for a relatively perturbed SC.}
\end{figure*}

  \cref{fig:il_illustration1} shows that an SCF $\bbH_k$ is an instantiation of the continuous gradient and curl frequency response functions $\tilde{h}_{k,\rm{G}}(\lambda)$ and $\tilde{h}_{k,\rm{C}}(\lambda)$ where the frequency responses of $\bbH_k$ are instantiated at the discrete simplicial frequencies determined by the SC.   \cref{fig:il_illustration2} shows two gradient frequency response functions with different integral Lipschitz properties, where $\hat{\lambda}_{{\rm{G}},i}=\lambda_{{\rm{G}},i}(1+\epsilon)$ with an error $\epsilon$. For a fixed $\epsilon$, small eigenvalues are less perturbed than large eigenvalues. 
  For an SCF defined by $\tilde{h}_{\rm{G}}(\lambda)$, its frequency response at both small and large eigenvalues does not vary much when eigenvalues are relatively shifted by a small $\epsilon$. However, for an SCF defined by $\tilde{g}_{\rm{G}}(\lambda)$, its frequency response at large eigenvalues changes substantially from $\tilde{g}_{\rm{G}}(\lambda_{{\rm{G}},N_{\rm{G}}})$ to $\tilde{g}_{\rm{G}}(\hat{\lambda}_{{\rm{G},N_{\rm{G}}}})$ as it has a weaker integral Lipschitz property.


\subsection{Proof of Stability of SCCNNs in   \cref{thm:stability}} \label{app:proof_stability}

For an SCCNN in \eqref{eq.weighted_sccnn} in a weighted SC $\ccalS$, we consider its perturbed version in a perturbed SC $\widehat{\ccalS}$ at layer $l$, given by
\begin{equation}
  \hat{\bbx}_k^{l} = \sigma(\widehat{\bbH}_{k,\rmd}^{l} \widehat{\bbR}_{k,\rmd} \hat{\bbx}_{k-1}^{l-1} + \widehat{\bbH}_{k}^{l}\hat{\bbx}_k^{l-1} + \widehat{\bbH}_{k,\rmu}^{l} \widehat{\bbR}_{k,\rmu} \hat{\bbx}_{k+1}^{l-1} )
\end{equation}
which is defined based on perturbed Laplacians with the same set of filter coefficients, and the perturbed projection operators following \cref{def:relative_perturbation}. Given the initial input $\bbx_k^0$ for $k=[K]$, our goal is to upper bound the Euclidean distance between the outputs $\bbx_k^l$ and $\hat{\bbx}_k^l$ for $l=1,\dots,L$,
\begin{equation} \label{proof.eq.distance_xkl}
  \lVert\hat{\bbx}_k^{l}-\bbx_k^l \rVert_2 = \lVert \sigma(\widehat{\bbH}_{k,\rmd}^{l} \widehat{\bbR}_{k,\rmd} \hat{\bbx}_{k-1}^{l-1} - \bbH_{k,\rmd}^{l} \bbR_{k,\rmd} \bbx_{k-1}^{l-1} + \widehat{\bbH}_{k}^{l}\hat{\bbx}_k^{l-1} - \bbH_{k}^{l}\bbx_k^{l-1} + \widehat{\bbH}_{k,\rmu}^{l} \widehat{\bbR}_{k,\rmu} \hat{\bbx}_{k+1}^{l-1} - \bbH_{k,\rmu}^{l} \bbR_{k,\rmu} \bbx_{k+1}^{l-1}  ) \rVert_2.
\end{equation}
We proceed the proof in two steps: first, we analyze the operator norm $\lVert\widehat{\bbH}_k^l-\bbH_k^l\rVert_2$ of an SCF $\bbH_k^l$ and its perturbed version $\widehat{\bbH}_k^l$; then we look for the bound of the output distance for a general $L$-layer SCCNN. 
To ease notations, we omit the subscript such that $\lVert \bbA \rVert = \max_{\lVert\bbx\rVert_2=1} \lVert \bbA \bbx\rVert_2$ is the operator norm (spectral radius) of a matrix $\bbA$, and $\lVert\bbx\rVert$ is the Euclidean norm of a vector $\bbx$.

In the first step we omit the indices $k$ and $l$ for simplicity since they hold for general $k$ and  $l$. We first give a useful lemma. 
\begin{lemma}\label{pf.lemma1}
  Given the $i$th eigenvector $\bbu_i$ of $\bbL = \bbU\bLambda \bbU^\top$, for lower and upper perturbations $\bbE_\rmd$ and $\bbE_\rmu$ in \cref{def:relative_perturbation}, we have  
  \begin{equation}
    \begin{aligned}
      \bbE_{\rmd} \bbu_i = q_{\rmd i}\bbu_i + \bbE_1 \bbu_i, \quad 
      \bbE_{\rmu} \bbu_i = q_{\rmu i}\bbu_i + \bbE_2 \bbu_i 
    \end{aligned}
  \end{equation}
  with eigendecompositions $\bbE_\rmd = \bbV_\rmd \bbQ_\rmd\bbV_\rmd^\top$ and $\bbE_\rmu = \bbV_{\rmu}\bbQ_\rmu\bbV_\rmu^\top$ where $\bbV_\rmd$, $\bbV_\rmu$ collect the eigenvectors and $\bbQ_\rmd$, $\bbQ_\rmu$ the eigenvalues. It holds that  $\lVert\bbE_1 \lVert \leq \epsilon_{\rmd} \delta_{\rmd}$ and $\lVert\bbE_2\lVert \leq \epsilon_{\rmu} \delta_{\rmu}$, with $\delta_{\rmd} = (\lVert\bbV_\rmd -\bbU\lVert+1)^2-1$ and $\delta_{\rmu} = (\lVert\bbV_\rmu - \bbU \lVert+1)^2-1$ measuring the eigenvector misalignments. 
\end{lemma}
\begin{proof}
  We first prove that $\bbE_{\rmd} \bbu_i = q_{\rmd i}\bbu_i + \bbE_1 \bbu_i$. The perturbation matrix on the lower Laplacian can be written as $\bbE_{\rmd}=\bbE_{\rmd}^{'}+ \bbE_1$ with $\bbE_{\rmd}^\prime = \bbU\bbQ_\rmd\bbU^\top$ and $\bbE_1 = (\bbV_{\rmd}-\bbU)\bbQ_\rmd(\bbV_{\rmd}-\bbU)^\top + \bbU\bbQ_\rmd(\bbV_{\rmd}-\bbU)^\top + (\bbV_{\rmd}-\bbU)\bbQ_\rmd\bbU^\top$. For the $i$th eigenvector $\bbu_i$, we have that 
  \begin{equation}
    \bbE_{\rmd} \bbu_i = \bbE_{\rmd}^{'} \bbu_i + \bbE_1 \bbu_i = q_{\rmd i} \bbu_i + \bbE_1 \bbu_i
  \end{equation}
  where the second equality follows from $\bbE_{\rmd}^\prime \bbu_i = q_{\rmd i}\bbu_i$. Since $\lVert\bbE_{\rmd}\lVert\leq \epsilon_{\rmd}$, it follows that $\lVert\bbQ_\rmd\lVert\leq\epsilon_{\rmd}$. Then, applying the triangle inequality, we have that 
  \begin{equation}
    \begin{aligned}
      \lVert\bbE_1\lVert  \leq & \lVert (\bbV_{\rmd}-\bbU)\bbQ_\rmd(\bbV_{\rmd}-\bbU)^\top \lVert  + \lVert \bbU\bbQ_\rmd(\bbV_{\rmd}-\bbU)^\top \lVert
      + \lVert (\bbV_{\rmd}-\bbU)\bbQ_\rmd\bbU \lVert \\ 
      \leq & \lVert \bbV_{\rmd} - \bbU \lVert^2 \lVert\bbQ_\rmd\lVert + 2\lVert\bbV_{\rmd}-\bbU\lVert \lVert \bbQ_\rmd\lVert \lVert\bbU\lVert \leq \epsilon_{\rmd}  \lVert \bbV_{\rmd} - \bbU \lVert^2 + 2\epsilon_{\rmd} \lVert\bbV_{\rmd}-\bbU\lVert \\
      = & \epsilon_{\rmd} ( (\lVert\bbV_{\rmd}-\bbU\lVert+1)^2-1 ) = \epsilon_{\rmd} \delta_{\rmd},
    \end{aligned}
  \end{equation}
  which completes the proof for the lower perturbation matrix. Likewise, we can prove  for $\bbE_{\rmu}\bbu_i$.
\end{proof}

\subsubsection{Step I: Stability of the SCF $\bbH_k^l$}\label{app:proof_stability_step_1}

\begin{proof}
  
\textbf{1. Low-order approximation of $\widehat{\bbH}-\bbH$.}
Given an SCF $\bbH = \sum_{t=0}^{T_\rmd}w_{\rmd,t}\bbLd^t + \sum_{t=0}^{T_\rmu} w_{\rmu,t} \bbLu^t$, we denote its perturbed version by 
$
  \widehat{\bbH} = \sum_{t=0}^{T_\rmd}w_{\rmd,t}\hatLd^t + \sum_{t=0}^{T_\rmu} w_{\rmu,t} \hatLu^t,
$
where the filter coefficients are the same.
The difference between $\bbH$ and $\widehat{\bbH}$ can be expressed as 
\begin{equation} \label{pf.eq.diff_filters}
    \widehat{\bbH} - \bbH = \sum_{t=0}^{T_\rmd}w_{\rmd,t}(\hatLd^t - \bbLd^t) + \sum_{t=0}^{T_\rmu} w_{\rmu,t} (\hatLu^t-\bbLu^t),  
\end{equation}
in which we can compute the first-order Taylor expansion of $\hatLd^t$ as  
\begin{equation} \label{pf.eq.hatldk}
    \hatLd^t  = (\bbLd + \bbEd\bbLd + \bbLd \bbEd)^t = \bbLd^t + \bbD_{{\rmd},t} + \bbC_{\rmd}
\end{equation}
with $\bbD_{{\rmd},t}:= \sum_{r=0}^{t-1} (\bbLd^r\bbEd\bbLd^{t-r} + \bbLd^{r+1}\bbEd\bbLd^{t-r-1})$ parameterized by $t$ and $\bbC_{\rmd}$ following $\lVert \bbC_{\rmd} \lVert \leq \sum_{r=2}^t \binom{t}{r} \lVert \bbEd \bbLd + \bbLd \bbEd \lVert^r \lVert\bbLd\lVert^{t-r} $. Likewise, we can expand $\hatLu^t$ as 
\begin{equation} \label{pf.eq.hatluk}
    \hatLu^t  = (\bbLu + \bbEu\bbLd + \bbLd \bbu)^t = \bbLu^t + \bbD_{{\rmu},t} + \bbC_{\rmu}
\end{equation}
with $\bbD_{{\rmu},t}:= \sum_{r=0}^{t-1} (\bbLu^r\bbEu\bbLu^{t-r} + \bbLu^{r+1}\bbEu\bbLu^{t-r-1})$ parameterized by $t$ and $\bbC_{\rmu}$ following $\lVert\bbC_{\rmu} \lVert \leq \sum_{r=2}^t \binom{t}{r} \lVert \bbEu \bbLu + \bbLu \bbEu \lVert^r \lVert \bbLu \lVert^{t-r}$. Then, by substituting \eqref{pf.eq.hatldk} and \eqref{pf.eq.hatluk} into \eqref{pf.eq.diff_filters}, we have 
\begin{equation} \label{pf.eq.approx_diff_filters}
    \widehat{\bbH} - \bbH = \sum_{t=0}^{T_\rmd}w_{\rmd,t} \bbD_{{\rmd},t} + \sum_{t=0}^{T_\rmu}w_{\rmu,t} \bbD_{{\rmu},t} + \bbF_{\rmd} + \bbF_{\rmu}
\end{equation}
with negligible terms $\lVert \bbF_{\rmd} \lVert = \ccalO(\lVert \bbEd \lVert^2)$ and $\lVert \bbF_{\rmu} \lVert = \ccalO(\lVert \bbEu \lVert^2)$ because perturbations are small and the coefficients of higher-order power terms are the derivatives of analytic functions $\tilde{h}_{\rm{G}}(\lambda)$ and $\tilde{h}_{\rm{C}}(\lambda)$, which are bounded [cf. \cref{def.integral_lip_def_1}]. 

\textbf{2. Spectrum of $(\widehat{\bbH}-\bbH)\bbx$.}
Consider a simplicial signal $\bbx$ with an SFT $\tilde{\bbx} = \bbU^\top\bbx = [\tilde{x}_{1},\dots,\tilde{x}_{N}]$, thus, $\bbx = \sum_{i=1}^{N}\tilde{x}_i\bbu_i$. Then, we study the effect of the difference of the SCFs on a simplicial signal from the spectral perspective via
\begin{equation} \label{pf.eq.hat-h_x}
  \begin{aligned}
    (\widehat{\bbH} - \bbH)\bbx =  \sum_{i=1}^{N}\tilde{x}_i\sum_{t=0}^{T_\rmd}w_{\rmd,t}\bbD_{{\rmd},t}^t\bbu_i  + \sum_{i=1}^{N}\tilde{x}_i\sum_{t=0}^{T_\rmd}w_{\rmu,t}\bbD_{{\rmu},t}^t\bbu_i + \bbF_{\rmd} \bbx +  \bbF_{\rmu} \bbx
  \end{aligned}
\end{equation}  
where we have 
\begin{equation}
  \bbD_{{\rmd},t}^t\bbu_i = \sum_{r=0}^{t-1} (\bbLd^r\bbEd\bbLd^{t-r} + \bbLd^{r+1}\bbEd\bbLd^{t-r-1}) \bbu_i, \text{ and } 
  \bbD_{{\rmu},t}^t\bbu_i = \sum_{r=0}^{t-1} (\bbLu^r\bbEu\bbLu^{t-r} + \bbLu^{r+1}\bbEu\bbLu^{t-r-1}) \bbu_i.
\end{equation}
Since the lower and upper Laplacians admit the eigendecompositions [cf. \eqref{eq.decomposition_Ld_Lu}]
\begin{equation}
  \bbLd\bbu_i = \lambda_{\rmd i}\bbu_i, \quad \bbLu\bbu_i = \lambda_{\rmu i} \bbu_i,
\end{equation}
we can express the terms in \eqref{pf.eq.hat-h_x} as 
\begin{equation} \label{pf.eq.ddkd_term1}
  \begin{aligned}
    \bbLd^r\bbEd\bbLd^{t-r}\bbu_i  = \bbLd^r\bbEd\lambda_{\rmd i}^{t-r}\bbu_i = \lambda_{\rmd i}^{t-r} \bbLd^r(q_{\rmd i}\bbu_i + \bbE_1\bbu_i) = q_{\rmd i} \lambda_{\rmd i}^t  \bbu_i +  \lambda_{\rmd _i}^{t-r}\bbLd^r\bbE_1\bbu_i, 
  \end{aligned}
\end{equation}
where the second equality holds from \cref{pf.lemma1}. Thus, we have 
\begin{equation} \label{pf.eq.ddkd_term2}
  \bbLd^{r+1}\bbEd\bbLd^{t-r-1}\bbu_i = q_{\rmd i} \lambda_{\rmd i}^t\bbu_i + \lambda_{\rmd i}^{t-r-1}\bbLd^{r+1}\bbE_1\bbu_i.
\end{equation}
With the results in \eqref{pf.eq.ddkd_term1} and \eqref{pf.eq.ddkd_term2}, we can write the first term in \eqref{pf.eq.hat-h_x} as 
\begin{equation} \label{pf.eq.hat-h_x_term1}
  \begin{aligned}
     \sum_{i=1}^{N}\tilde{x}_i\sum_{t=0}^{T_\rmd}w_{\rmd,t}\bbD_{{\rmd},t}^t\bbu_i =  
     \underbrace{\sum_{i=1}^{N}\tilde{x}_i\sum_{t=0}^{T_\rmd}w_{\rmd,t} \sum_{r=0}^{t-1} 2q_{\rmd i} \lambda_{\rmd i}^t  \bbu_i}_{\textrm{term 1}}
     + 
     \underbrace{\sum_{i=1}^{N}\tilde{x}_i\sum_{t=0}^{T_\rmd}w_{\rmd,t} \sum_{r=0}^{t-1} (\lambda_{\rmd i}^{t-r}\bbLd^r\bbE_1\bbu_i + \lambda_{\rmd i}^{t-r-1}\bbLd^{r+1}\bbE_1\bbu_i)}_{\textrm{term 2}}.
  \end{aligned}
\end{equation}
Term 1 can be further expanded as 
\begin{equation} \label{pf.eq.hat-h_x_term1_term1}
  \begin{aligned}
    \textrm{term 1} = 2 \sum_{i=1}^{N} \tilde{x}_i q_{\rmd i} \sum_{t=0}^{T_\rmd} t w_{\rmd,t} \lambda_{\rmd i}^t \bbu_i 
     = 2\sum_{i=1}^{N} \tilde{x}_i q_{\rmd i} \lambda_{\rmd i} \tilde{h}_{\rm{G}}^\prime(\lambda_{\rmd i})\bbu_i 
  \end{aligned}
\end{equation}
where we used the fact that $\sum_{t=0}^{T_\rmd}tw_{\rmd,t}\lambda_{\rmd i}^t = \lambda_{\rmd i}\tilde{h}_{\rm{G}}^\prime (\lambda_{\rmd i})$. Using $\bbLd=\bbU\bLambda_{\rmd}\bbU^\top$ we can write term 2 in \eqref{pf.eq.hat-h_x_term1} as 
\begin{equation} \label{pf.eq.hat-h_x_term1_term2}
  \begin{aligned}
    \textrm{term 2} = \sum_{i=1}^{N}\tilde{x}_i \bbU \diag(\bbg_{\rmd i}) \bbU^\top \bbE_1 \bbu_i
  \end{aligned}
\end{equation}
where $\bbg_{\rmd i}\in\setR^{N}$ has the $j$th entry
\begin{equation} \label{pf.eq.hat-h_x_term1_term2_detail}
  \begin{aligned}
    [\bbg_{\rmd i}]_j = \sum_{t=0}^{T_\rmd} w_{\rmd,t}  \sum_{r=0}^{t-1} \bigg( \lambda_{\rmd i}^{t-r} [\bLambda_\rmd]^{r}_{j} + \lambda_{\rmd i}^{t-r-1} [\bLambda_{\rmd}]_j^{r+1} \bigg) 
     = 
    \begin{cases}
      2 \lambda_{\rmd i} \tilde{h}_{\rm{G}}^\prime(\lambda_{\rmd i}) \quad & \text{ for } j=i, \\
      \frac{\lambda_{\rmd i}+\lambda_{\rmd j}}{\lambda_{\rmd i} - \lambda_{\rmd j}} (\tilde{h}_{\rm{G}}(\lambda_{\rmd i}) - \tilde{h}_{\rm{G}}(\lambda_{\rmd j})) & \text{ for } j\neq i.
    \end{cases}
  \end{aligned}
\end{equation}
Now, substituting \eqref{pf.eq.hat-h_x_term1_term1} and \eqref{pf.eq.hat-h_x_term1_term2} into \eqref{pf.eq.hat-h_x_term1}, we have 
\begin{equation} \label{pf.eq.hat-h_x_term1_detail}
  \begin{aligned}
    \sum_{i=1}^{N}\tilde{x}_i\sum_{t=0}^{T_\rmd}w_{\rmd,t}\bbD_{{\rmd},t}^t\bbu_i = 2\sum_{i=1}^{N} \tilde{x}_i q_{\rmd i} \lambda_{\rmd i} \tilde{h}_{\rm{G}}^\prime(\lambda_{\rmd i})\bbu_i + \sum_{i=1}^{N}\tilde{x}_i \bbU \diag(\bbg_{\rmd i}) \bbU^\top \bbE_1 \bbu_i.
  \end{aligned}
\end{equation}

By following the same steps as in \eqref{pf.eq.hat-h_x_term1}-\eqref{pf.eq.hat-h_x_term1_term2_detail}, we can express also the second term in \eqref{pf.eq.hat-h_x} as 
\begin{equation}
  \begin{aligned}
   \sum_{i=1}^{N}\tilde{x}_i\sum_{t=0}^{T_\rmd}w_{\rmu,t}\bbD_{{\rmu},t}^t\bbu_i = 2\sum_{i=1}^{N} \tilde{x}_i q_{\rmu i} \lambda_{\rmu i} \tilde{h}_{\rm{C}}^\prime(\lambda_{\rmu i})\bbu_i + \sum_{i=1}^{N}\tilde{x}_i \bbU \diag(\bbg_{\rmu i}) \bbU^\top \bbE_2 \bbu_i
  \end{aligned}
\end{equation}
where $\bbg_{\rmu i}\in\setR^{N}$ is defined as 
\begin{equation}
  \begin{aligned}
    [\bbg_{\rmu i}]_j  = \sum_{t=0}^{T_\rmd} w_{\rmu,t}  \sum_{r=0}^{t-1} \bigg( \lambda_{\rmu i}^{t-r} [\bLambda_\rmu]^{r}_{j} + \lambda_{\rmu i}^{t-r-1} [\bLambda_{\rmu}]_j^{r+1} \bigg) 
     = 
    \begin{cases}
      2 \lambda_{\rmu i} \tilde{h}_{\rm{C}}^\prime(\lambda_{\rmu i}) \quad & \text{ for } j=i, \\
      \frac{\lambda_{\rmu i}+\lambda_{\rmu j}}{\lambda_{\rmu i} - \lambda_{\rmu j}} (\tilde{h}_{\rm{C}}(\lambda_{\rmu i}) - \tilde{h}_{\rm{C}}(\lambda_{\rmu j})) & \text{ for } j\neq i.
    \end{cases}
  \end{aligned}
\end{equation}

\textbf{3. Bound of $\lVert (\widehat{\bbH} - \bbH)\bbx \rVert$.}
Now we are ready to bound $\lVert (\widehat{\bbH} - \bbH)\bbx \rVert$ based on triangle inequality. First, given the small perturbations $\lVert \bbEd \rVert \leq \epsilon_{\rmd}$ and $\lVert \bbEu \rVert \leq \epsilon_{\rmu}$, we have for the last two terms in \eqref{pf.eq.hat-h_x}
\begin{equation} \label{pf.eq.ineq_Fd_Fu}
  \lVert \bbF_\rmd \bbx \rVert \leq \ccalO(\epsilon_\rmd^2)\lVert\bbx\rVert,\text{ and }  \lVert \bbF_\rmu \bbx \rVert \leq \ccalO(\epsilon_\rmu^2) \lVert \bbx\rVert.
\end{equation}
Second, for the first term  $\lVert\sum_{i=1}^{N}\tilde{x}_i\sum_{t=0}^{T_\rmd}w_{\rmd,t}\bbD_{\rmd}^t\bbu_i\rVert$ in \eqref{pf.eq.hat-h_x}, 
we can bound its two terms in \eqref{pf.eq.hat-h_x_term1_term1} and \eqref{pf.eq.hat-h_x_term1_term2} as 
\begin{equation} \label{pf.eq.ineq_term1} 
  \begin{aligned}
    \bigg\lVert \sum_{i=1}^{N}\tilde{x}_i\sum_{t=0}^{T_\rmd}w_{\rmd,t}\bbD_{{\rmd},t}^t\bbu_i \bigg\rVert \leq \bigg\lVert 2\sum_{i=1}^{N} \tilde{x}_i q_{\rmd i} \lambda_{\rmd i} \tilde{h}_{\rm{G}}^\prime(\lambda_{\rmd i})\bbu_i \bigg\rVert + \bigg\lVert \sum_{i=1}^{N}\tilde{x}_i \bbU \diag(\bbg_{\rmd i}) \bbU^\top \bbE_1 \bbu_i \bigg\rVert.
  \end{aligned}
\end{equation}
For the first term on the RHS of \eqref{pf.eq.ineq_term1}, we can write 
\begin{equation} \label{pf.eq.ineq_term1_term1}
  \begin{aligned}
   \bigg\lVert 2\sum_{i=1}^{N} \tilde{x}_i q_{\rmd i} \lambda_{\rmd i} \tilde{h}_{\rm{G}}^\prime(\lambda_{\rmd i})\bbu_i \bigg\rVert^2 \leq 4  \sum_{i=1}^{N} |\tilde{x}_i|^2 |q_{\rmd i}|^2  | \lambda_{\rmd i}  \tilde{h}_{\rm{G}}^\prime(\lambda_{\rmd i})|^2  \leq 4 \epsilon_d^2  C_\rmd^2   \lVert \bbx \rVert^2,
  \end{aligned}
\end{equation}
which results from, first, $|q_{\rmd i}|\leq \epsilon_{\rmd} = \lVert\bbEd\rVert$ since $q_{\rmd i}$ is an eigenvalue of $\bbEd$; second, the integral Lipschitz property of the SCF $|\lambda \tilde{h}_{\rm{G}}^\prime(\lambda)|\leq C_{\rmd}$; and lastly, the fact that $\sum_{i=1}^{N} |\tilde{x}_i|^2 = \lVert\tilde{\bbx}\rVert^2 = \lVert\bbx\rVert^2$ and $\lVert\bbu_i\rVert^2=1$. We then have 
\begin{equation}
  \bigg\lVert 2\sum_{i=1}^{N} \tilde{x}_i q_{\rmd i} \lambda_{\rmd i} \tilde{h}_{\rm{G}}^\prime(\lambda_{\rmd i})\bbu_i \bigg\rVert \leq 2\epsilon_{\rmd}C_\rmd \lVert \bbx \rVert.
\end{equation}
For the second term in RHS of \eqref{pf.eq.ineq_term1}, we have 
\begin{equation}
  \begin{aligned}
    \bigg\lVert \sum_{i=1}^{N}\tilde{x}_i \bbU  \diag(\bbg_{\rmd i}) \bbU^\top \bbE_1 \bbu_i \bigg\rVert  \leq \sum_{i=1}^{N} |\tilde{x}_i| \lVert\bbU \diag(\bbg_{\rmd i})\bbU^\top \rVert \lVert\bbE_1 \rVert \lVert\bbu_i\rVert,
  \end{aligned}
\end{equation}
which stems from the triangle inequality. 
We further have $\lVert\bbU\diag(\bbg_{\rmd i})\bbU^\top\rVert = \lVert\diag(\bbg_{\rmd i})\rVert\leq 2C_\rmd$ resulting from $\lVert\bbU\rVert=1$ and the $C_\rmd$-integral Lipschitz of $\tilde{h}_{\rm{G}}(\lambda)$ [cf. \cref{def.integral_lip_def_1}]. Moreover, it follows that $\lVert\bbE_1\rVert\leq\epsilon_{\rmd}\delta_{\rmd}$ from   \cref{pf.lemma1}, which results in 
\begin{equation}\label{pf.eq.ineq_term1_term2}
  \bigg\lVert \sum_{i=1}^{N}\tilde{x}_i \bbU \diag(\bbg_{\rmd i}) \bbU^\top \bbE_1 \bbu_i \bigg\rVert \leq 2C_\rmd\epsilon_{\rmd}\delta_{\rmd}\sqrt{N} \lVert \bbx \rVert
\end{equation} 
where we use that $\sum_{i=1}^{N}|\tilde{x}_i| = \lVert\tilde{\bbx}\rVert_1\leq\sqrt{N}\lVert\tilde{\bbx}\rVert = \sqrt{N}\lVert\bbx\rVert$. 
By combining \eqref{pf.eq.ineq_term1_term1} and \eqref{pf.eq.ineq_term1_term2}, we have 
\begin{equation} \label{pf.eq.ineq_term1_final}
  \bigg\lVert \sum_{i=1}^{N}\tilde{x}_i\sum_{t=0}^{T_\rmd}w_{\rmd,t}\bbD_{{\rmd},t}^t\bbu_i \bigg\rVert \leq 2\epsilon_\rmd C_\rmd \lVert\bbx\rVert + 2C_\rmd\epsilon_{\rmd}\delta_{\rmd}\sqrt{N} \lVert \bbx \rVert.
\end{equation}
Analogously, we can show that 
\begin{equation} \label{pf.eq.ineq_term3_final}
  \bigg\lVert \sum_{i=1}^{N}\tilde{x}_i\sum_{t=0}^{T_\rmd}w_{\rmu,t}\bbD_{{\rmu},t}^t\bbu_i\bigg\rVert \leq 2\epsilon_\rmu C_\rmu\lVert\bbx\rVert+2C_\rmu\epsilon_\rmu\delta_\rmu\sqrt{N}\lVert\bbx\rVert.
\end{equation}
Now by combining \eqref{pf.eq.ineq_Fd_Fu}, \eqref{pf.eq.ineq_term1_final} and \eqref{pf.eq.ineq_term3_final}, we can bound $\lVert(\widehat{\bbH}  -\bbH)\bbx \rVert$ as 
\begin{equation}
  \begin{aligned}
    \lVert(\widehat{\bbH}  -\bbH)\bbx \rVert \leq  2\epsilon_\rmd C_\rmd \lVert\bbx\rVert + 2C_\rmd\epsilon_{\rmd}\delta_{\rmd}\sqrt{N} \lVert \bbx \rVert + \ccalO(\epsilon_\rmd^2)\lVert\bbx\rVert  
    + 2\epsilon_\rmu C_\rmu\lVert\bbx\rVert+2C_\rmu\epsilon_\rmu\delta_\rmu\sqrt{N}\lVert\bbx\rVert + \ccalO(\epsilon_\rmu^2)\lVert\bbx\rVert.
  \end{aligned}
\end{equation}
By defining $\Delta_{\rmd}=2(1+\delta_{\rmd}\sqrt{N})$  and $\Delta_{\rmu}=2(1+\delta_{\rmu}\sqrt{N})$, we can obtain that 
\begin{equation} \label{pf.eq.filter_distance_bound}
  \lVert\widehat{\bbH}  -\bbH \rVert \leq C_\rmd\Delta_{\rmd}\epsilon_{\rmd} + C_\rmu\Delta_{\rmu}\epsilon_{\rmu}  + \ccalO(\epsilon_{\rmd}^2) + \ccalO(\epsilon_{\rmu}^2).
\end{equation} 
Thus, we have $\lVert \bbH_k^l - \widehat{\bbH}_k^l \rVert \leq C_{k,\rmd} \Delta_{k,\rmd} \epsilon_{k,\rmd} + C_{k,\rmu}\Delta_{k,\rmu}\epsilon_{k,\rmu}$ with $\Delta_{k,\rmd} = 2(1+\delta_{k,\rmd}\sqrt{N_k})$ and $\Delta_{k,\rmu}=2(1+\delta_{k,\rmu}\sqrt{N_k})$ where we ignore the second and higher order terms on $\epsilon_{k,\rmd}$ and $\epsilon_{k,\rmu}$. Likewise, we have $\lVert \bbH_{k,\rmd}^l - \widehat{\bbH}_{k,\rmd}^l \rVert \leq C_{k,\rmd} \Delta_{k,\rmd} \epsilon_{k,\rmd}$ for the lower SCF and $\lVert \bbH_{k,\rmu}^l - \widehat{\bbH}_{k,\rmu}^l \rVert \leq  C_{k,\rmu} \Delta_{k,\rmu}\epsilon_{k,\rmu}$ for the upper SCF. 
\end{proof}

\subsubsection{Step II: Stability of SCCNNs}\label{app:proof_stability_step_2}
\begin{proof}
Given the initial input $\bbx_k^0$, the Euclidean distance between $\bbx_k^l$ and $\hat{\bbx}_k^l$ at layer $l$ can be bounded by using triangle inequality and the $C_\sigma$-Lipschitz property of $\sigma(\cdot)$ [cf. \cref{assump:lip_nonlinear}] as 
\begin{equation}
  \lVert\hat{\bbx}_k^{l}-\bbx_k^l \rVert_2 \leq  C_\sigma (\phi_{k,\rmd}^l + \phi_k^l +  \phi_{k,\rmu}^l),
\end{equation}
with 
\begin{equation}
  \begin{aligned}
    \phi_{k,\rmd}^l:= &\lVert\widehat{\bbH}_{k,\rmd}^{l} \widehat{\bbR}_{k,\rmd} \hat{\bbx}_{k-1}^{l-1} - \bbH_{k,\rmd}^{l} \bbR_{k,\rmd} \bbx_{k-1}^{l-1}\rVert, \\
    \phi_k^l:= &\lVert\widehat{\bbH}_{k}^{l}\hat{\bbx}_k^{l-1} - \bbH_{k}^{l}\bbx_k^{l-1}\rVert,\\
    \phi_{k,\rmu}^l:= & \lVert\widehat{\bbH}_{k,\rmu}^{l} \widehat{\bbR}_{k,\rmu} \hat{\bbx}_{k+1}^{l-1} - \bbH_{k,\rmu}^{l} \bbR_{k,\rmu} \bbx_{k+1}^{l-1}\rVert.
  \end{aligned}
\end{equation}
We now focus on upper bounding each of the terms. 

\textbf{1. Term $\phi_k^l$.} By subtracting and adding $\widehat{\bbH}_k^l\bbx_k^{l-1}$ within the norm, and using the triangle inequality, we obtain 
\begin{equation} \label{proof.eq.phi_kl}
  \begin{aligned}
    \phi_k^l & \leq  \lVert \widehat{\bbH}_{k}^{l} (\hat{\bbx}_k^{l-1} - \bbx_k^{l-1}) \rVert + \lVert (\widehat{\bbH}_k^l - \bbH_{k}^{l}) \bbx_k^{l-1} \rVert  \leq \lVert \hat{\bbx}_k^{l-1} - \bbx_k^{l-1} \rVert + \lVert \widehat{\bbH}_k^l - \bbH_{k}^{l}\rVert  \lVert \bbx_k^{l-1} \rVert \\ 
    & \leq \lVert \hat{\bbx}_k^{l-1} - \bbx_k^{l-1} \rVert  + (C_{k,\rmd} \Delta_{k,\rmd} \epsilon_{k,\rmd} + C_{k,\rmu}\Delta_{k,\rmu}\epsilon_{k,\rmu})  \lVert \bbx_k^{l-1} \rVert
  \end{aligned}
\end{equation}
where we used the SCF stability in \eqref{pf.eq.filter_distance_bound} and that all SCFs have a normalized bounded frequency response in  \cref{assump:bounded_filters}. 
Note that $\widehat{\bbH}_k^l$ is also characterized by $\tilde{h}_{\rm{G}}(\lambda)$ with the same set of filter coefficients as $\bbH_k^l$.  

\textbf{2. Term $\phi_{k,\rmd}^l$ and $\phi_{k,\rmu}^l$.}
By subtracting and adding a term $ \widehat{\bbH}_{k,\rmd}^l\widehat{\bbR}_{k,\rmd}\bbx_{k-1}^{l-1}$ within the norm, we have 
\begin{equation}\label{pf.eq.phi_kdl}
  \begin{aligned}
    \phi_{k,\rmd}^l & \leq \lVert\widehat{\bbH}_{k,\rmd}^{l} \widehat{\bbR}_{k,\rmd} (\hat{\bbx}_{k-1}^{l-1} - \bbx_{k-1}^{l-1}) \rVert + \lVert (\widehat{\bbH}_{k,\rmd}^l\widehat{\bbR}_{k,\rmd} - \bbH_{k,\rmd}^{l} \bbR_{k,\rmd}) \bbx_{k-1}^{l-1}\rVert  \\ 
    & \leq \lVert \widehat{\bbR}_{k,\rmd} \rVert \lVert\hat{\bbx}_{k-1}^{l-1} - \bbx_{k-1}^{l-1}\rVert + \lVert \widehat{\bbH}_{k,\rmd}^l\widehat{\bbR}_{k,\rmd} - \bbH_{k,\rmd}^{l} \bbR_{k,\rmd}\rVert \lVert \bbx_{k-1}^{l-1}\rVert,
  \end{aligned}
\end{equation}
where we used again triangle inequality and $\lVert \widehat{\bbH}_{k,\rmd}^{l}\rVert\leq 1$ from \cref{assump:bounded_filters}. For the term $\lVert\widehat{\bbR}_{k,\rmd}^l\rVert$,  we have $\lVert\widehat{\bbR}_{k,\rmd}^l\rVert\leq \lVert\bbR_{k,\rmd}^l\rVert + \lVert\bbJ_{k,\rmd}\rVert \lVert \bbR_{k,\rmd}^l\rVert \leq r_{k,\rmd} (1+\vepsilon_{k,\rmd})$ where we used $\lVert\bbR_{k,\rmd}^l\rVert\leq r_{k,\rmd}$ in   \cref{assump:bounded_proj} and $\lVert\bbJ_{k,\rmd}^l\rVert\leq \vepsilon_{k,\rmd}$. For the second term of RHS in \eqref{pf.eq.phi_kdl}, by adding and subtracting $\widehat{\bbH}_{k,\rmd}^l \bbR_{k,\rmd}^l$ we have 
\begin{equation} \label{pf.eq.phi_kdl_midterm}
  \begin{aligned}
    \lVert \widehat{\bbH}_{k,\rmd}^l\widehat{\bbR}_{k,\rmd} - \bbH_{k,\rmd}^{l} \bbR_{k,\rmd}\rVert  & = \lVert \widehat{\bbH}_{k,\rmd}^l\widehat{\bbR}_{k,\rmd} - \widehat{\bbH}_{k,\rmd}^l \bbR_{k,\rmd}^l + \widehat{\bbH}_{k,\rmd}^l \bbR_{k,\rmd}^l - \bbH_{k,\rmd}^{l} \bbR_{k,\rmd}\rVert \\ 
    & \leq  \lVert \widehat{\bbH}_{k,\rmd}^l\rVert \lVert\widehat{\bbR}_{k,\rmd} - \bbR_{k,\rmd}\rVert +  \lVert \widehat{\bbH}_{k,\rmd}^l- \bbH_{k,\rmd}^{l}\rVert \lVert\bbR_{k,\rmd}\rVert  \\ & \leq r_{k,\rmd}\vepsilon_{k,\rmd} + C^\prime_{k,\rmd}\Delta_{k,\rmd}\epsilon_{k,\rmd}r_{k,\rmd} 
  \end{aligned}
\end{equation}
where we use the stability result of the lower SCF $\bbH_{k,\rmd}^l$ in \eqref{pf.eq.filter_distance_bound}. By substituting \eqref{pf.eq.phi_kdl_midterm} into \eqref{pf.eq.phi_kdl}, we have 
\begin{equation}
  \phi_{k,\rmd}^l\leq \hat{r}_{k,\rmd}\lVert\hat{\bbx}_{k-1}^{l-1} - \bbx_{k-1}^{l-1}\rVert  + (r_{k,\rmd}\vepsilon_{k,\rmd} + C^\prime_{k,\rmd}\Delta_{k,\rmd}\epsilon_{k,\rmd}r_{k,\rmd}) \lVert \bbx_{k-1}^{l-1}\rVert.
\end{equation}
By following the same procedure [cf. \eqref{pf.eq.phi_kdl} and \eqref{pf.eq.phi_kdl_midterm}], we obtain
\begin{equation}
  \phi_{k,\rmu}^l \leq \hat{r}_{k,\rmu} \lVert\hat{\bbx}_{k+1}^{l-1} - \bbx_{k+1}^{l-1}\rVert  + (r_{k,\rmu}\vepsilon_{k,\rmu} + C^\prime_{k,\rmu}\Delta_{k,\rmu}\epsilon_{k,\rmu}r_{k,\rmu}) \lVert \bbx_{k+1}^{l-1}\rVert.
\end{equation}

\textbf{3. Bound of $\lVert\hat{\bbx}_k^{l}-\bbx_k^l \rVert$.}
Using the notations $t_k,t_{k,\rmd}$ and $t_{k,\rmu}$ in   \cref{thm:stability}, we then have a set of recursions, for $k=[K]$ 
\begin{equation}
  \begin{aligned}
    \lVert\hat{\bbx}_k^{l}-\bbx_k^l \rVert \leq &  C_\sigma   ( \hat{r}_{k,\rmd}\lVert\hat{\bbx}_{k-1}^{l-1} - \bbx_{k-1}^{l-1}\rVert  + t_{k,\rmd} \lVert \bbx_{k-1}^{l-1}\rVert +  \lVert \hat{\bbx}_k^{l-1} - \bbx_k^{l-1} \rVert + t_k \lVert \bbx_k^{l-1} \rVert \\ & + \hat{r}_{k,\rmu}\lVert\hat{\bbx}_{k+1}^{l-1} - \bbx_{k+1}^{l-1}\rVert  + t_{k,\rmu} \lVert\bbx_{k+1}^{l-1}\rVert ).
  \end{aligned}
\end{equation}
Define vector $\bbb^l$ as $[\bbb^l]_k = \lVert\hat{\bbx}_k^{l}-\bbx_k^l \rVert $ with $\bbb^0 = \mathbf{0}$. Let $\bbeta^l$ collect the energy of all outputs at layer $l$, with $[\bbeta^l]_k:=\lVert\bbx_{k}^{l-1}\rVert$.
We can express the Euclidean distances of all $k$-simplicial signal outputs for $k=[K]$, as 
\begin{equation}\label{pf.eq.d_l_recursion_l-1}
  \bbb^l \preceq C_\sigma \widehat{\bbZ}\bbb^{l-1} + C_\sigma \bbT \bbeta^{l-1}
\end{equation}
where $\preceq$ indicates elementwise smaller than or equal, and we have 
\begin{equation}\label{thm.eq.T_Zhat}
  \bbT = 
  \begin{bmatrix}
    t_0 & t_{0,\rmu} &  &  &  \\ 
    t_{1,\rmd} & t_1 & t_{1,\rmu} & & \\
      & \ddots & \ddots & \ddots & \\ 
      & & t_{K-1,\rmd} & t_{K-1} & t_{K-1,\rmu} \\ 
      & & & t_{K,\rmd} & t_K
  \end{bmatrix} \text{ and }
  \widehat{\bbZ} = 
    \begin{bmatrix}
      1 & \hat{r}_{0,\rmu} &  &  &  \\ 
      \hat{r}_{1,\rmd} & 1 & \hat{r}_{1,\rmu} & & \\
        & \ddots & \ddots & \ddots & \\ 
        & & \hat{r}_{K-1,\rmd} & 1 & \hat{r}_{K-1,\rmu} \\ 
        & & &  \hat{r}_{K,\rmd} & 1 \\ 
    \end{bmatrix}.
\end{equation}
We are now interested in building a recursion for \eqref{pf.eq.d_l_recursion_l-1} for all layers $l$. We start with term $\bbx_k^l$.
Based on its expression in \eqref{eq.weighted_sccnn}, we bound it as 
\begin{equation}
  \begin{aligned}
    \lVert \bbx_{k}^{l}\rVert & \leq C_\sigma (\lVert\bbH_{k,\rmd}^l\rVert \lVert\bbR_{k,\rmd}\rVert\lVert\bbx_{k-1}^{l-1}\rVert + \lVert\bbH_k^l\rVert \lVert\bbx_x^{l-1}\rVert + \lVert\bbH_{k,\rmu}^l\rVert \lVert\bbR_{k,\rmu}\rVert \lVert\bbx_{k+1}^{l-1}\rVert) \\
    & \leq C_\sigma (r_{k,\rmd} \lVert\bbx_{k-1}^{l-1}\rVert + \lVert\bbx_x^{l-1}\rVert + r_{k,\rmu}\lVert\bbx_{k+1}^{l-1}\rVert),
  \end{aligned}
\end{equation}
which holds for $k=[K]$. Thus, it can be expressed in the vector form as $ \bbeta^l \preceq C_\sigma \bbZ \bbeta^{l-1}$, with 
\begin{equation}\label{thm.eq.Z}
  \bbZ = 
  \begin{bmatrix}
    1 & r_{0,\rmu} &  &  &  \\ 
    r_{1,\rmd} & 1 & r_{1,\rmu} & & \\
      & \ddots & \ddots & \ddots & \\ 
      & & r_{K-1,\rmd} & 1 & r_{K-1,\rmu} \\ 
      & & &  r_{K,\rmd} & 1 \\ 
  \end{bmatrix}.
\end{equation}
Similarly, we have $\bbeta^{l-1}\preceq C_\sigma \bbZ \bbeta^{l-2}$, leading to $\bbeta^l\preceq C_\sigma^l \bbZ^l \bbeta^0$ with $\bbeta^0 = \bbeta$ [cf. \cref{assump:bounded_input}]. We can then express the bound \eqref{pf.eq.d_l_recursion_l-1} as 
\begin{equation}
  \bbb^l\preceq C_\sigma \widehat{\bbZ}\bbb^{l-1} + C_\sigma^l \bbT \bbZ^{l-1}\bbeta.
\end{equation}
Thus, we have 
\begin{equation}
  \bbb^0 = \mathbf{0},\,\, \bbb^1 \preceq C_\sigma \bbT\bbeta,\,\, \bbb^2 \preceq C_\sigma^2 (\widehat{\bbZ}\bbT\bbeta + \bbT\bbZ\bbeta), \,\, \bbb^3 \preceq C_\sigma^3 (\widehat{\bbZ}^2\bbT\bbeta + \widehat{\bbZ}\bbT\bbZ\bbeta + \bbT\bbZ^2\bbeta), \,\, \bbb^4 \preceq \dots,
\end{equation}
which, inductively, leads to
\begin{equation}
  \bbb^l \preceq C_\sigma^l \sum_{i=1}^l \widehat{\bbZ}^{i-1}\bbT\bbZ^{l-i}\bbeta.
\end{equation}
Bt setting $l=L$, we obtain the bound $\bbb^L\preceq\bbd = C_\sigma^L \sum_{l=1}^{L}\widehat{\bbZ}^{l-1}\bbT\bbZ^{L-l}\bbeta$ in   \cref{thm:stability}. 
\end{proof}

\subsection{Discussions on Stability of SCCNNs in \cref{thm:stability}}\label{app:discussion_stability}
From the bound $\bbd$, we see how factors $\bbT$ and $\bbZ,\widehat{\bbZ}$ affect the stability of an SCCNN, which further reflects the factors of $k$-simplices and other simplices. For $L=1$, we have $\bbd=C_\sigma \bbT\bbeta$, and $d_k=C_\sigma (t_{k,\rmd}\beta_{k-1}+t_k\beta_k + t_{k+1}\beta_{k+1})$. That is, the stability $d_k$ of $k$-simplicial output depends on $k$-simplicial input energy $\beta_k$ via constant $t_k$ and $k\pm 1$-simplicial input energies via constants $t_{k,\rmd}$ and $t_{k,\rmu}$. The common factors $\Delta_{k,\rmd}$, $\epsilon_{k,\rmd}$ in $t_{k,\rmd}$ and $t_k$ measure the perturbations on lower simplicial adjacency. Specifically, $\Delta_{k,\rmd}$ is a degree of topology misalignment between $\bbL_{k,\rmd}$ and $\widehat{\bbL}_{k,\rmd}$, and $\epsilon_{k,\rmd}$ is a measure of the perturbation magnitude. The factor $C_{k,\rmd}$ is the integral Lipschitz property of the SCFs operating in the gradient frequencies. The factor $r_{k,\rmd}$ in $t_{k,\rmd}$ measures the degree of projections $\bbR_{k,\rmd}$ from $(k-1)$-simplices to $k$-simplices. Discussions can be made on other factors in $t_{k,\rmu}$ $t_k$ likewise. Thus, the stability of $k$-simplicial output is dependent on perturbations on $k$-simplicial adjacencies, integral Lipschitz properties of $k$-SCFs, and projection degrees from $(k\pm 1)$- to $k$-simplices, given fixed inputs after a one-layer SCCNN. 

For $L=2$, we have $\bbd=C_\sigma(\widehat{\bbZ}\bbT\bbeta + \bbT\bbZ\bbeta)$. In an SC order $K=2$, we can find the entries of $\bbT\bbZ\bbeta$ as 
\begin{equation}
  \bbT\bbZ\bbeta = 
  \begin{bmatrix}
    t_0(\beta_0+r_{0,\rmu}\beta_1) + t_{0,\rmu}(r_{1,\rmd}\beta_0+\beta_1+r_{1,\rmu}\beta_2) \\
    t_{1,\rmd}(p_0+r_{0,\rmu}\beta_1) + t_1(r_{1,\rmd}\beta_0+\beta_1+r_{1,\rmu}\beta_2) + t_{1,\rmu}(r_{2,\rmd}\beta_1+\beta_2) \\ 
    t_{2,\rmd}(r_{1,\rmd}\beta_0+\beta_1+r_{1,\rmu}\beta_2) + t_2(r_{2,\rmd}\beta_1+\beta_2) 
  \end{bmatrix}
\end{equation}
from which, we see that the stability of node output depends also on the triangle input energy $\beta_2$,which is scaled by the projection degree $r_{1,\rmu}$ from triangles to edges, then by $t_{0,\rmu}$. Moreover, for the other term $\widehat{\bbZ}\bbT\bbeta$, we have its entries as 
\begin{equation}
  \widehat{\bbZ}\bbT\bbeta = 
  \begin{bmatrix}
    (t_0\beta_0 + t_{0,\rmu}\beta_1) + \hat{r}_{0,\rmu}(t_{1,\rmd}\beta_0+t_1\beta_1 +t_{1,\rmu}\beta_2 )\\ 
    \hat{r}_{1,\rmd}(t_0\beta_0+t_{0,\rmu}\beta_1) + (t_{1,\rmd}\beta_0 +t_1\beta_1 + t_{1,\rmu}\beta_1) + \hat{r}_{1,\rmu}(t_{2,\rmd}\beta_1+t_2\beta_2) \\ 
    \hat{r}_{2,\rmd}(t_{1,\rmd}\beta_0+t_1\beta_1+t_{1,\rmu}\beta_2) + (t_{2,\rmd}\beta_1 +t_2\beta_2)    
  \end{bmatrix}.
\end{equation}
We see that the dependence of node output stability on $\beta_2$ is also mediated through the factors $t_{1,\rmd},t_1$ and $t_{1,\rmu}$, which describe instead perturbations on edge adjacencies, integral Lipschitz properties of edge-SCFs and projection degrees from nodes and triangles to edges. In turn, these induce instability of node outputs through projection perturbations from edges to nodes. Thus, the stability $d_k$ is in general dependent on perturbations, integral Lipschitz properties of SCFs and projections in the $k$-simplicial space (constants $t_{k,\rmd},t_k$ and $t_{k,\rmu}$), as well as those in the $(k\pm 1)$-simplicial spaces through projection perturbations $\hat{r}_{k,\rmd}$ and $\hat{r}_{k,\rmu}$.

When $L=3$, we have $\bbd = C_\sigma^3 (\widehat{\bbZ}^2\bbT\bbeta + \widehat{\bbZ}\bbT\bbZ\bbeta + \bbT\bbZ^2\bbeta)$. We would observe that factors of perturbations on triangle adjacencies and properties of triangle SCFs, measured by $t_{2,\rmd}$ and $t_2$, will further appear in the stability $d_0$ of node output, controlled by the degrees $\hat{r}_{1,\rmu}$ and $\hat{r}_{0,\rmu}$ of upper perturbations on edge and node. 

As the SCCNN becomes deeper, for $k$-simplicial output, the dependence of its stability on perturbations, properties of SCFs and projections in $k$-simplicial space remains, and the dependence on those factors in simplicial spaces of different orders will extend, up to the whole SC. Specifically, $d_k$ has a dependence on factors in (up to) $(k\pm l)$-simplicial space when $L=l+1$.
This mutual stability dependence of different simplicial outputs is the result of the extended inter-simplicial locality of SCCNNs.

\begin{remark}
  The stability analysis of SCCNNs adapts to other NNs on SCs. For the NNs in \citet{yang2021simplicial}, matrices $\bbT,\bbZ$ and $\widehat{\bbZ}$ will contain only the diagonal entries. We can focus on the stability bound $[\bbd]_k$ of $k$-simplicial output, given by $[\bbd]_k\leq C_\sigma^L L t_k [\bbeta]_k$ with $t_k=C_{k,\rmd}\Delta_{k,\rmd}\epsilon_{k,\rmd} + C_{k,\rmu}\Delta_{k,\rmu}\epsilon_{k,\rmu}$. Since this method focuses only on simplices of the same order, it is unaffected by the perturbations from simplices of other orders. The same holds true for the NNs in \citet{ebli2020simplicial} but with $t_k=C_k\Delta_k\epsilon_k$ where we have a relative perturbation model on $\bbL_k$ with a perturbation $\lVert\bbE_k\rVert\leq\epsilon$ and $\Delta_k$ measures the eigenvector misalignment between $\bbL_k$ and $\bbE_k$. Setting $k=0$, we then obtain the stability bound for GNNs in \citet{defferrard2017,gama2020stability}. For the simplicial NNs performing one-step simplicial shifting, e.g., \citet{bunch2020simplicial}, a stability bound can be obtained based on the procedure in \cref{app:proof_stability_step_2} which will be related to the degree of the simplicial shiftings, measured by spectral radii of Hodge Laplacians. This is another limitation of one-step shifting as the filters have only a linear shape affecting their selectivity \cref{cor:evd_Ld_Lu}, compared to higher-order convolutions in SCCNNs and \citet{ebli2020simplicial,yang2021simplicial}. 
\end{remark}


\section{Experiments}
In this section, we provide more details about the two experiments.\footnote{Code available at \url{https://github.com/cookbook-ms/Learning_on_SCs}.} All of the experiments were run on a single NVIDIA A40 GPU with 48 GB of memory using CUDA 11.5.

\subsection{Simplex Prediction} \label{app:experiments_simplex_pred}

Link prediction on graphs is to predict whether two nodes in a network are likely to form a link. Methods based on GNNs have shown better performance compared to the classical heuristic and latent feature based approaches \citep{kipf2016variational,zhang2018link}. 
In higher-order network models, there is a growing need for the task of higher-order link prediction. This task is relevant in, e.g., predicting new groups of friends in social networks, inferring new relationships between genes and diseases, or recommending online thread participation for a group of users \citep{benson2018simplicial}. 
Here we consider the task of simplex prediction:\emph{ given all the $(k-1)$-simplices in a set of $k+1$ nodes, to predict if this set will be closed to form a $k$-simplex}. 
We can view this set of $k+1$ nodes as an open $k$-simplex. For example, in 2-simplex (triangle) prediction, we predict if an open triangle $t=[i,j,k]$ will form a 2-simplex (a closed triangle), given its three edges $[i,j],[j,k]$ and $[i,k]$ present in the SC. 
Although \citet{benson2018simplicial} extended traditional methods for link prediction to SCs for a similar task, there is limited research on learning-based approaches, especially in SCs. Therefore, we aim to address 2- and 3-simplex predictions based on SCCNNs in this experiment. 

\subsubsection{Method}
We propose the following method for simplex prediction, which is generalized from link prediction based on GNNs by \citet{zhang2018link}: For $k$-simplex prediction, we use an SCCNN in an SC of order $k$ to first learn the features of lower-order simplices up to order $k-1$. Then, we concatenate these embedded lower-order simplicial features and input them to a multilayer perceptron (MLP) which predicts if a $k$-simplex is positive (closed, shall be included in the SC) or negative (open, not included in the SC). 

Specifically, in 2-simplex prediction, consider an SC of order two, which is built based on nodes, edges and (existing positive) triangles. Given the initial inputs on nodes $\bbx_0$ and on edges $\bbx_1$ and zero inputs on triangles $\bbx_2=\mathbf{0}$ since we assume no prior knowledge on triangles, for an open triangle $t=[i,j,k]$, 
an SCCNN is used to learn features on nodes and edges (denoted by $\bby$). Then, we input the concatenation of the features on three nodes or three edges to an MLP, i.e., ${\rm{MLP_{node}}}([\bby_0]_i\Vert[\bby_0]_j\Vert[\bby_0]_k)$ or ${\rm{MLP}_{edge}}([\bby]_{[i,j]}\Vert [\bby]_{[j,k]} \Vert [\bby]_{[i,k]})$, to predict if triangle $t$ is positive or negative. An MLP taking both node and edge features is possible, but we keep it on one simplex level for complexity purposes. Similarly, we consider an SCCNN in an SC of order three for 3-simplex prediction, which is followed by an MLP operating on either nodes, edges or triangles.



\subsubsection{Data Preprocessing}
We consider the data from the Semantic Scholar Open Research Corpus \cite{ammar2018construction} to construct a coauthorship complex where nodes are authors and collaborations between $k$-author are represented by $(k-1)$-simplices. Following the preprocessing in \cite{ebli2020simplicial}, we obtain 352 nodes, 1472 edges, 3285 triangles, 5019 tetrahedrons (3-simplices) and a number of other higher-order simplices. 
The node signal $\bbx_0$, edge flow $\bbx_1$ and triangle flow $\bbx_2$ are the numbers of citations of single author papers and the collaborations of two and three authors, respectively. 

For the 2-simplex prediction, we use the collaboration impact (the number of citations) to split the total set of triangles into the positive set $\ccalT_P = \{ t|[\bbx_2]_t > 7 \}$ containing 1482 closed triangles and the negative set $\ccalT_N = \{ t|[\bbx_2]_t \leq 7 \} $ containing 1803 open triangles such that we have balanced positive and negative samples.  We further split the $80\%$ of the positive triangle set for training,  $10\%$ for validation and $10\%$ for testing; likewise for the negative triangle set. Note that in the construction of the SC, i.e., the incidence matrix $\bbB_2$, Hodge Laplacians $\bbL_{1,\rmu}$ and $\bbL_{2,\rmd}$, we ought to remove negative triangles in the training set and all triangles in the test set. That is, for 2-simplex prediction, we only make use of the training set of the positive triangles since the negative ones are not in the SC. 

Similarly, we prepare the dataset for 3-simplex (tetrahedron) prediction, amounting to the tetradic collaboration prediction. We obtain balanced positive and negative tetrahedron sets based on the citation signal $\bbx_3$. In the construction of $\bbB_3$, $\bbL_{2,\rmu}$ and $\bbL_{3,\rmd}$, we again only use the tetrahedrons in the positive training set. 

\subsubsection{Models}
For comparison, we first use heuristic methods proposed in \citet{benson2018simplicial} as baselines to determine if a triangle $t=[i,j,k]$ is closed, namely,

1) Harmonic mean: $s_t = 3/([\bbx_1]^{-1}_{[i,j]} + [\bbx_1]^{-1}_{[j,k]} + [\bbx_1]^{-1}_{[i,k]})$, 

2) Geometric mean: $ s_t = \lim_{p\to 0}[([\bbx_1]^p_{[i,j]} + [\bbx_1]^p_{[j,k]} + [\bbx_1]^p_{[i,k]})]^{1/p}$, and 

3) Arithmetic mean: $s_t = ([\bbx_1]_{[i,j]} + [\bbx_1]_{[j,k]} + [\bbx_1]_{[i,k]})/3$,

which compute the triangle weight based on its three faces. Similarly, we generalized these mean methods to compute the weight of a 3-simplex $[i,j,k,m]$ based on the four triangle faces in 3-simplex prediction. 

We also consider different learning methods. Specifically,

1) Simplicial 2-Complex NN (``Bunch'') [cf. \eqref{eq.bunch_model}] by \citet{bunch2020simplicial} (we also generalized this model to 3-dimension for 3-simplex prediction), which provides a baseline for SCCNNs to see the effect of using the higher-order convolutions.

2) Convolutional filters on SC (``CF-SC'') by \citet{isufi2022convolutional}, which is a linear version of SCCNN and has a limited extended simplicial locality. The comparison to it reveals the effect of the extended simplicial locality. 

3) Principled SNN (``PSNN'') by \citet{roddenberry2021principled}, which performs a one-step simplicial shifting, providing as a baseline to observe the benefit of higher-order convolutions and inter-simplicial couplings;  

4) SNN by \citet{ebli2020simplicial}, which performs higher-order convolutions but without separating the lower and upper adjacencies. We can see the benefit of allowing different processing in the gradient and curl spaces by comparing to it. 

5) Simplicial Convolutional NN (``SCNN'') [cf. \eqref{eq.scnn_yang}] by \citet{yang2021finite}, which does not include information from adjacent simplices. The comparison to it shows the effect of inter-simplicial couplings.  

6) SCF on edges by \citet{yang2022simplicial}, which is a linear version of SCNN.

7) GNN [cf. \eqref{eq.gnn}] by \citet{defferrard2017} and \citet{gama2020stability}, which performs a higher-order graph convolution without using the SC model. The comparison to it is necessary to observe the benefit of SC models in simplex prediction. 

8) Linear graph filters on nodes (``GF'') by \citet{sandryhaila2013discrete}, which is a linear version of GNNs; 

9) MLP: $\bbY_k = \sigma(\bbX_k\bbW_k)$, providing as a baseline for the effect of using inductive models.

For MLP, Bunch, CF-SC and our SCCNN, we consider the outputs in the node and edge spaces, respectively, for 2-simplex prediction, which are denoted by a suffix ``-Node'' or ``-Edge''. For 3-simplex prediction, the output in the triangle space can be used as well, denoted by a suffix ``-Tri.'', where we also build SCNNs in both edge and triangle spaces. 

\subsubsection{Experimental Setup and Hyperparameters}

\textbf{2-Simplex Prediction.}
We consider normalized Hodge Laplacians and incidence matrices. Specifically, we use the symmetric version of the normalized random walk Hodge Laplacians in the edge space, proposed by \citet{schaub2020random}, given by $\bbL_{1}^s = \bbM_{11}^{-1/2}\bbL_1\bbM_{11}^{1/2}$ where
\begin{equation}\label{eq.normalized_l1}
  \bbL_{1,\rmd} = \bbM_{11}\bbB_1^\top\bbM_{01}^{-1}\bbB_1, \quad \bbL_{1,\rmu} = \bbB_2 \bbM_{21} \bbB_2^\top \bbM_{11}^{-1}
\end{equation}
with $\bbM_{11} = \max(\diag(|\bbB_2|\mathbf{1}),\bbI)$, $\bbM_{01} =2\diag(|\bbB_1|\bbM_{11}\mathbf{1})$ and $\bbM_{21} = 1/3\bbI$. In the node space, we use the symmetric random walk graph Laplacian, given by $\bbL_0^s = \bbM_{00}^{-1/2}\bbL_0\bbM_{00}^{1/2}$, where we have 
\begin{equation}\label{eq.normalized_l0}
  \bbL_{0} = \bbB_1\bbM_{10}\bbB_1^\top\bbM_{00}^{-1}
\end{equation}
with $\bbM_{00} = \max(\diag(|\bbB_1|\mathbf{1}),\bbI)$ and $\bbM_{10}=\bbI$. In the triangle space, we have 
\begin{equation}\label{eq.normalized_l2d}
  \bbL_{2,\rmd} = \bbM_{22}\bbB_2^\top\bbM_{12}^{-1}\bbB_2
\end{equation}
with $\bbM_{12} = \diag(|\bbB_2|\mathbf{1})$ and $\bbM_{22}=\bbI$, and the symmetric version is $\bbL_{2,\rmd}^{s} = \bbM_{22}^{-1/2}\bbL_{2,\rmd} \bbM_{22}^{1/2}$. 
In the projection steps, we have $\bbR_{0,\rmu}=\bbM_{01}^{-1}\bbB_1$, $\bbR_{1,\rmd}=\bbM_{11}\bbB_1^\top\bbM_{01}^{-1}$, $\bbR_{1,\rmu}=\bbB_2\bbM_{21}$, and $\bbR_{2,\rmd}= \bbM_{22}\bbB_2^\top\bbM_{12}^{-1}$. The above normalizations also respect the definitions in \citet{bunch2020simplicial}.    

\textbf{3-Simplex Prediction.} 
Keeping the same for $k=0,1$, for $k=2$, we have 
\begin{equation}
  \bbL_{2,\rmd} = \bbM_{22}\bbB_2^\top\bbM_{12}^{-1}\bbB_2, \quad \bbL_{2,\rmu}= \bbB_3\bbM_{32}\bbB_3^\top\bbM_{22}^{-1}
\end{equation}
with $\bbM_{22} = \max(\diag(|\bbB_3|\mathbf{1}),\bbI)$, $\bbM_{12}=3\diag(|\bbB_2|\bbM_{22}\mathbf{1})$ and $\bbM_{32} = 1/4\bbI$. This definition respects the way of constructing random walk Laplacians on simplices \citep{schaub2020random}, also used by \citet{chen2022time}. In the 3-simplex space, we have 
\begin{equation}
  \bbL_{3,\rmd} = \bbM_{33}\bbB_3^\top\bbM_{23}^{-1}\bbB_3
\end{equation}
with $\bbM_{33}=\bbI$ and $\bbM_{23}=\diag(|\bbB_3|\mathbf{1})$. For the projection matrices, we have $\bbR_{1,\rmu} =\bbM_{12}^{-1}\bbB_2$,  $\bbR_{2,\rmd} = \bbM_{22}\bbB_2^\top\bbM_{12}^{-1}$, $\bbR_{2,\rmu} = \bbB_3\bbM_{32}$ and $ \bbR_{3,\rmd} = \bbM_{33}\bbB_3^\top\bbM_{23}^{-1}$. 

\textbf{Hyperparameters.}
The hyperparameters are set as:

1) the number of layers: $L\in\{1,2,3,4,5\}$;

2) the number of intermediate and output features to be the same as $F\in\{16,32\}$;
 
3) the convolution orders for SCCNNs are set to be the same, i.e., $T_\rmd'=T_\rmd=T_\rmu=T_\rmu=T\in\{1,2,3,4,5\}$. We do so to avoid the exponential growth of the parameter search space. For GFs \citep{sandryhaila2014discrete}, GNNs \citep{defferrard2017} and SNNs \citep{ebli2020simplicial}, we set the convolution orders to be $T\in\{1,2,3,4,5\}$ while for SCNNs \citep{yang2021simplicial} and SCFs \citep{yang2022simplicial}, we allow the lower and upper convolutions to have different orders with $T_\rmd,T_\rmu\in\{1,2,3,4,5\}$;

4) the nonlinearity in the feature learning phase: LeakyReLU with a negative slope $0.01$;

5) the MLP in the prediction phase: two layers with a sigmoid nonlinearity. For 2-simplex prediction, the number of the input features for the node features is $3F$, and for the edge features is $3F$. For 3-simplex prediction, the number of the input features for the node features is $4F$, for the edge features is $6F$ and for the triangle features is $4F$ since a 3-simplex has four nodes, six edges and four triangles. The number of the intermediate features is the same as the input features, and that of the output features is one;

6) the binary cross entropy loss and the adam optimizer with a learning rate of $0.001$ are used; the number of the epochs is 1000 where an early stopping is used. We compute the AUC to compare the performance and run the same experiments for ten times with random data splitting. 

\begin{table}[t!]
  \caption{2- \emph{(Left)} and 3-Simplex \emph{(Right)} prediction AUC ($\%$) results, where the \emph{first} and \emph{second} best are highlighted in \emph{\textcolor{red}{red}} and \emph{\textcolor{blue}{blue}}. }
  \label{tab:simlex_prediction_full_results}
  \vskip 0.15in
  \begin{center}
  \begin{small}
  \begin{sc}
    \resizebox{0.495\columnwidth}{!}{
  \begin{tabular}{lcr}
  \toprule
  Methods & AUC & Parameters \\
  \midrule
  Harm. Mean &62.8$\pm$2.7 & --- \\
  Arith. Mean &60.8$\pm$3.2 & --- \\
  Geom. Mean &61.7$\pm$3.1 & --- \\
  MLP-Node &68.5$\pm$1.6& $L=1,F=32 $\\
  MLP-Edge &65.5$\pm$4.5& $L=4,F=16 $\\
  \midrule
  GF &78.7$\pm$1.2 &$L=2,F=32,T=2 $\\
  SCF & 92.6$\pm$1.8 &$L=1,F=32,T_\rmd =1,T_\rmu=5 $\\
  CF-SC-Node & 96.9$\pm$0.8 & $L=5,F=32,T =4 $\\
  CF-SC-Edge & 93.2$\pm$1.6 & $L=2,F=32,T =4  $\\
  \midrule
  GNN &93.9$\pm$1.0& $L=5,F=32,T=2 $\\
  SNN &92.0$\pm$1.8& $L=5,F=32,T=5 $\\
  PSNN &95.6$\pm$1.3& $L=5, F=32 $\\
  SCNN &96.5$\pm$1.5 & $L=5,F=32,T_\rmd=5,T_\rmu=2 $\\
  Bunch-Node &\textcolor{blue}{98.0$\pm$0.5}& $L=4,F=32 $\\
  Bunch-Edge &94.6$\pm$1.2& $L=4,F=16 $\\ 
  \midrule
  SCCNN-Node &\textcolor{red}{98.4$\pm$0.5}& $L=2,F=32,T=2 $\\
  SCCNN-Edge &95.9$\pm$1.0&$L=5,F=32,T=3 $\\
  \bottomrule
  \end{tabular}
    }
    \resizebox{0.495\columnwidth}{!}{
  \begin{tabular}{lcr}
    \toprule
    Methods & AUC & Parameters \\
    \midrule
    Harm. Mean &63.6$\pm$1.6 & --- \\
    Arith. Mean &62.2$\pm$1.4 & --- \\
    Geom. Mean &63.1$\pm$1.4 & --- \\
    MLP-Node &68.9$\pm$1.8& $L=1,F=32 $\\
    MLP-Tri. &69.0$\pm$2.2& $L=3,F=32 $\\
    \midrule
    GF &83.9$\pm$2.3 &$L=1,F=32,T=2 $\\
    SCF-Edge &77.6$\pm$3.8 &$L=5,F=32,T_\rmd=0,T_\rmu=4$ \\
    SCF-Tri. &94.9$\pm$1.0 &$L=2,F=32,T_\rmd=T_\rmu=3 $\\
    CF-SC-Node &95.8$\pm$1.0 &$L=3,F=32,T=2$ \\
    CF-SC-Edge &97.9$\pm$0.7 &$L=5,F=32,T=3$ \\ 
    CF-SC-Tri. &96.7$\pm$0.5 &$L=2,F=32,T=5$ \\
    \midrule
    GNN &96.6$\pm$0.5& $L=5,F=32,T=5 $\\
    SNN &95.1$\pm$1.2& $L=5,F=32,T=5 $\\
    PSNN &98.1$\pm$0.5& $L=5, F=32  $\\
    SCNN-Edge &98.3$\pm$0.4&$L=4,F=32,T_\rmd=T_\rmu=3$\\
    SCNN-Tri. &98.3$\pm$0.4 & $L=5,F=32,T_\rmd=2,T_\rmu=1 $\\
    Bunch-Edge &98.5$\pm$0.5& $L=4,F=16 $\\
    Bunch-Tri. &96.6$\pm$1.0& $L=2,F=32 $\\
    \midrule
    \textcolor{blue}{SCCNN-Node} &\textcolor{blue}{99.4$\pm$0.3}&$L=3,F=32,T=3 $\\
    \textcolor{red}{SCCNN-Edge} &\textcolor{red}{99.0$\pm$1.0}&$L=5,F=32,T=5 $\\
    SCCNN-Tri. &97.4$\pm$0.9&$L=4,F=32,T=4 $\\
    \bottomrule
    \end{tabular}
    }
  \end{sc}
  \end{small}
  \end{center}
  \vskip -0.1in
\end{table}

\subsubsection{Results}
In \cref{tab:simlex_prediction_full_results}, we report the best results of each method with the corresponding hyperparameters. Different hyperparameters can lead to similar results, but we report the ones with the \emph{least} complexity. We make the following observations:

1) methods on graphs and SCs perform much better than heuristic methods and MLPs, because of the inductive bias proposed by the graph and SC models;

2) both NNs and filtering methods on SCs perform better than their counterparts on graphs in general, which validates the SC model in simplex predictions, showing the importance of higher-order structures;

3) jointly performing the simplicial convolution based on $\bbL_k$ limits the expressive power as shown in \cref{sec:spectral_analysis} and \cref{fig:frequency_response_sccnns} from the spectral perspective, which can be seen from that SCNN performs better than SNN; 

4) higher-order simplicial convolutions improve the performance. In 2-simplex prediction, SCNN is better than PSNN or Bunch-Edge which only performs a simplicial shifting. Moreover, using either node or edge features of SCCNN gives better results than the counterpart of Bunch, which also requires a deeper NN, inducing stronger mutual dependence between the stability of different simplicial outputs as studied in \cref{sec:stability_analysis}.  
In 3-simplex prediction, the same observations can be made; 

5) taking into account the information from simplices of different orders further improves the performance, which corroborates our architecture.

\subsection{Frequency Response}
To further see the effects of using higher-order convolutions and differentiating the two adjacencies, we illustrate the frequency responses of SCFs $\bbH_1$, $\bbH_{1,\rmd}$ and $\bbH_{1,\rmu}$ at first layer for several feature indices in \cref{fig:frequency_response_sccnns}. First, we see that the SCF $\bbH_1$ has different gradient and curl frequency responses, which allows separate processing in the gradient and curl frequencies. This is the result of differentiating the lower and upper adjacencies by using different parameter spaces, in contrast to \citet{ebli2020simplicial}. Second, we see that using higher-order convolution enables a more versatile relation rather than linear or simple scaling, which increases the expressive power of SCCNNs. As a counterexample, \citet{bunch2020simplicial} and \citet{roddenberry2021principled} perform a one-step simplicial shifting by $\bbL_1$, the frequency response of which is simply a straight line in terms of the simplicial frequencies. Also, merely summing the projected information leads to a simple scaling.  This limited expressive power of Bunch requires more layers to have a comparable performance than SCCNN-Node.

\begin{figure}[ht!]
  \centering
  \begin{subfigure}{0.19\linewidth}
    \includegraphics[width=\linewidth]{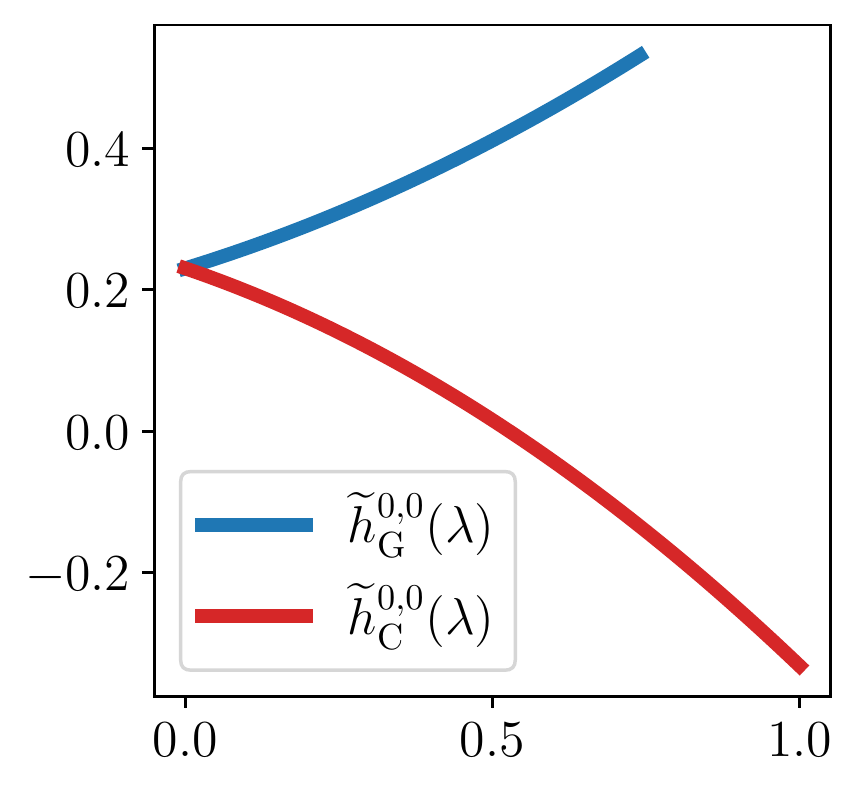}
    \caption{$\bbH_1$@0}
  \end{subfigure}
  \begin{subfigure}{0.19\linewidth}
    \includegraphics[width=\linewidth]{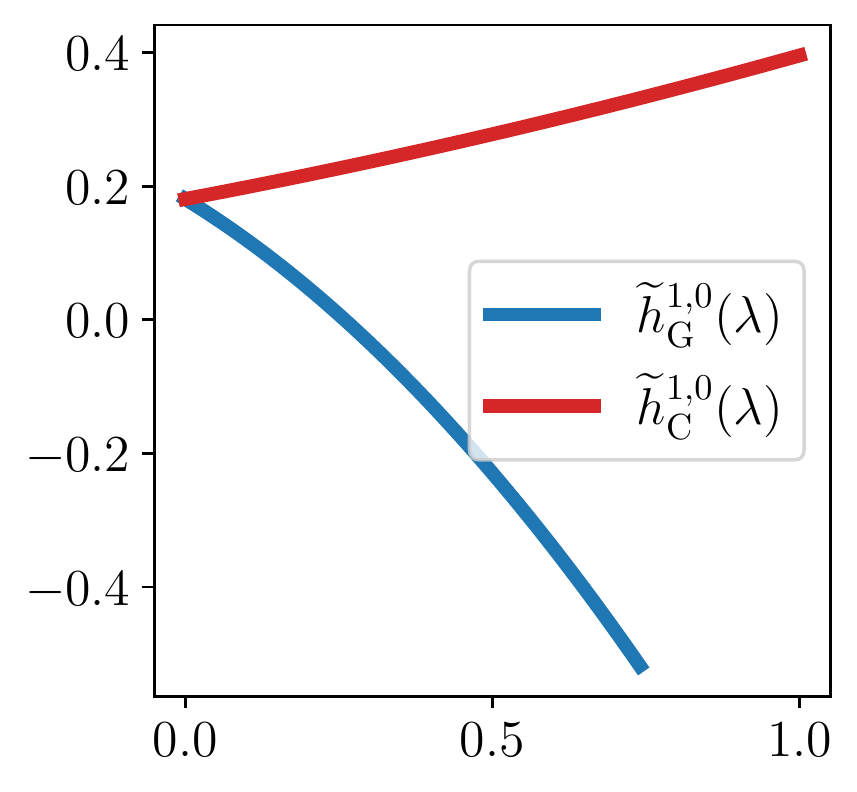}
    \caption{$\bbH_1$@1}
  \end{subfigure}
  \begin{subfigure}{0.19\linewidth}
    \includegraphics[width=\linewidth]{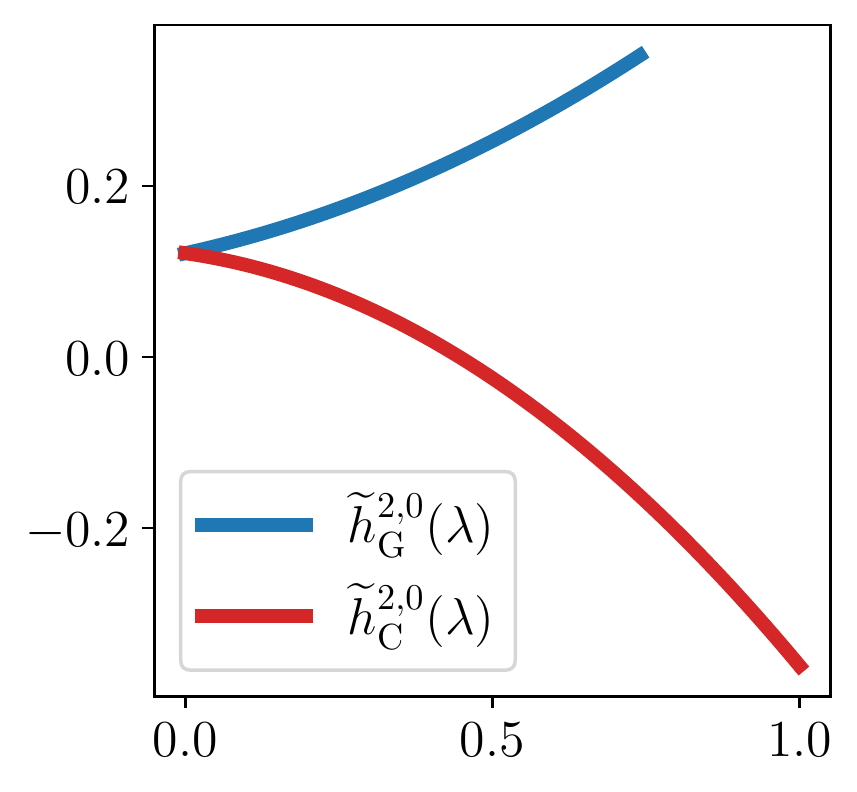}
    \caption{$\bbH_1$@2}
  \end{subfigure}
  \begin{subfigure}{0.19\linewidth}
    \includegraphics[width=\linewidth]{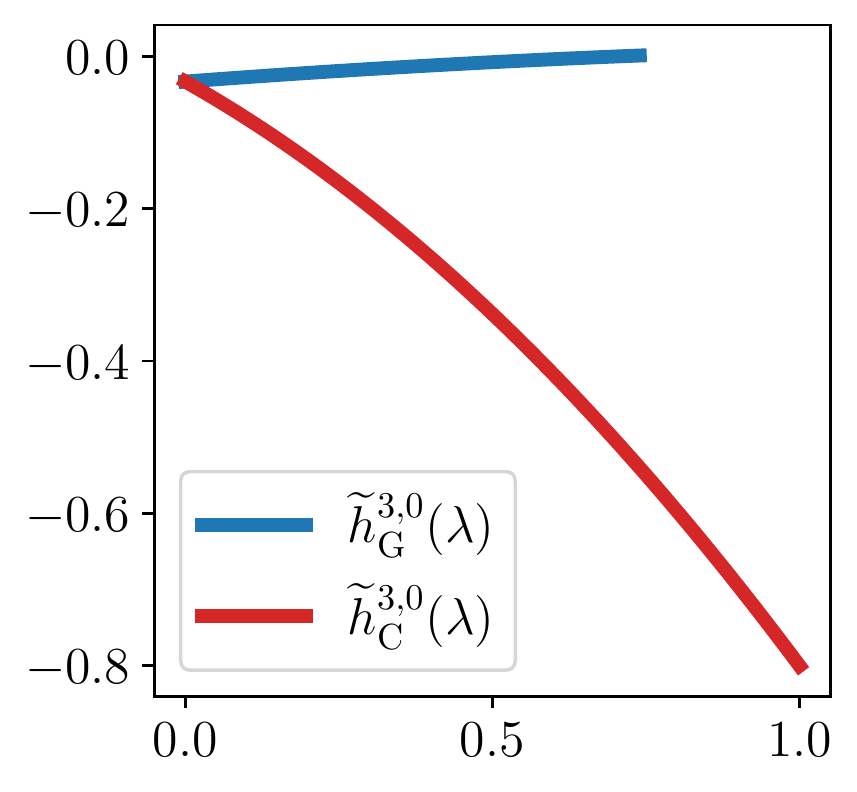}
    \caption{$\bbH_1$@3}
  \end{subfigure}
  \begin{subfigure}{0.19\linewidth}
    \includegraphics[width=\linewidth]{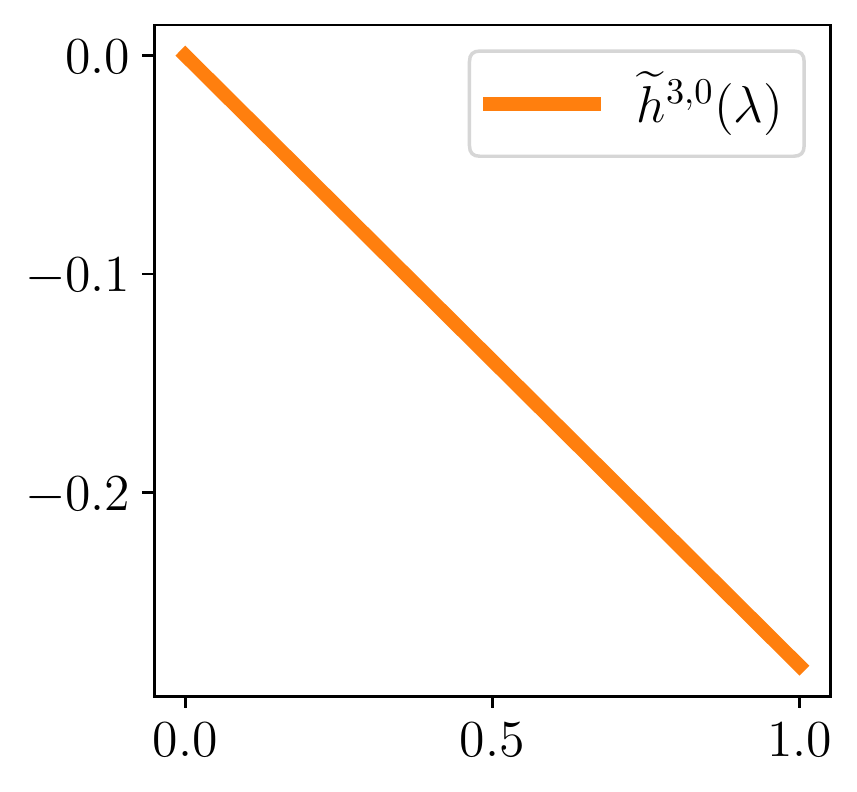}
    \caption{$\bbH_1$,Bunch@3}
  \end{subfigure}
  \begin{subfigure}{0.19\linewidth}
    \includegraphics[width=\linewidth]{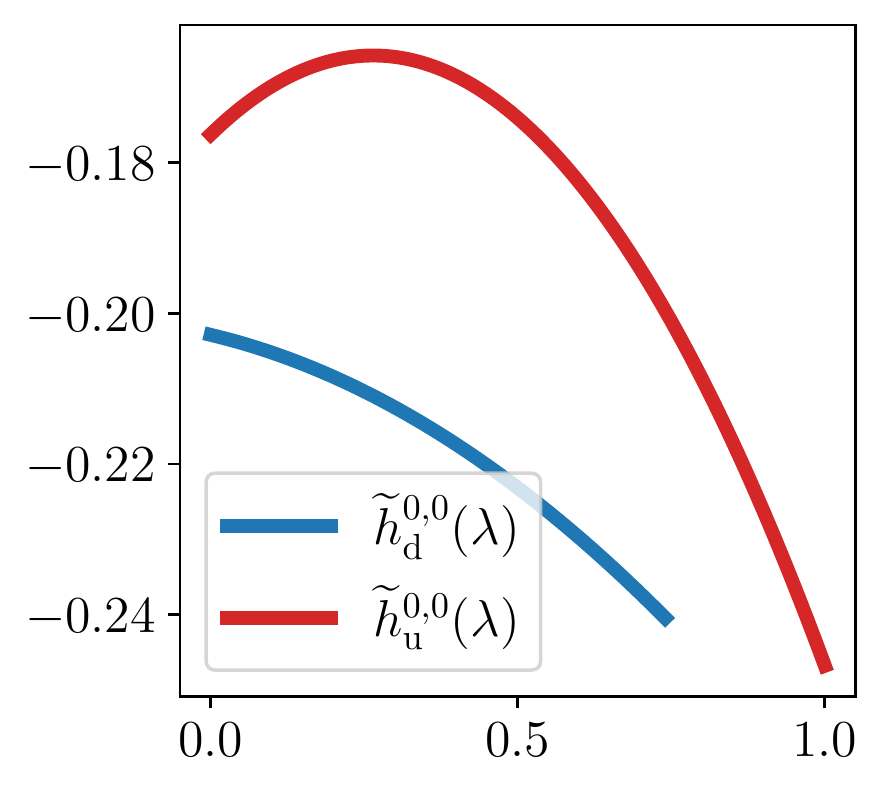}
    \caption{$\bbH_{1,\rmd},\bbH_{1,\rmu}$@0}
  \end{subfigure}
  \begin{subfigure}{0.19\linewidth}
    \includegraphics[width=\linewidth]{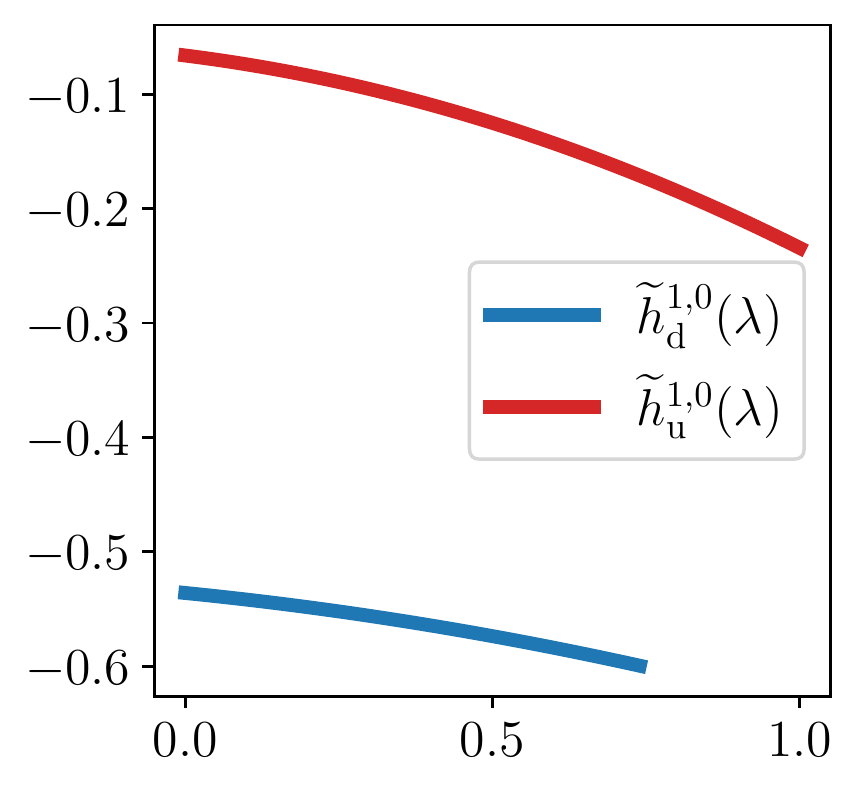}
    \caption{$\bbH_{1,\rmd},\bbH_{1,\rmu}$@1}
  \end{subfigure}
  \begin{subfigure}{0.19\linewidth}
    \includegraphics[width=\linewidth]{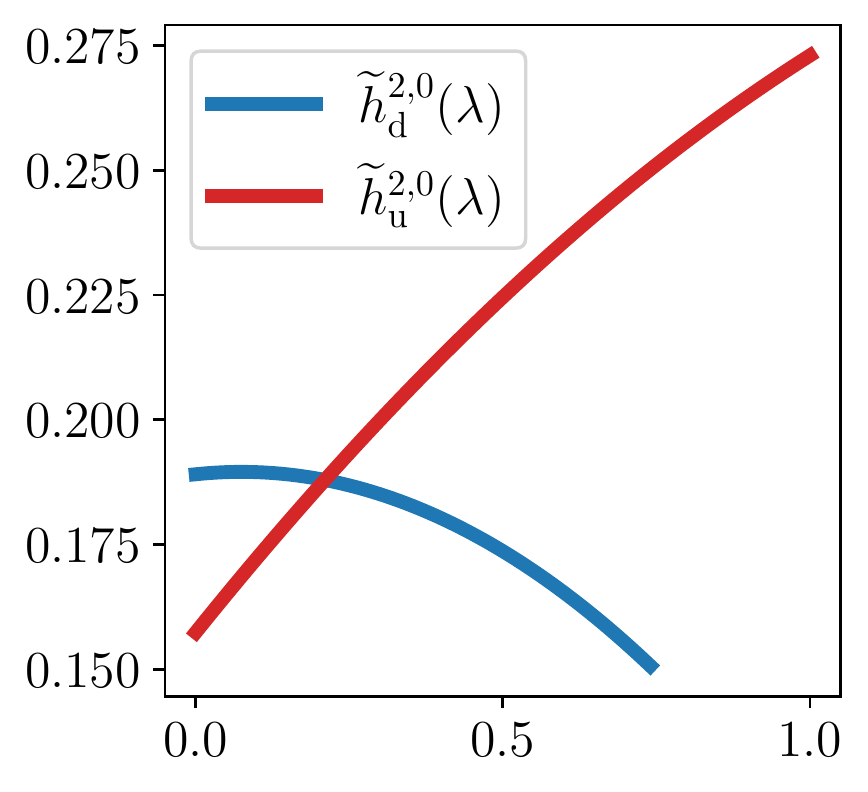}
    \caption{$\bbH_{1,\rmd},\bbH_{1,\rmu}$@2}
  \end{subfigure}
  \begin{subfigure}{0.19\linewidth}
    \includegraphics[width=\linewidth]{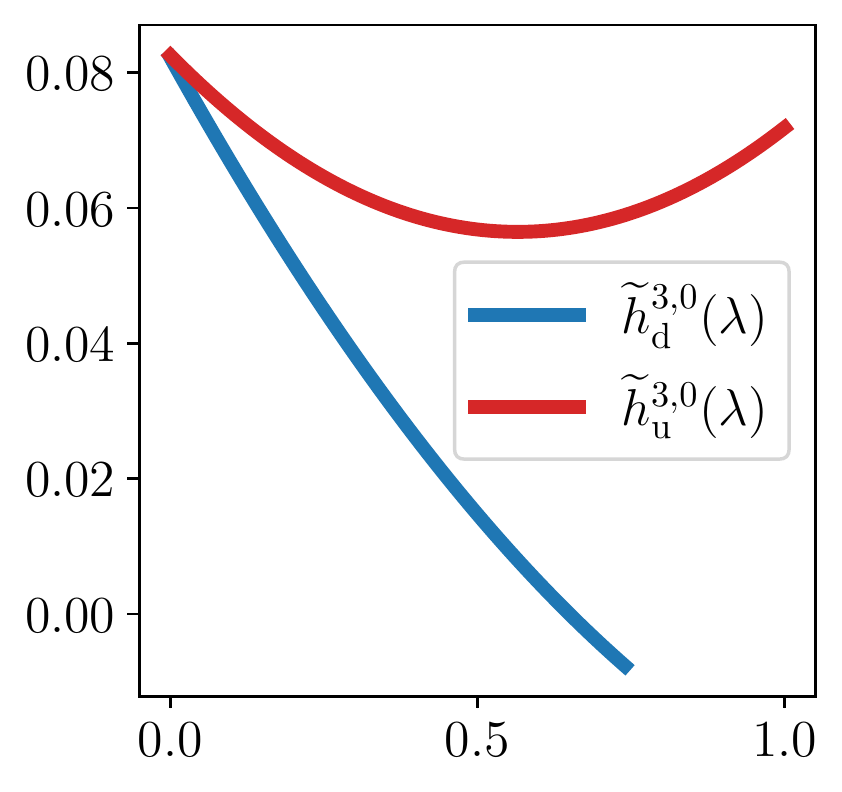}
    \caption{$\bbH_{1,\rmd},\bbH_{1,\rmu}$@3}
  \end{subfigure}
  \begin{subfigure}{0.19\linewidth}
    \includegraphics[width=\linewidth]{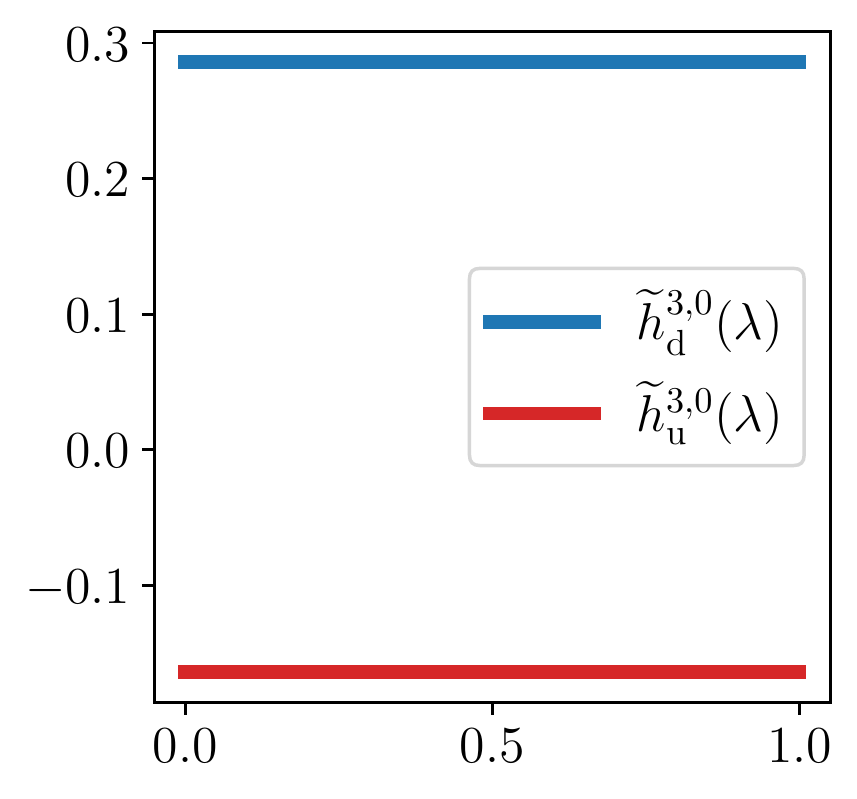}
    \caption{$\bbH_{1,\rmd},\bbH_{1,\rmu}$,Bunch@3}
  \end{subfigure}
  \caption{The frequency responses of three SCFs of SCCNN-Node with $L=2,T=2,F=32$ at first layer for different output-input feature pairs $(f_1,f_0)$ with $f_1=0,1,2,3,4$ and $f_0=0$. \emph{(Top Row)} (a-d):SCFs $\bbH_1$@$f_1$  where gradient and curl frequency responses are in blue and red; and (e):$\bbH_1=\bbL_1$ for Bunch (in orange since it has one frequency response for all frequencies). \emph{(Bottem Row)} (f-i): lower and upper SCFs $\bbH_{1,\rmd}$ and $\bbH_{1,\rmu}$@$f_l$ in blue and red, resepctively; and (j): $\bbH_{1,\rmd}=\bbH_{1,\rmu}=\bbI$ for Bunch.}
  \label{fig:frequency_response_sccnns}
\end{figure}

\subsubsection{Ablation Study}
We perform an ablation study to observe the roles of different components in SCCNNs. 

\textbf{Limited Input.} We study the influence of limited input data for model SCCNN-Node of order two. Specifically, we consider a missing rate of $\{10\%,40\%,70\%,100\% \}$ for either node input or edge input. From \cref{tab:missing_input}, we see that the prediction performance does not deteriorate at a cost of the model complexity (higher convolution orders) when a certain part of the input missing except with full zeros as input. This ability of learning from limited data shows the robustness of SCCNNs. 

\begin{table}[ht!]
  \caption{2-Simplex Prediction results of SCCNN-Node with missing input data.}
  \label{tab:missing_input}
  \vskip 0.15in
  \begin{center}
  \begin{small}
  \begin{sc}
    \resizebox{0.495\columnwidth}{!}{
  \begin{tabular}{lccr}
  \toprule
  Missing Rate & AUC & Parameters \\
  \midrule
  Node-0.1, Edge-0 &98.4$\pm$0.5&$L=2,F=32,T=2 $\\
  Node-0.4, Edge-0 &98.6$\pm$0.5&$L=2,F=32,T=5 $\\
  Node-0.7, Edge-0 &98.6$\pm$0.5&$L=2,F=32,T=3 $\\
  Node-1.0, Edge-0 &98.2$\pm$0.5&$L=2,F=32,T=4 $\\
  \midrule 
  Node-0, Edge-0.1 &98.4$\pm$0.4&$L=2,F=32,T=1 $\\
  Node-0, Edge-0.4 &98.5$\pm$0.6&$L=2,F=32,T=4 $\\
  Node-0, Edge-0.7 &98.6$\pm$0.3&$L=2,F=32,T=3 $\\
  Node-0, Edge-1.0 &98.1$\pm$0.4&$L=2,F=32,T=3 $\\
  \midrule 
  Node-0.1, Edge-0.1 &98.6$\pm$0.4&$L=2,F=32,T=1 $\\
  Node-0.4, Edge-0.4 &98.6$\pm$0.7&$L=2,F=32,T=4 $\\
  Node-0.7, Edge-0.7 &98.3$\pm$0.5&$L=2,F=32,T=4 $\\
  Node-1.0, Edge-1.0 &50.0$\pm$0.0& --- \\
  \bottomrule
  \end{tabular}
    }
  \end{sc}
  \end{small}
  \end{center}
  \vskip -0.1in
\end{table}

\textbf{SC Order $K$.}  We then investigate the influence of the SC order $K$. \cref{tab:sccnn_results_lower_order} reports the 2-simplex prediction results for $K=\{1,2\}$ and the 3-simplex prediction results for $K=\{1,2,3\}$.  We observe that for $k$-simplex prediction, it does not necessarily guarantee a better prediction with a higher-order SC, which further indicates that a positive simplex could be well encoded by both its faces and other lower-order subsets. For example, in 2-simplex prediction, SC of order one gives better results than SC of order two (similar for Bunch), showing that in this coauthorship complex, triadic collaborations are better encoded by features on nodes than pairwise collaborations. In 3-simplex prediction, SCs of different orders give similar results, showing that tetradic collaborations can be encoded by nodes, as well as by pairwise and triadic collaborations.  

\begin{table}[ht!]
  \caption{Prediction results of SCCNNs with lower order $K$. }
  \label{tab:sccnn_results_lower_order}
  \vskip 0.15in
  \begin{center}
  \begin{small}
  \begin{sc}
    \resizebox{0.495\columnwidth}{!}{
  \begin{tabular}{lcr} 
  \toprule 
  Method & 2-Simplex & Parameters \\
  \midrule  
  SCCNN-Node &98.7$\pm$0.5&$K=1,L=2,F=32,T=2 $\\
  Bunch-Node &98.3$\pm$0.4&$K=1,L=4,F=32$ \\
  SCCNN-Node &98.4$\pm$0.5&$K=2,L=2,F=32,T=2 $\\
  Bunch-Node &98.0$\pm$0.4&$K=2,L=4,F=32$ \\
  \midrule
  SCCNN-Edge &97.9$\pm$0.9&$K=1,L=3,F=32,T=5 $\\ 
  Bunch-Edge &97.3$\pm$1.1&$K=1,L=4,F=32$ \\
  SCCNN-Edge &95.9$\pm$1.0&$K=2,L=5,F=32,T=3 $\\
  Bunch-Edge &94.6$\pm$1.2&$K=2,L=4,F=32$ \\
  \bottomrule
\end{tabular}
    }
  \resizebox{0.495\columnwidth}{!}{
  \begin{tabular}{lcr} 
  \toprule 
  Method & 3-Simplex & Parameters \\
  \midrule
  SCCNN-Node &99.3$\pm$0.3&$K=1,L=2,F=32,T=1 $\\
  SCCNN-Node &99.3$\pm$0.2&$K=2,L=2,F=32,T=5 $\\
  SCCNN-Node &99.4$\pm$0.3&$K=3,L=3,F=32,T=3 $\\
  \midrule
  SCCNN-Edge &98.9$\pm$0.5&$K=1,L=3,F=32,T=5 $\\
  SCCNN-Edge &99.2$\pm$0.4&$K=2,L=5,F=32,T=5 $\\
  SCCNN-Edge &99.0$\pm$1.0&$K=3,L=5,F=32,T=5 $\\
  \midrule
  SCCNN-Tri. &97.9$\pm$0.7&$K=2,L=4,F=32,T=4 $\\
  SCCNN-Tri. &97.4$\pm$0.9&$K=3,L=4,F=32,T=4 $\\
  \bottomrule
  \end{tabular}
  }
  \end{sc}
  \end{small}
  \end{center}
  \vskip -0.1in
\end{table}

\textbf{Missing Components in SCCNNs.} With a focus on 2-simplex prediction with SCCNN-Node of order one, to avoid overcrowded settings, we study how each component of an SCCNN influences the prediction. We consider the following settings without: 
1) ``Edge-to-Node'', where the projection $\bbx_{0,\rmu}$ from edge to node is not included, equivalent to GNN;
2) ``Node-to-Node'', where for node output, we have $\bbx_0^l = \sigma(\bbH_{0,\rmu}^l\bbR_{1,\rmu}\bbx_1^{l-1})$; 
3) ``Node-to-Edge'', where the projection $\bbx_{1,\rmd}$ from node to edge is not included, i.e., we have $\bbx_1^l = \sigma(\bbH_1^l\bbx_1^{l-1})$;
and 4) ``Edge-to-Edge'', where for edge output, we have $\bbx_1^l = \sigma(\bbH_{1,\rmd}^l\bbR_{1,\rmd}\bbx_0^{l-1})$. 

\begin{table}[ht!]
  \caption{(Left): 2-Simplex prediction results with SCCNN-Node without certain components. (Right): Running time in ms per epoch.}
  \label{tab:sccnn_ablation_results}
  \vskip 0.15in
  \begin{center}
  \begin{small}
  \begin{sc}
  \begin{tabular}{lcr}
  \toprule 
  Missing Component  & 2-Simplex & Parameters \\
  \midrule
  --- &98.7$\pm$0.5&$K=1,L=2,F=32,T=2 $\\
  Edge-to-Node &93.9$\pm$0.8&$K=1,L=5,F=32,T=2 $\\
  Node-to-Node &98.7$\pm$0.4&$K=1,L=4,F=32,T=2 $\\
  Edge-to-Edge &98.5$\pm$1.0&$K=1,L=3,F=32,T=3 $\\
  Node-to-Edge &98.8$\pm$0.3& $K=1,L=4,F=32,T=3 $\\
  \bottomrule
  \end{tabular}
  \hspace{0.3in}
  \end{sc}
  \end{small}
  \end{center}
  \vskip -0.1in
\end{table}

From the results in \cref{tab:sccnn_ablation_results} (Left), we see that ``No Edge-to-Node'', i.e., GNN, gives much worse results as it leverages no information on edges with limited expressive power. For cases with other components missing, a similar performance can be achieved, however, at a cost of the model complexity, with either a higher convolution order or a larger number of layers $L$, while the latter in turn degrades the stability of the SCCNNs, as discussed in \cref{sec:stability_analysis}. As studied by \citet[Thm. 6]{bodnar2021weisfeiler}, SCCNNs with certain inter-simplicial couplings pruned/missing can be powerful as well, but it comes with a cost of complexity which might degrade the model stability if more number layers are required.

\subsubsection{Stability Analysis}\label{app:stability_analysis_experiment}

We then perform a stability analysis of SCCNNs. We artificially add perturbations to the normalization matrices, which resemble the weights of simplices. Specifically, normalization matrices $\bbM_{00}$ and $\bbM_{01}$ resemble the node weights, matrices $\bbM_{10}$, $\bbM_{11}$ and $\bbM_{12}$ the edge weights, and $\bbM_{21}$ and $\bbM_{22}$ the triangle weights. 

We consider small perturbations $\bbE_0$ on node weights which is a diagonal matrix following that $\lVert\bbE_0\rVert\leq {\epsilon_0}/{2}$. We generate its diagonal entries from a uniform distribution $[-{\epsilon_0}/{2},{\epsilon_0}/{2})$ with $\epsilon_0\in[0,1]$, which represents a one degree of deviation of the node weigths from the true ones. Normalization matrices $\bbM_{0\cdot}$ are accordingly deviated from the true ones, based which perturbed projection matrices and Hodge Laplacians are defined. Similarly, perturbations on edge weights and triangle weights are applied to study the stability. 

With an SCCNN-Node for 2-simplex prediction, we measure the distance between the node outputs with and without perturbations, i.e., $\lVert \bbx_0^L-\hat{\bbx}_0^L \rVert / \lVert \bbx_0^L\rVert$; and likewise the distance between the edge outputs $\lVert \bbx_1^L-\hat{\bbx}_1^L \rVert / \lVert \bbx_1^L\rVert$ and the distance between triangle outputs $\lVert \bbx_2^L-\hat{\bbx}_2^L \rVert / \lVert \bbx_2^L\rVert$. 
We take as the baseline the setting $L=2,F=32,T=2$ of the best result, studying the influence of the following three factors on the stability. 

\textbf{Stability Dependence.}
We first show the stability dependence between different simplices in \cref{fig:mutual_dependence}. We see that under perturbation on node weights, triangle output is not influenced until the number of layers becomes two; likewise, node output is not influenced by perturbations on triangle weights with a one-layer SCCNN. Also, a one-layer SCCNN under perturbations on edge weights will cause outputs on nodes, edges, triangles perturbed. Lastly, we observe that the same degree of perturbations added to different simplices causes different degrees of instability, owing to the number $N_k$ of $k$-simplices in the stability bound. Since $N_0<N_1<N_2$, the perturbations on node weights cause less instability than those on edge and triangle weights.  

\begin{figure}[ht!]
  \centering
  \begin{subfigure}{0.19\linewidth}
    \includegraphics[width=\linewidth]{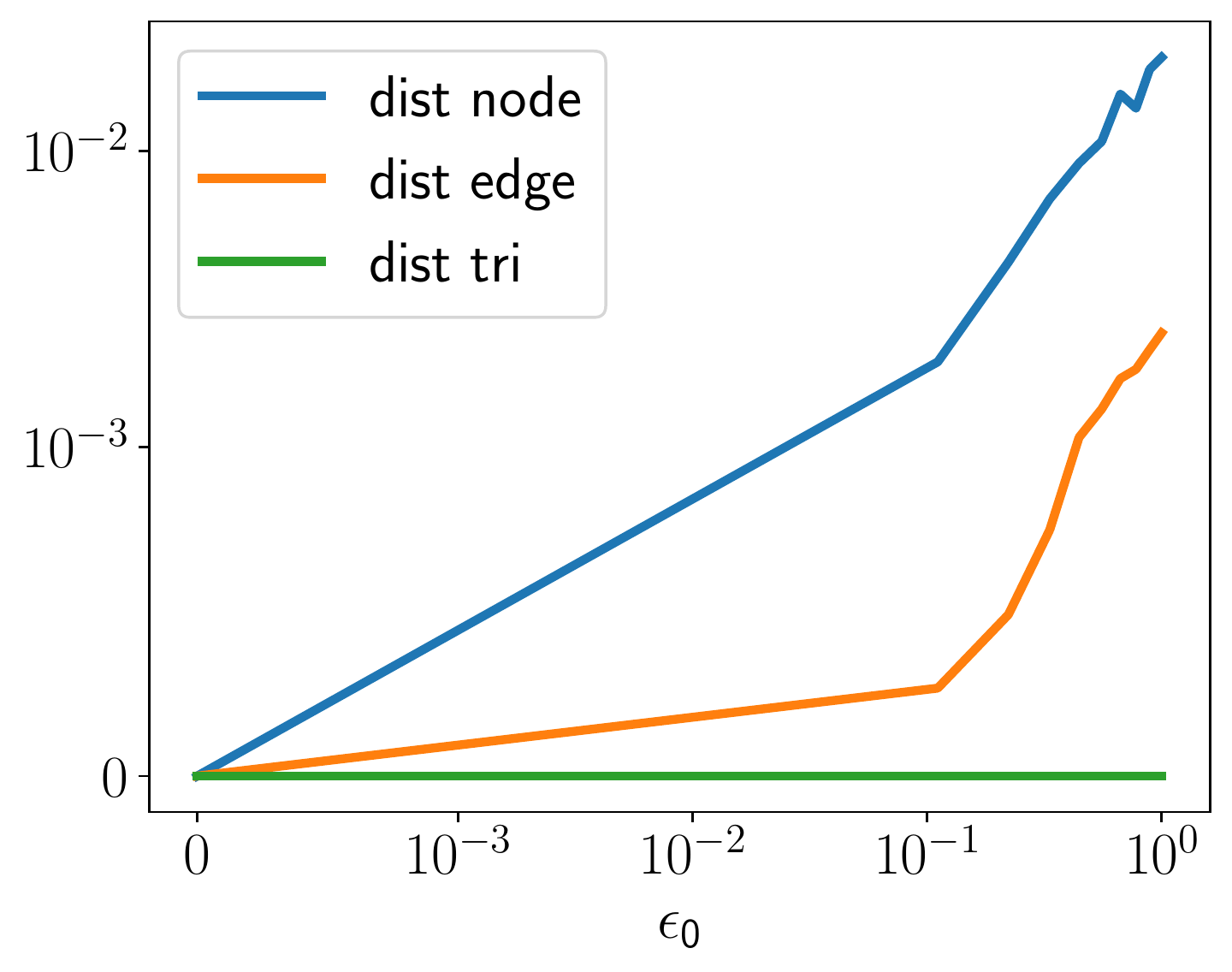}
    \caption{Node pert. $\bbE_0, L=1$}
  \end{subfigure}
  \begin{subfigure}{0.19\linewidth}
    \includegraphics[width=\linewidth]{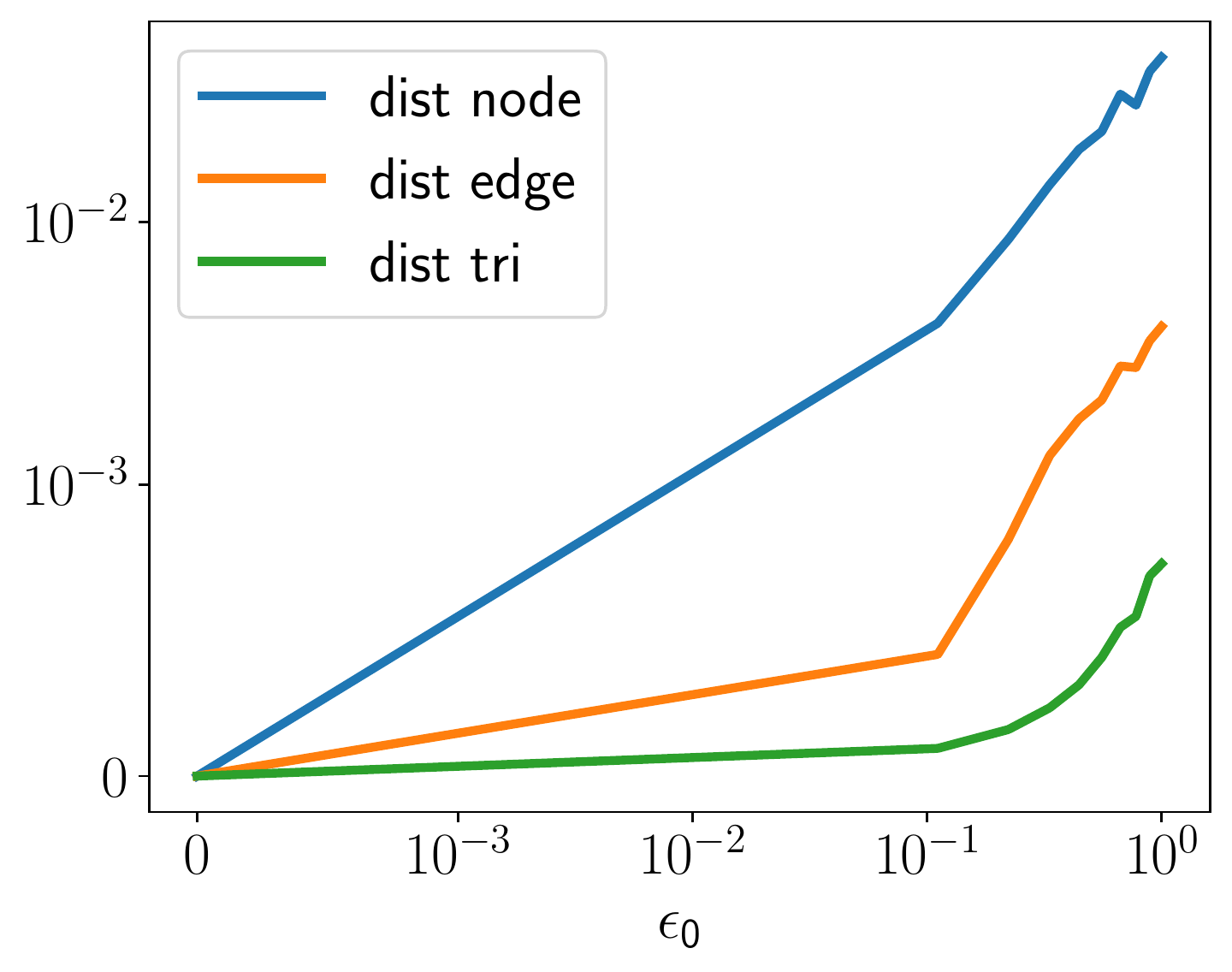}
    \caption{Node pert. $\bbE_0, L=2$}
  \end{subfigure}
  \begin{subfigure}{0.19\linewidth}
    \includegraphics[width=\linewidth]{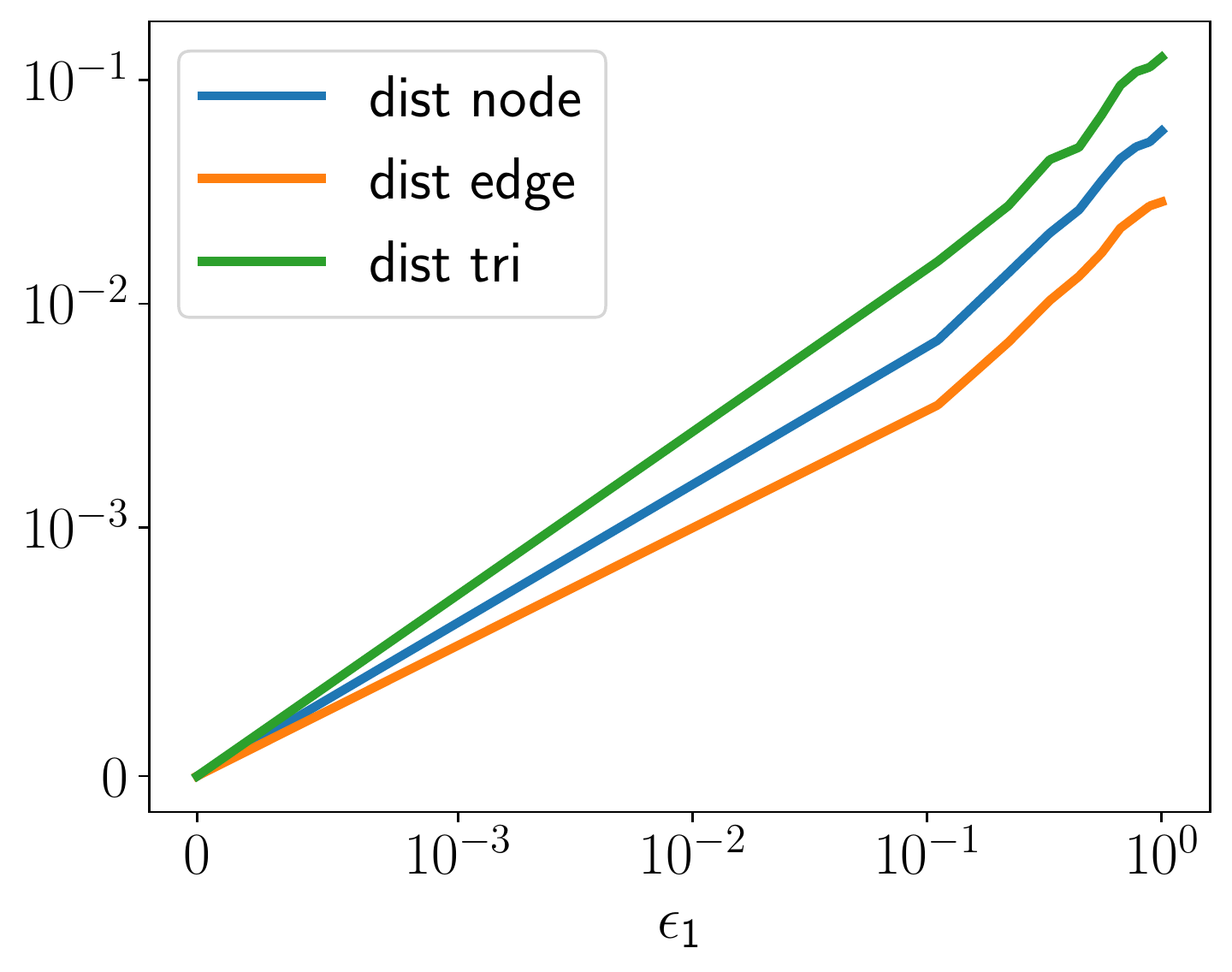}
    \caption{Edge pert. $\bbE_1, L=1$}
  \end{subfigure}
  \begin{subfigure}{0.19\linewidth}
    \includegraphics[width=\linewidth]{figures/pert_2/0layers_2orders_32features.pdf}
    \caption{Tri. pert. $\bbE_2, L=1$}
  \end{subfigure}
  \begin{subfigure}{0.19\linewidth}
    \includegraphics[width=\linewidth]{figures/pert_2/1layers_2orders_32features.pdf}
    \caption{Tri. pert. $\bbE_2, L=2$}
  \end{subfigure}
  \caption{The stabilities of different simplicial outputs are dependent on each other.}
  \label{fig:mutual_dependence}
\end{figure}

\textbf{Number of Layers.} \cref{fig:stability_layers} shows that the stability of SCCNNs degrades as the number of layers increases as studied in \cref{thm:stability}. 
\begin{figure}[ht!]
  \centering
  \begin{subfigure}{0.19\linewidth}
    \includegraphics[width=\linewidth]{figures/pert_1/0layers_2orders_32features.pdf}
    \caption{$\bbE_1, L=1$}
  \end{subfigure}
  \begin{subfigure}{0.19\linewidth}
    \includegraphics[width=\linewidth]{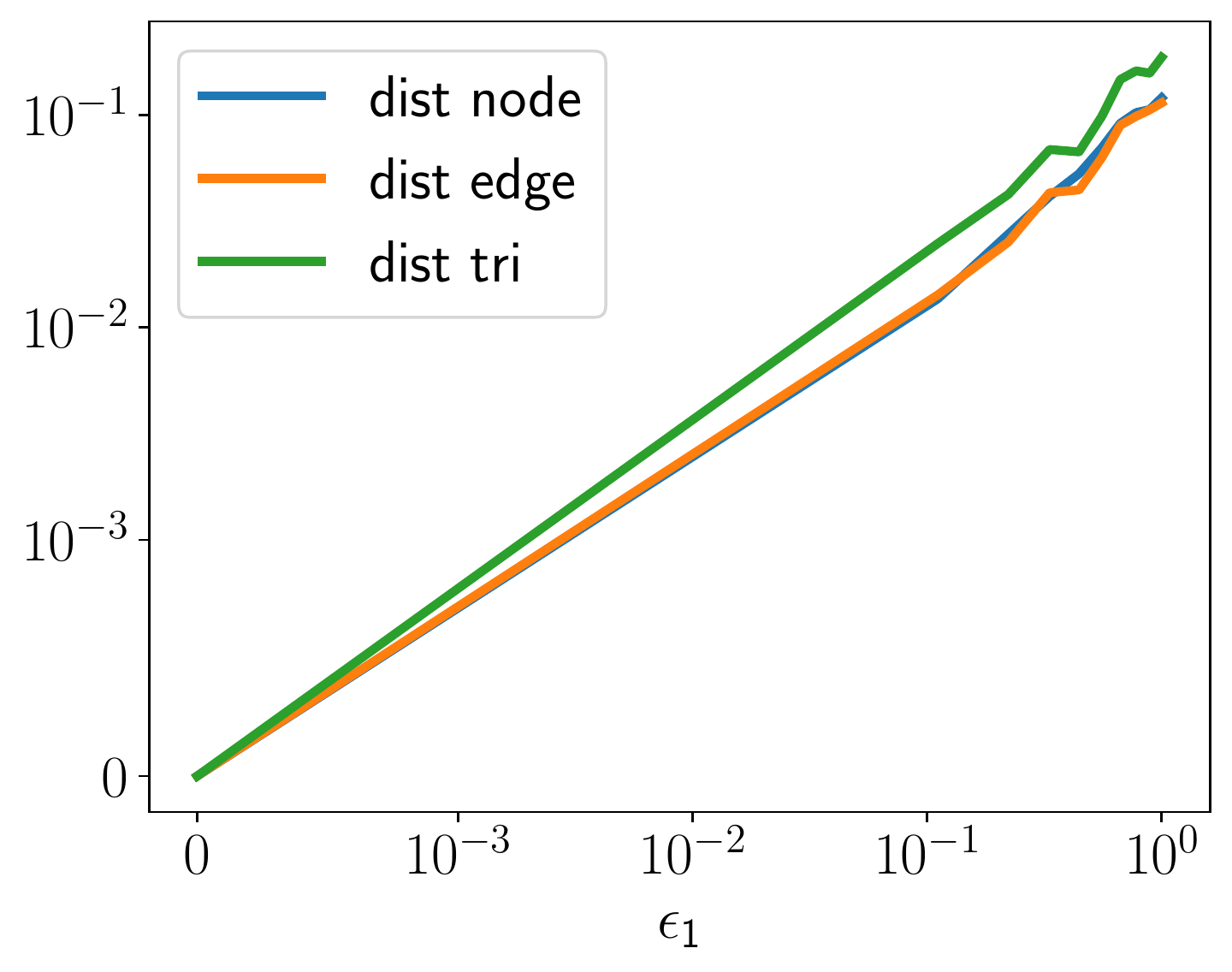}
    \caption{$\bbE_1, L=2$}
  \end{subfigure}
  \begin{subfigure}{0.19\linewidth}
    \includegraphics[width=\linewidth]{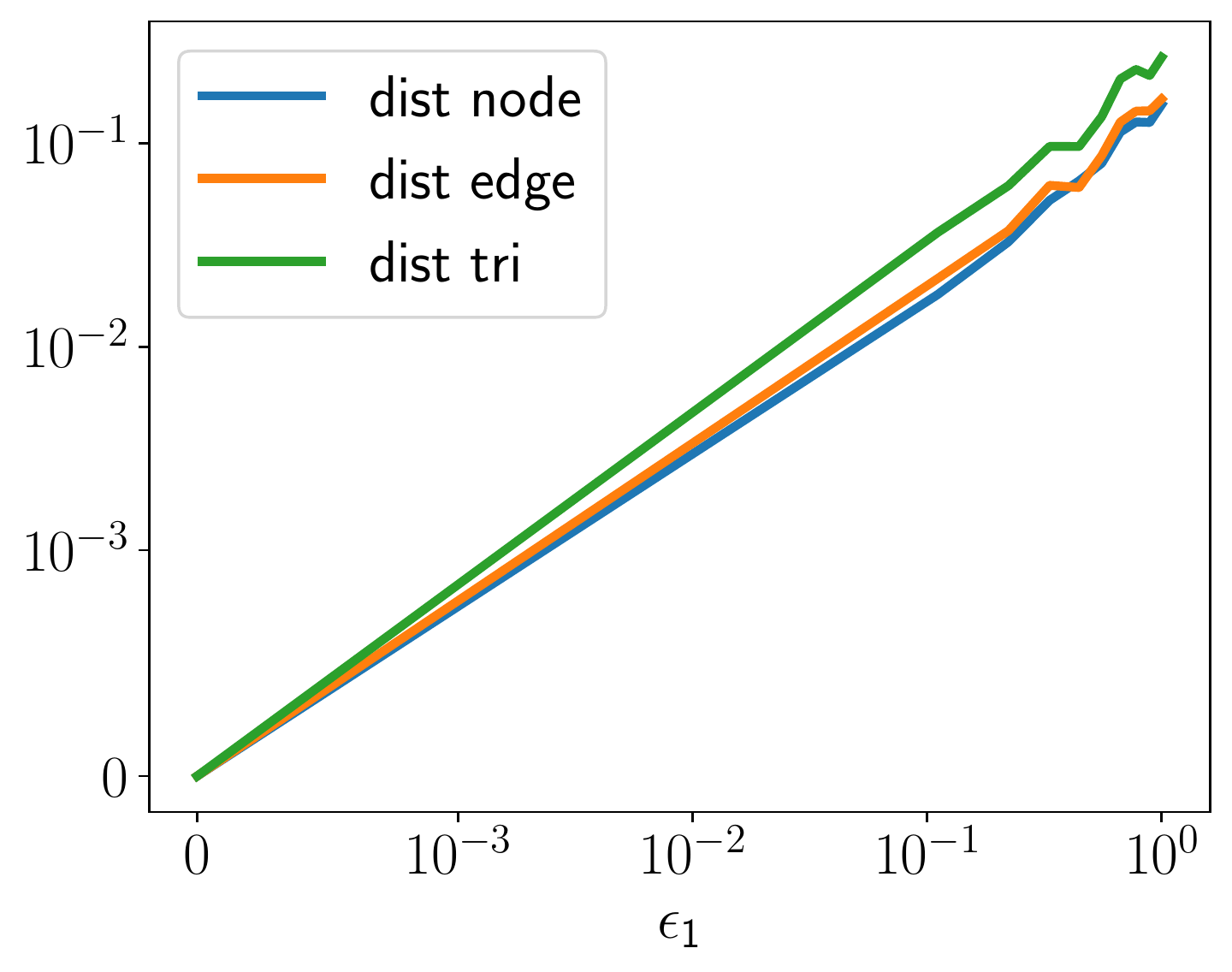}
    \caption{$\bbE_1, L=3$}
  \end{subfigure}
  \begin{subfigure}{0.19\linewidth}
    \includegraphics[width=\linewidth]{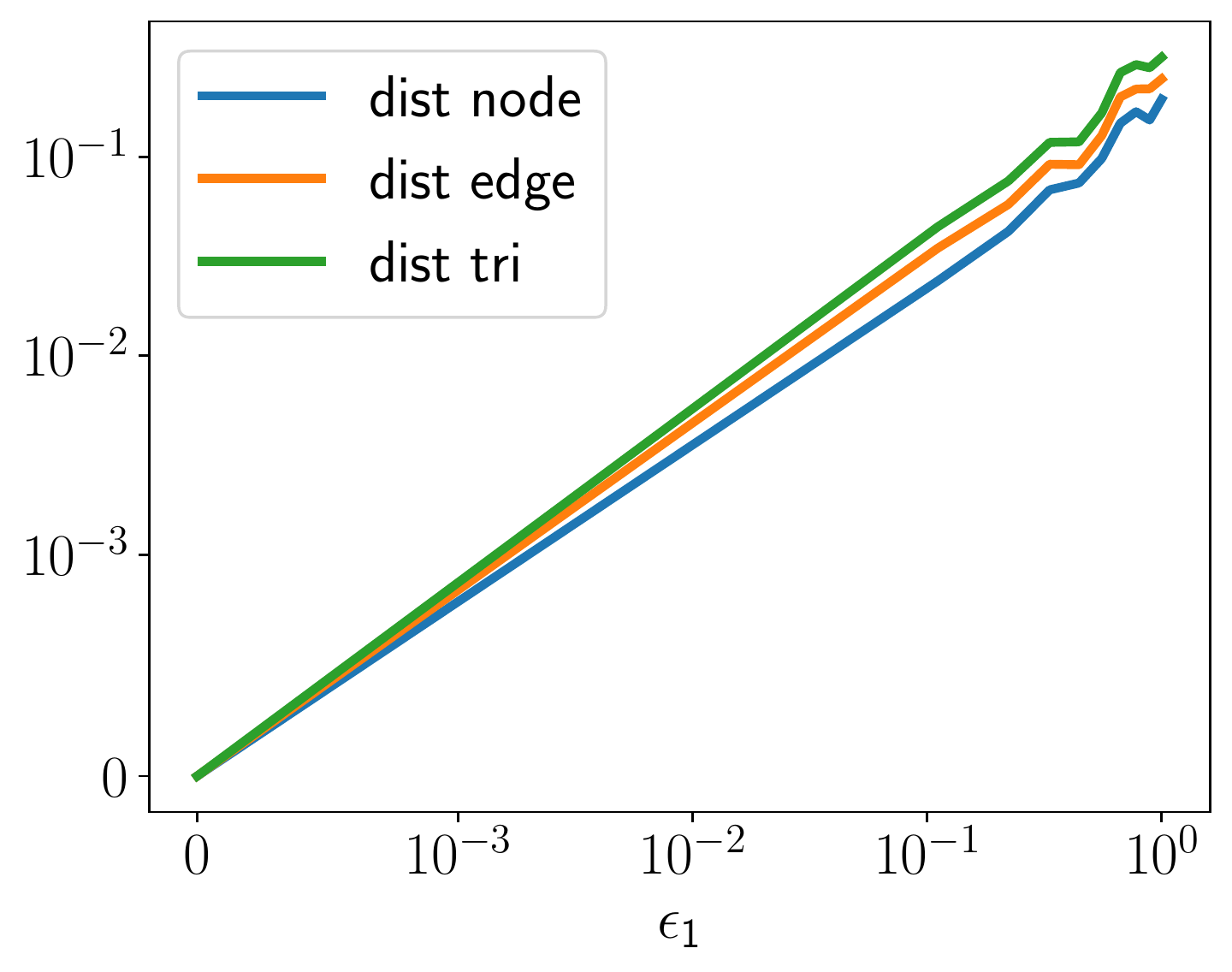}
    \caption{$\bbE_1, L=4$}
  \end{subfigure}
  \begin{subfigure}{0.19\linewidth}
    \includegraphics[width=\linewidth]{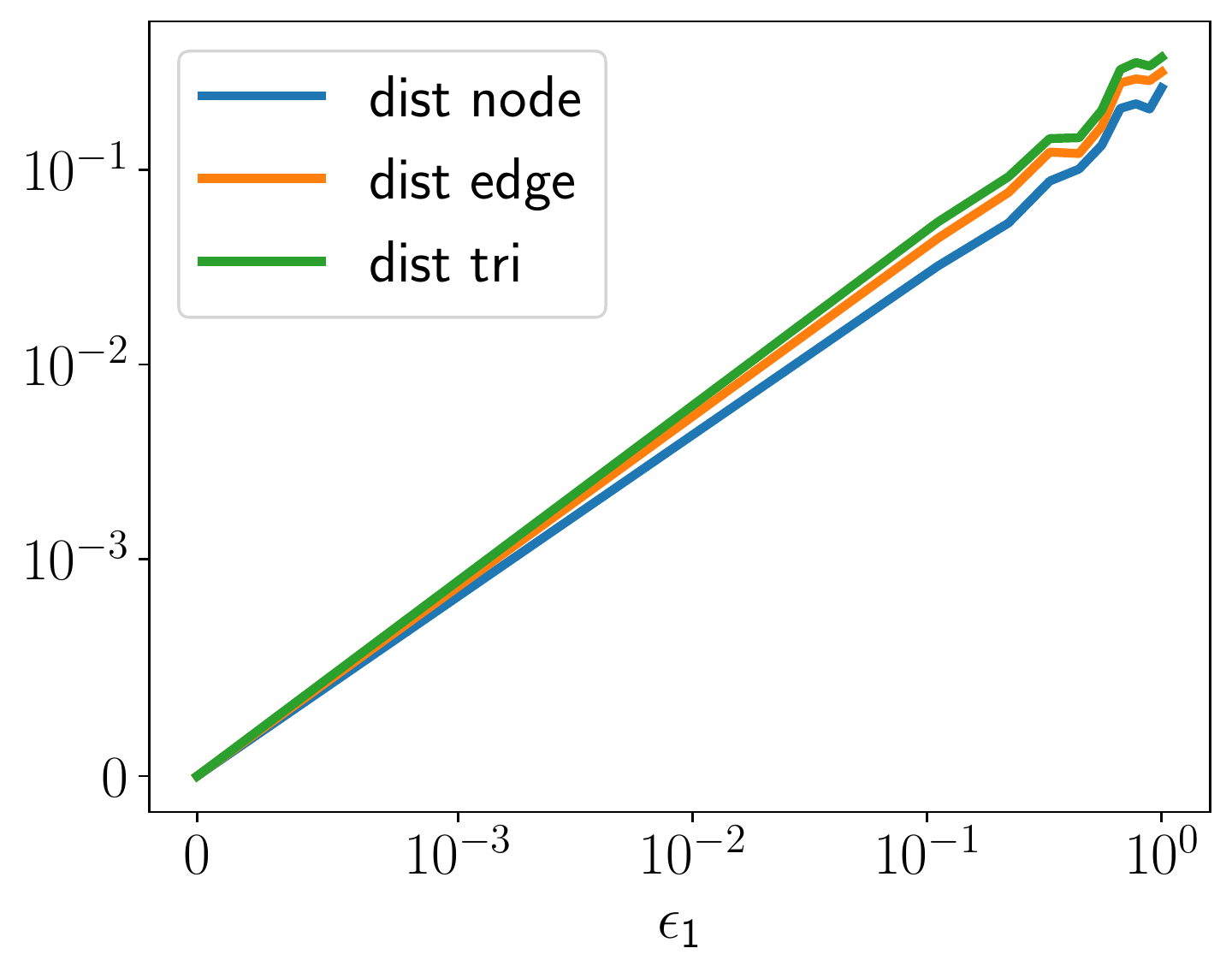}
    \caption{$\bbE_1, L=5$}
  \end{subfigure}
  \caption{The stability of SCCNNs in terms of different numbers of layers. We consider perturbations on edge weights.}
  \label{fig:stability_layers}
\end{figure}

\textbf{Convolution Order $T$.} As the order of SCFs becomes larger, the SCFs become more selective at a cost of stability as shown in \cref{fig:il_illustration2}. We see that the stability does not deteriorate when $T$ becomes larger. This is owing to the information spillage of the nonlinearity in SCCNNs such that information at higher frequencies where the SCFs cannot be more selective is spread over the lower frequency where the SCFs can be made arbitrarily selective, as studied by \citet{gama2019stability}. 

\begin{figure}[ht!]
  \centering
  \begin{subfigure}{0.19\linewidth}
    \includegraphics[width=\linewidth]{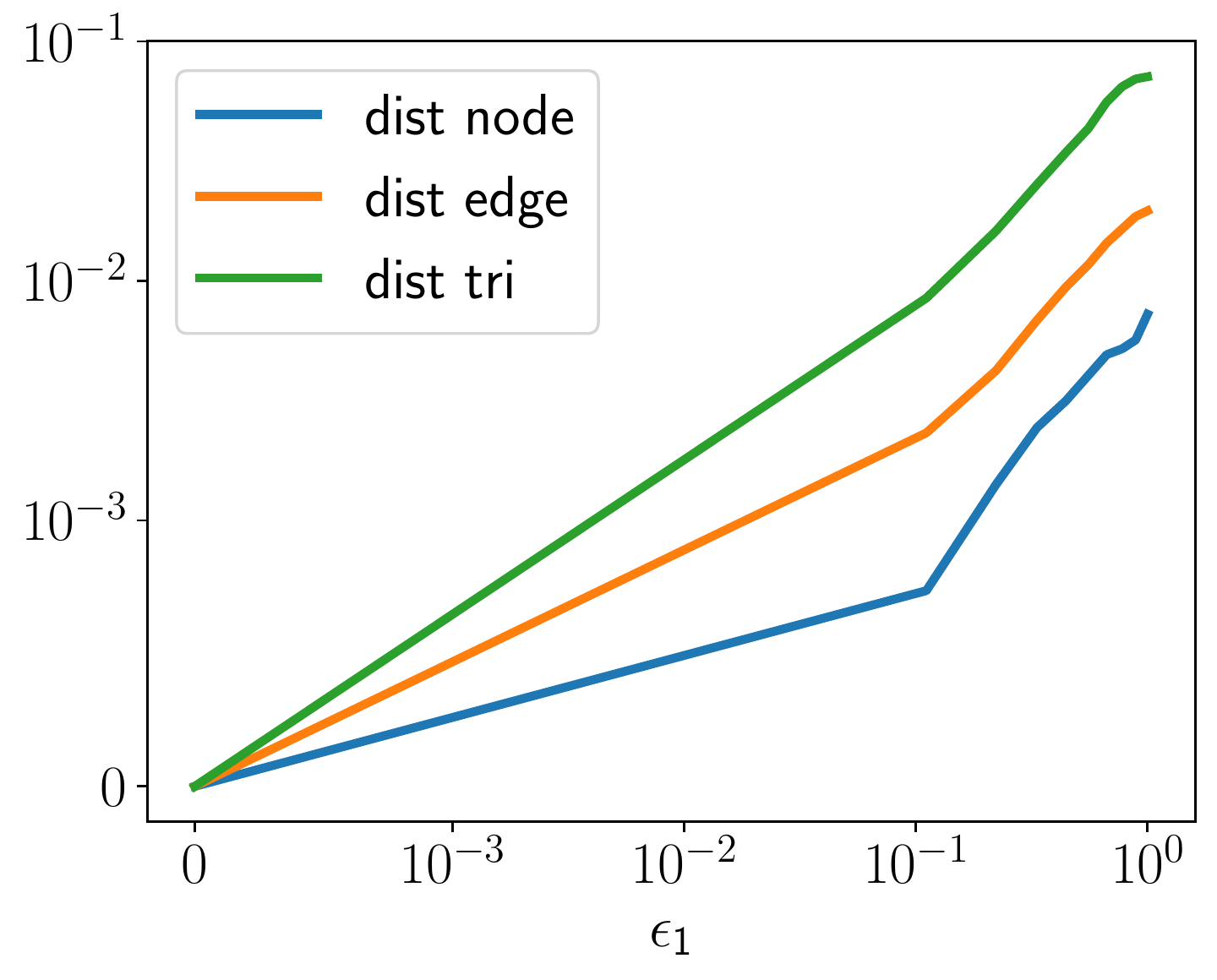}
    \caption{$\bbE_1, T=0$}
  \end{subfigure}
  \begin{subfigure}{0.19\linewidth}
    \includegraphics[width=\linewidth]{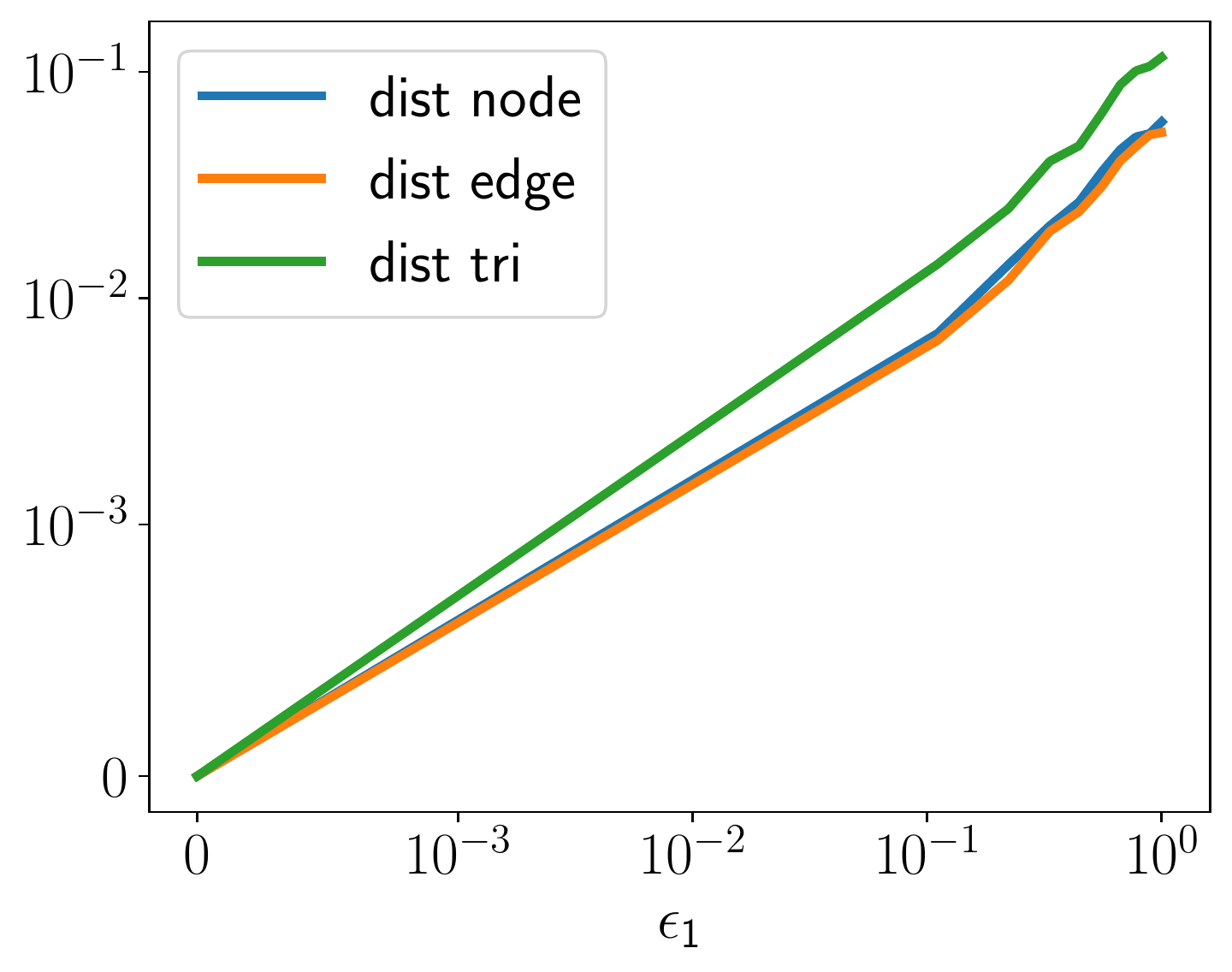}
    \caption{$\bbE_1, T=1$}
  \end{subfigure}
  \begin{subfigure}{0.19\linewidth}
    \includegraphics[width=\linewidth]{figures/pert_1/1layers_2orders_32features.pdf}
    \caption{$\bbE_1, T=2$}
  \end{subfigure}
  \begin{subfigure}{0.19\linewidth}
    \includegraphics[width=\linewidth]{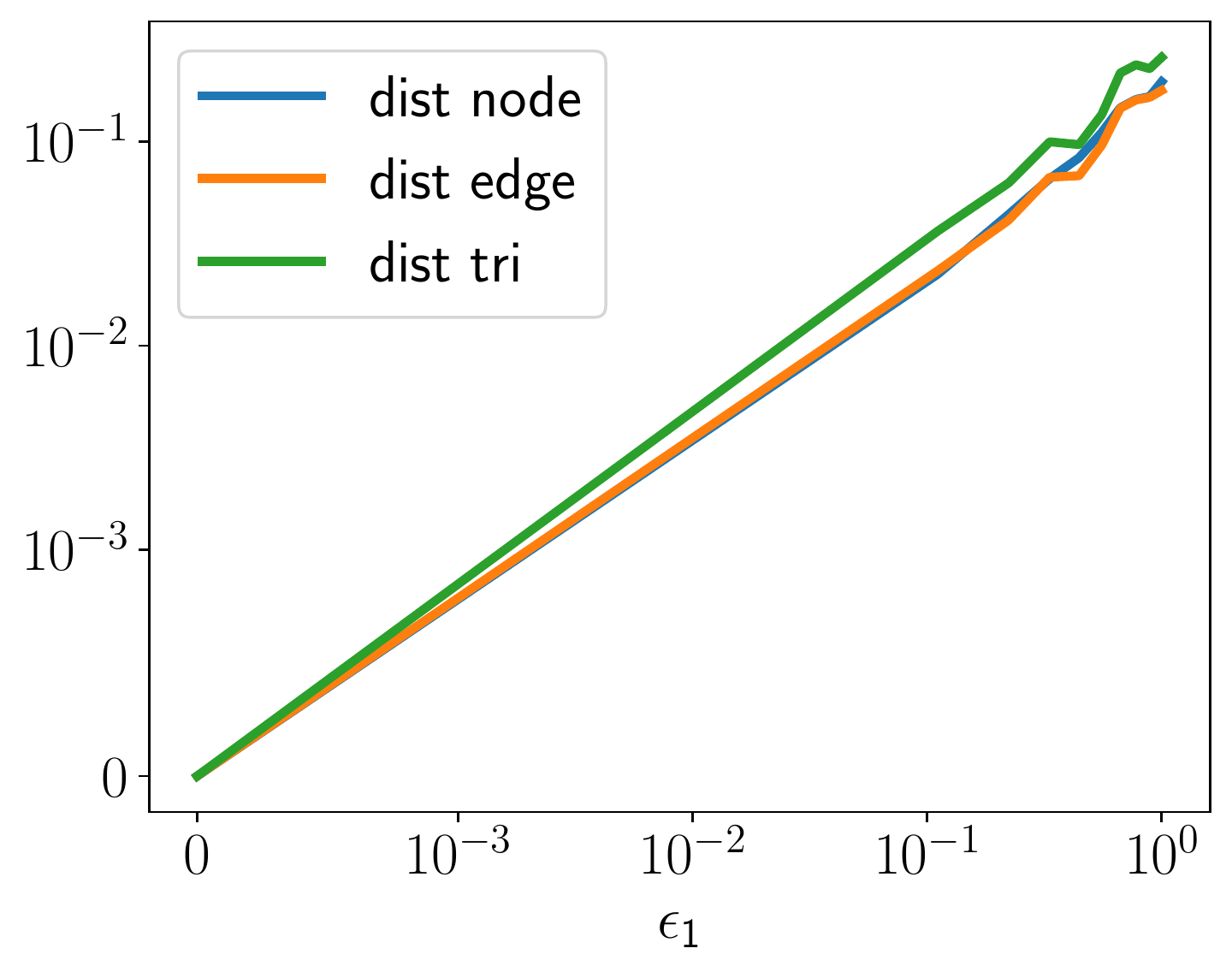}
    \caption{$\bbE_1, T=3$}
  \end{subfigure}
  \begin{subfigure}{0.19\linewidth}
    \includegraphics[width=\linewidth]{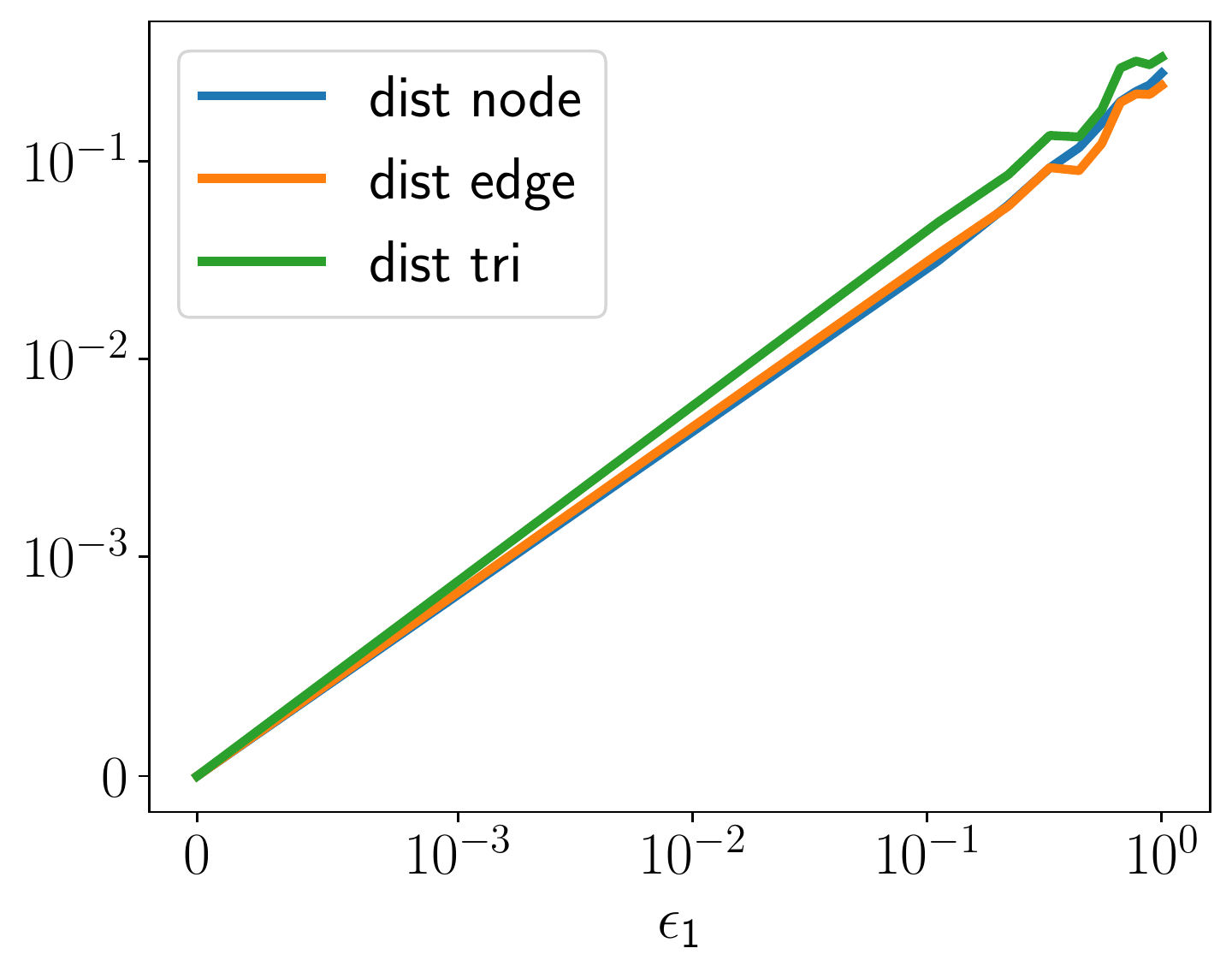}
    \caption{$\bbE_1, T=4$}
  \end{subfigure}
  \caption{The stability of SCCNNs in terms of the convolution orders. We consider perturbations on edge weights.}
  \label{Fig:stability_conv_orders}
\end{figure}

From this study, we see the advantage of using higher-order simplicial convolutions in the SCCNNs, which improves the expressive power of SCCNNs without sacrificing the stability much. In this case, it requires a small number of layers such that the stability is guaranteed.

\subsection{Trajectory Prediction}\label{app:traj_pred_results}
Machine learning based methods are commonly used in trajectory prediction, e.g., CNN, RNN, and GNN \citep{nikhil2018convolutional, rudenko2020human, wu2017modeling, cordonnier2019extrapolating}. While trajectories can be naturally viewed as edge flows \citep{ghosh2018topological, schaub2020random}, PSNN by \citet{roddenberry2021principled}, as an instance of the SCCNN, was the first attempt to use the SC model. Here, we first formulate the trajectory prediction problem proposed by \citet{roddenberry2019hodgenet}, then we evaluate the performance. 

\subsubsection{Problem Formulation}
A trajectory of length $m$ can be modeled as a sequence of nodes $[v_0,v_1,\dots,v_{m-1}]$ in an SC. The task is to predict the next node $v_m$ from the neighbors of $v_{m-1}$, $\ccalN_{v_{m-1}}$. The algorithm in \citet{roddenberry2021principled} first represents the trajectory equivalently as a sequence of oriented edges $[[v_0,v_1],[v_1,v_2],\dots,[v_{m-2},v_{m-1}]]$. Then, an edge flow $\bbx_1$ is defined, whose value on an edge $e$ is $[\bbx_1]_e=1$ if edge $e$ is traversed by the trajectory in a forward direction, $[\bbx_1]_e=-1$ if edge $e$ is traversed in a backward direction by the trajectory, and $[\bbx_1]_e=0$, otherwise. 

With the trajectory flow $\bbx_1$ as the input, together with zero inputs on the nodes and triangles, an SCCNN of order two is used to generate a representation $\bbx_1^L$ of the trajectory, which is the output on edges. This is followed by a projection step 
  $\bbx_{0,\rmu}^L= \bbB_1\bbW\bbx_1^L $,
where the output is first passed through a linear transformation via $\bbW$, then projected into the node space via $\bbB_1$. Lastly, a distribution over the candidate nodes $\ccalN_{v_{m-1}}$ is computed via a softmax operation, $\bbn_j=\text{softmax}([\bbx_{0,\rmu}^L]_j),j\in\ccalN_{v_{m-1}}$. The best candidate is selected as $v_m = \text{argmax}_{j} \bbn_j$. We refer to \citet[Alg. S-2]{roddenberry2019hodgenet} for more details. 

Given that an SCCNN of order two generates outputs also on nodes, we can directly apply the node feature output $\bbx_0^L$ to compute a distribution over the candidate nodes $\ccalN_{v_{m-1}}$ without the projection step. We refer to this as SCCNN-Node, and the method of using the edge features with the projection step as SCCNN-Edge. 

\subsubsection{Model}
In this experiment, we consider the following methods: 1) PSNN by \citet{roddenberry2021principled}; 2) SNN by \citet{ebli2020simplicial}; 3) SCNN by \citet{yang2021simplicial} where we consider different lower and upper convolution orders $T_\rmd,T_\rmu$; and 4) S2CCNN (Bunch) by \citet{bunch2020simplicial} where we consider both the node features and edge features, namely, Bunch-Node and Bunch-Edge, as SCCNN. 

\subsubsection{Data and Experimental Setup}
\begin{wrapfigure}{r}{0.5\columnwidth}
  \vskip -0.1in 
  \centering
  \begin{subfigure}{0.245\columnwidth}
    \includegraphics[width=\linewidth]{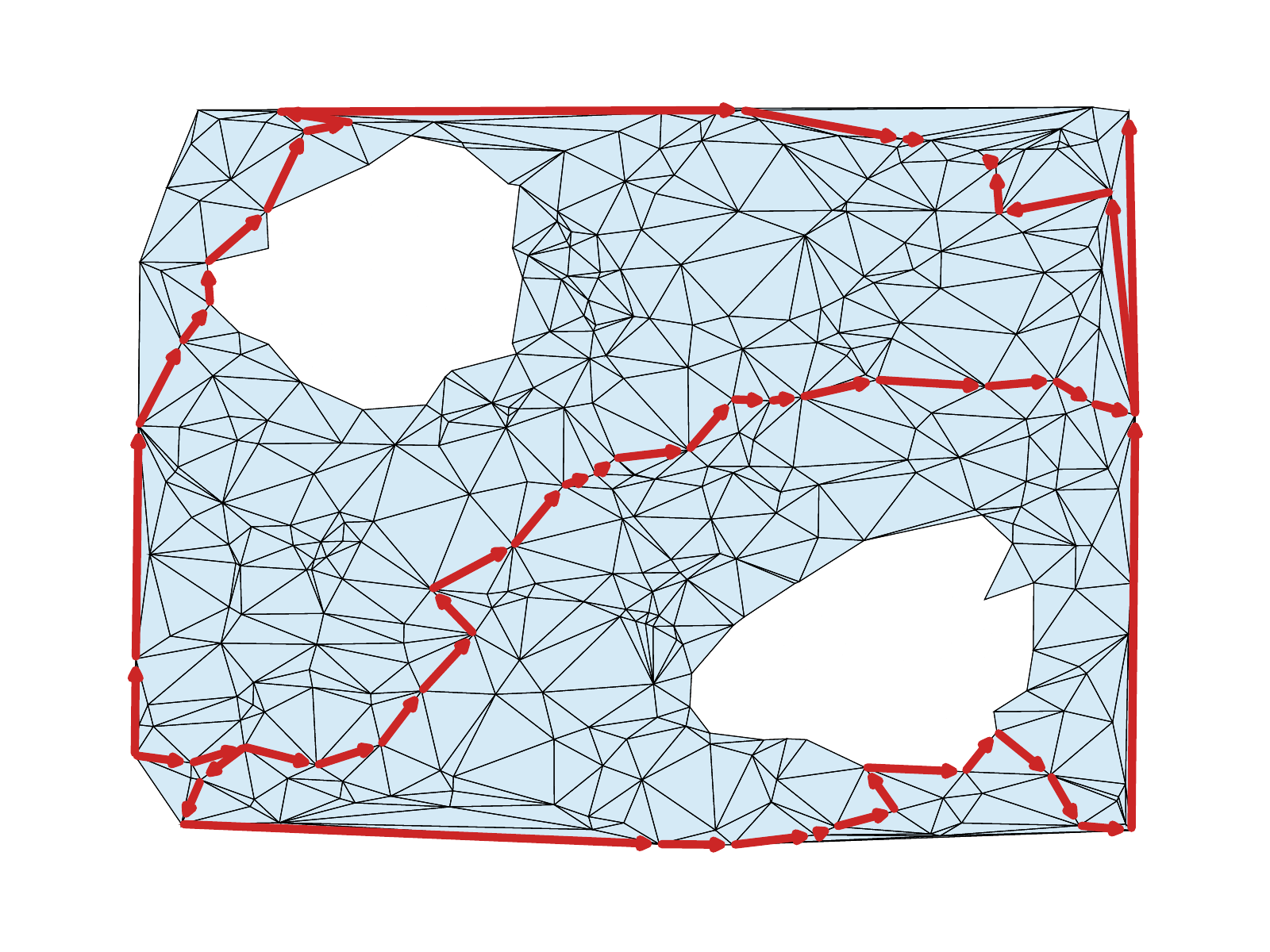}
    \label{fig:synthetic_trajectories}
  \end{subfigure}
  \begin{subfigure}{0.245\columnwidth}
    \includegraphics[width=\linewidth]{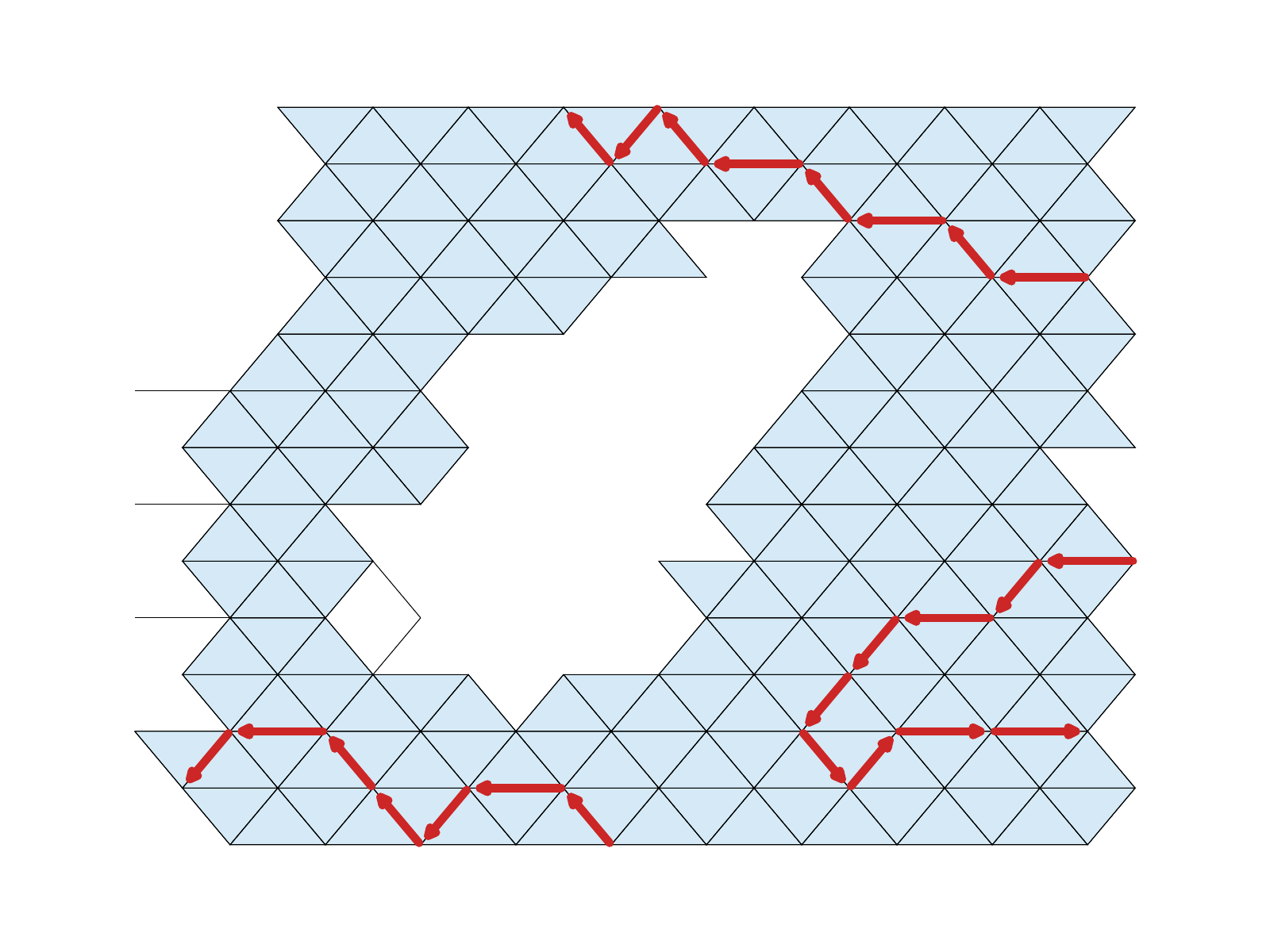}
    \label{fig:real_trajectories}
   \end{subfigure}
   \vskip -0.2in 
  \caption{Synthetic and real trajectories illustrations.}
 \end{wrapfigure}
\textbf{Synthetic Data.} Following the procedure in \citet{schaub2020random}, we generate 1000 trajectories as follows. First, we create an SC with two ``holes'' by uniformly drawing 400 random points in the unit square, and then a Delaunay triangulation is applied to obtain a mesh, followed by the removal of nodes and edges in two regions, as depicted in \cref{fig:synthetic_trajectories}. To generate a trajectory, we consider a starting point at random in the lower-left corner, and then connect it via a shortest path to a random point in the upper left, center, or lower-right region, which is connected to another random point in the upper-right corner via a shortest path. 

In this experiment, we consider the random walk Hodge Laplacians $\bbL_{0}, \bbL_{1,\rmd}, \bbL_{1,\rmu}$ and $\bbL_{2,\rmd}$ defined in \eqref{eq.normalized_l0}, \eqref{eq.normalized_l1} and \eqref{eq.normalized_l2d}. For Bunch method, we set the shifting matrices as the simplicial adjacency matrices defined in \citet{bunch2020simplicial}. We consider different NNs with three intermediate layers where each layer contains $F=16$ intermediate features. The $\text{tanh}$ nonlinearity is used such that the orientation equivariance holds. The final projection $\bbn$ generates a node feature of dimension one. 
In the 1000-epoch training, we use the cross-entropy loss function between the output $\bbd$ and the true candidate and we consider an adam optimizer with a learning rate of $0.001$ and a batch size 100. To avoid overfitting, we apply a weight decay of $5\cdot 10^{-6}$ and an early stopping.

As done in \citet{roddenberry2021principled}, besides the standard trajectory prediction task, we also perform a reverse task where the training set remains the same but the direction of the trajectories in the test set is reversed and a generalization task where the training set contains trajectories running along the upper left region and the test set contains trajectories around the other region. We evaluate the correct prediction ratio by averaging the performance over 10 different data generations. 

\textbf{Real Data.} We also consider the Global Drifter Program dataset,\footnote{Data available at \url{http://www.aoml.noaa.gov/envids/gld/}.} localized around Madagascar. It consists of ocean drifters whose coordinates are logged every 12 hours. The ocean is meshed with hexagons, as shown in   \cref{fig:real_trajectories}, and we consider the trajectories of drifters based on their presence in the hexagonal meshes following the procedure in \citet{schaub2020random}. An SC can then be created by treating each mesh as a node, connecting adjacent meshes via an edge and filling the triangles, where the ``hole'' is yielded by the island. Following the process in \citet{roddenberry2021principled}, it results in 200 trajectories and we use 180 of them for training. In the training, a batch size of 10 is used and no weight decay is used. The rest experiment setup remains the same as the synthetic case.

\subsubsection{Results}

We report the prediction accuracy of different tasks for both datasets in \cref{tab:synthetic_trajectory_prediction}. We first investigate the effects of applying higher-order SCFs in the simplicial convolution and accounting for the lower and upper contributions. From the standard accuracy for both datasets, we observe that increasing the convolution orders improves the prediction accuracy, e.g., SCNNs become better as the orders $T_\rmd,T_\rmu$ increase and perform always better than PSNN, and SCCNNs better than Bunch. Also, differentiating the lower and upper convolutions does help improve the performance as SCNN of orders $T_\rmd=T_\rmu=3$ performs better than SNN of $T=3$.

However, accounting for the node and triangle contributions in the SCCNN does not help the prediction compared to the SCNNs, likewise for Bunch compared to PSNN. This is due to the zero node and triangle inputs because there are no available node and triangle features. Similarly, the prediction directly via the node output features is not accurate compared to projection from edge features. 

Moreover, we also observe that the performance of SCCNNs that are trained with the same data does not deteriorate in the reverse task because the orientation equivariance ensures SCCNNs to be unaffected by the orientations of the simplicial data. Lastly, we see that, like other NNs on SCs, SCCNNs have good transferability to the unseen data. 

\begin{table}[ht!]
  \caption{Trajectory Prediction Accuracy. \emph{(Left)}: Synthetic trajectory in the standard, reverse and generalization tasks. \emph{(Right)}: Ocean drifter trajectories. For SCCNNs, we set the lower and upper convolution orders $T_\rmd,T_\rmu$ to be the same as $T$. }
  \label{tab:synthetic_trajectory_prediction}
  \vskip 0.15in
  \begin{center}
  \begin{small}
  \begin{sc}
  \begin{tabular}{lcccr}
  \toprule
  Methods & Standard & Reverse & Generalization & Parameters \\
  \midrule
  PSNN & 63.1$\pm$3.1 & 58.4$\pm$3.9 & 55.3$\pm$2.5 &  --- \\
  SCNN & 65.6$\pm$3.4 & 56.6$\pm$6.0 & 56.1$\pm$3.6 & $T_\rmd=T_\rmu=2$ \\ 
  SCNN & 66.5$\pm$5.8 & 57.7$\pm$5.4 & 60.6$\pm$4.0 & $T_\rmd=T_\rmu=3$ \\
  SCNN & 67.3$\pm$2.3 & 56.9$\pm$4.8 & 59.4$\pm$4.2 & $T_\rmd=T_\rmu=4$ \\ 
  SCNN & 67.7$\pm$1.7 & 55.3$\pm$5.3 & 61.2$\pm$3.2 & $T_\rmd=T_\rmu=5$ \\ 
  SNN  & 65.5$\pm$2.4 & 53.6$\pm$6.1 & 59.5$\pm$3.7 & $T=3$ \\ 
  \midrule
  Bunch-Node & 35.4$\pm$3.4 & 38.1$\pm$4.6 & 29.0$\pm$3.0 & --- \\
  Bunch-Edge & 62.3$\pm$4.0 & 59.6$\pm$6.1 & 53.9$\pm$3.1 & --- \\
  SCCNN-Node & 46.8$\pm$7.3 & 44.5$\pm$8.2 & 31.9$\pm$5.0 & $T=1$ \\
  SCCNN-Edge & 64.6$\pm$3.9 & 57.2$\pm$6.3 & 54.0$\pm$3.0 & $T=1$ \\
  SCCNN-Node & 43.5$\pm$9.6 & 44.4$\pm$7.6 & 32.8$\pm$2.6 & $T=2$ \\
  SCCNN-Edge & 65.2$\pm$4.1 & 58.9$\pm$4.1 & 56.8$\pm$2.4 & $T=2$ \\
  \bottomrule
  \end{tabular}
  \hspace{0.1in}
  \begin{tabular}{cr}
    \toprule
     Standard  & Parameters \\
    \midrule
    49.0$\pm$8.0 &  --- \\
    52.5$\pm$9.8 & $T_\rmd=T_\rmu=2$ \\ 
    52.5$\pm$7.2 & $T_\rmd=T_\rmu=3$ \\
    52.5$\pm$8.7 & $T_\rmd=T_\rmu=4$ \\ 
    53.0$\pm$7.8 & $T_\rmd=T_\rmu=5$ \\ 
    52.5$\pm$6.0 & $T=3$ \\ 
    \midrule
    35.0$\pm$5.9 & --- \\
    46.0$\pm$6.2 & --- \\
    40.5$\pm$4.7 & $T=1$ \\
    52.5$\pm$7.2 & $T=1$ \\
    45.5$\pm$4.7 & $T=2$ \\
    54.5$\pm$7.9 & $T=2$ \\
    \bottomrule
    \end{tabular}
  \end{sc}
  \end{small}
  \end{center}
  \vskip -0.1in
\end{table}


\subsubsection{Integral Lipschitz Property}
We investigate the effect of the integral Lipschitz property of the SCFs in an NN on SC. To do so, given an NN on SCs with an SCF $\bbH_k$ for $k$-simplicial signals, we add the following integral Lipschitz regularizer to the loss function during training so to promote the integral Lipschitz property
\begin{equation}
  \begin{aligned}
    r_{\rm{IL}}   = \lVert \lambda_{k,\rm{G}} \tilde{h}_{k,\rm{G}}^\prime(\lambda_{k,\rm{G}})  \rVert + \lVert \lambda_{k,\rm{C}} \tilde{h}_{k,\rm{C}}^\prime (\lambda_{k,\rm{C}}) \rVert   =
     \Bigg\lVert  \sum_{t=0}^{T_\rmd} t w_{k,\rmd,t} \lambda_{k,\rm{G}}^t \Bigg \rVert + \Bigg \lVert  \sum_{t=0}^{T_\rmu} t w_{k,\rmu,t} \lambda_{k,\rm{C}}^t \Bigg\rVert
  \end{aligned}
\end{equation}
for $\lambda_{k,\rm{G}}\in\{\lambda_{k,{\rm{G}},i}\}_{i=1}^{N_{k,\rm{G}}}$ and $\lambda_{k,\rm{C}}\in\{\lambda_{k,{\rm{C}},i}\}_{i=1}^{N_{k,\rm{C}}}$, which are the gradient and curl frequencies. To avoid computing the eigendecomposition of the Hodge Laplacian, we can approximate the true frequencies by sampling certain number of points in the frequency band $(0,\lambda_{k,{\rm{G,m}}}]$ and $(0,\lambda_{k,\rm{C,m}}]$ where the maximal gradient and curl frequencies can be computed by efficient algorithms, e.g., power iteration \citep{watkins2007matrix,sleijpen2000jacobi}. 

Here, to illustrate that the integral Lipschitz property of the SCFs helps the stability of NNs on SCs, we consider the effect of regularizer $r_{\rm{IL}}$ against perturbations in PSNNs and SCNNs in \eqref{eq.scnn_yang} with different $T_\rmd$ and $T_\rmu$ for the standard synthetic trajectory prediction. The regularization weight on $r_{\rm{IL}}$ is set as $5\cdot 10^{-4}$ and the number of samples to approximate the frequencies is set such that the sampling interval is 0.01.  

\cref{fig:effect_il_regularizer} shows the prediction accuracy and the relative distance between the edge outputs of the NNs trained with and without the integral Lipschitz regularizer in terms of different levels of perturbations. We see that the integral Lipschitz regularizer helps the stability of the NNs, especially for large SCF orders, where the edge output is less influenced by the perturbations compared to without the regularizer. Meanwhile, SCNN with higher-order SCFs, e.g., $T_\rmd=T_\rmu=5$, achieves better prediction than PSNN (with one-step simplicial shifting), while maintaining a good stability with its output not influenced by perturbations drastically.

We also measure the lower and upper integral Lipschitz constants of the trained NNs across different layers and features, given by $\max_{\lambda_{k,\rm{G}}} |\lambda_{k,\rm{G}} \tilde{h}_{k,\rm{G}}(\lambda_{k,\rm{G}})| $ and $\max_{\lambda_{k,\rm{C}}} |\lambda_{k,\rm{C}} \tilde{h}_{k,\rm{C}}(\lambda_{k,\rm{C}})| $, shown in  
\cref{fig:il_constants_trajectory_prediction}. We see that the SCNN trained with $r_{\rm{IL}}$ indeed has smaller integral Lipschitz constants than the one trained without the regularizer, thus, a better stability, especially for NNs with higher-order SCFs.

\begin{figure*}[htp!]
  \centering
  \begin{subfigure}{0.246\linewidth}
    \includegraphics[width=\linewidth]{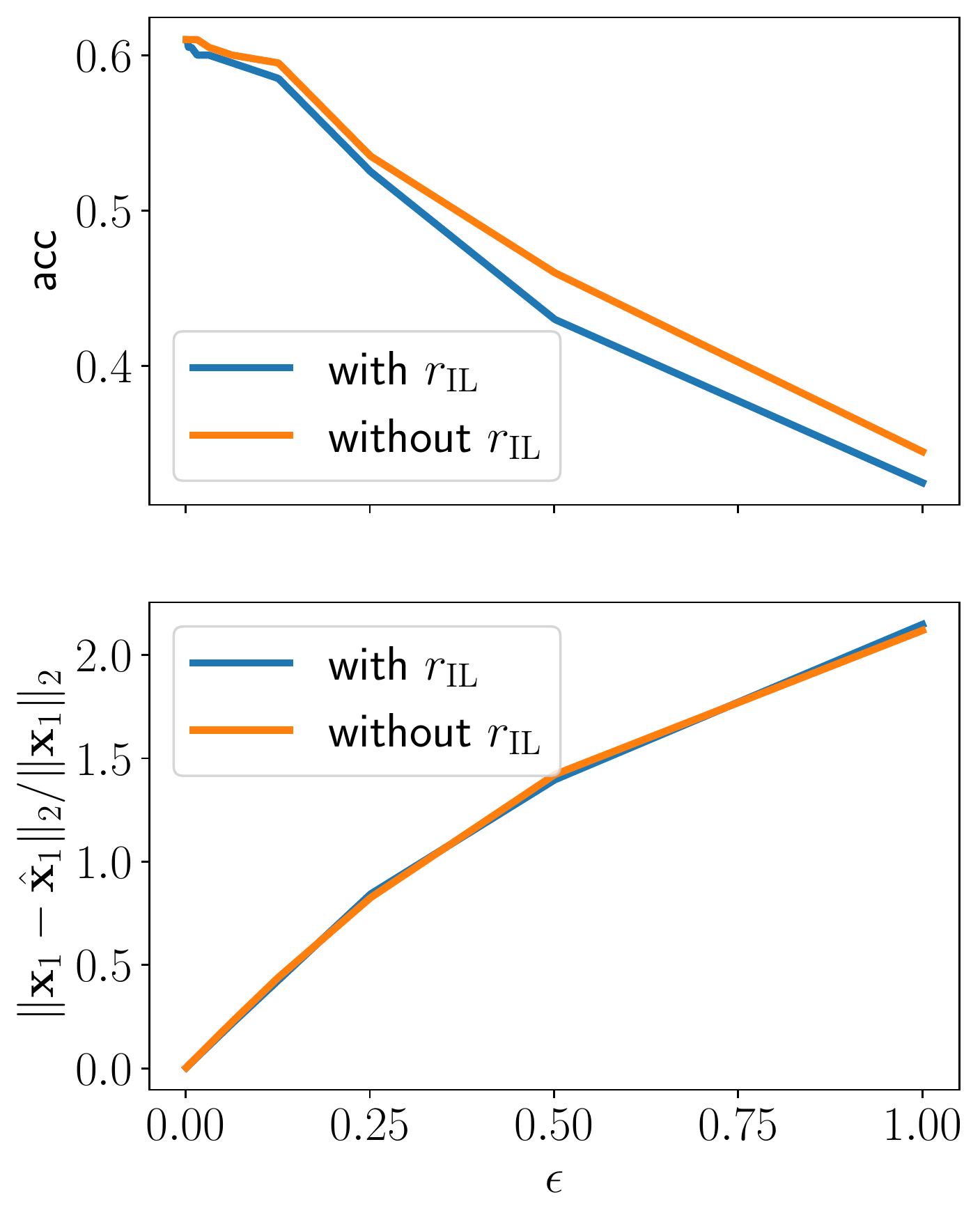}
    \caption{PSNN}
    \label{fig:pert_acc_dist_psnn}
  \end{subfigure}
  \begin{subfigure}{0.246\linewidth}
    \includegraphics[width=\linewidth]{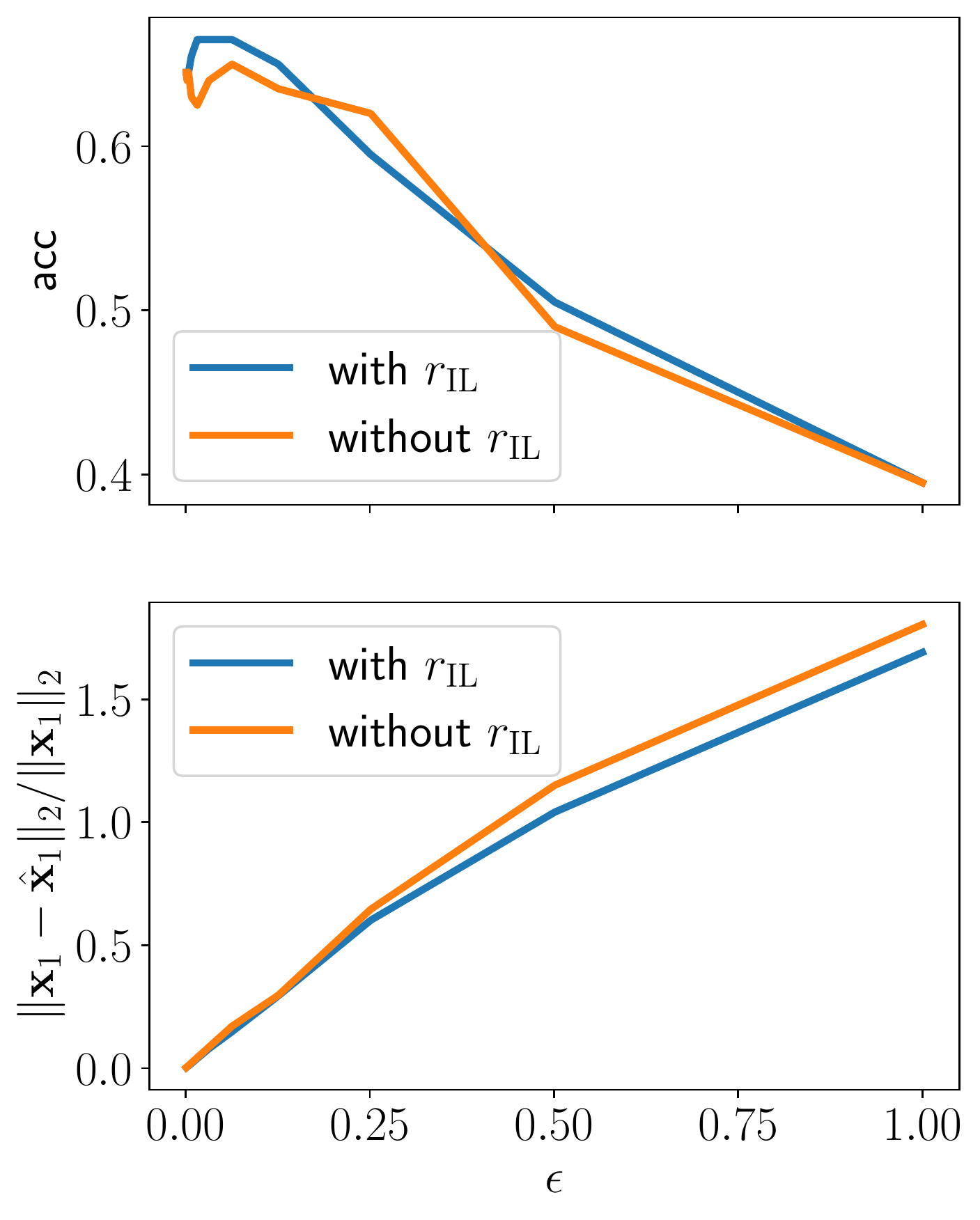}
    \caption{SCNN13}
    \label{fig:pert_acc_dist_scnn13}
  \end{subfigure}
  \begin{subfigure}{0.246\linewidth}
    \includegraphics[width=\linewidth]{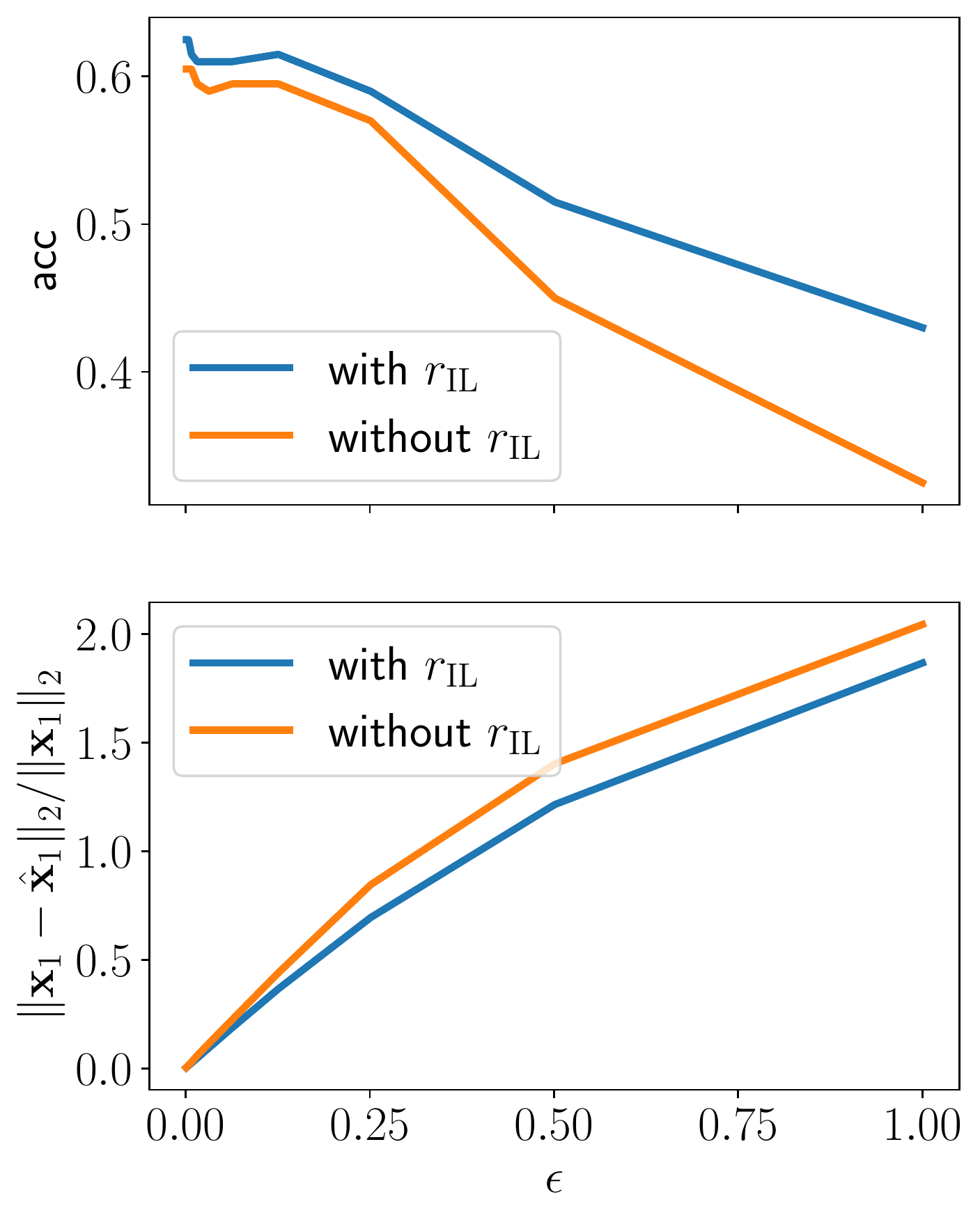}
    \caption{SCNN31}
    \label{fig:pert_acc_dist_scnn31}
  \end{subfigure}
  \begin{subfigure}{0.246\linewidth}
    \includegraphics[width=\linewidth]{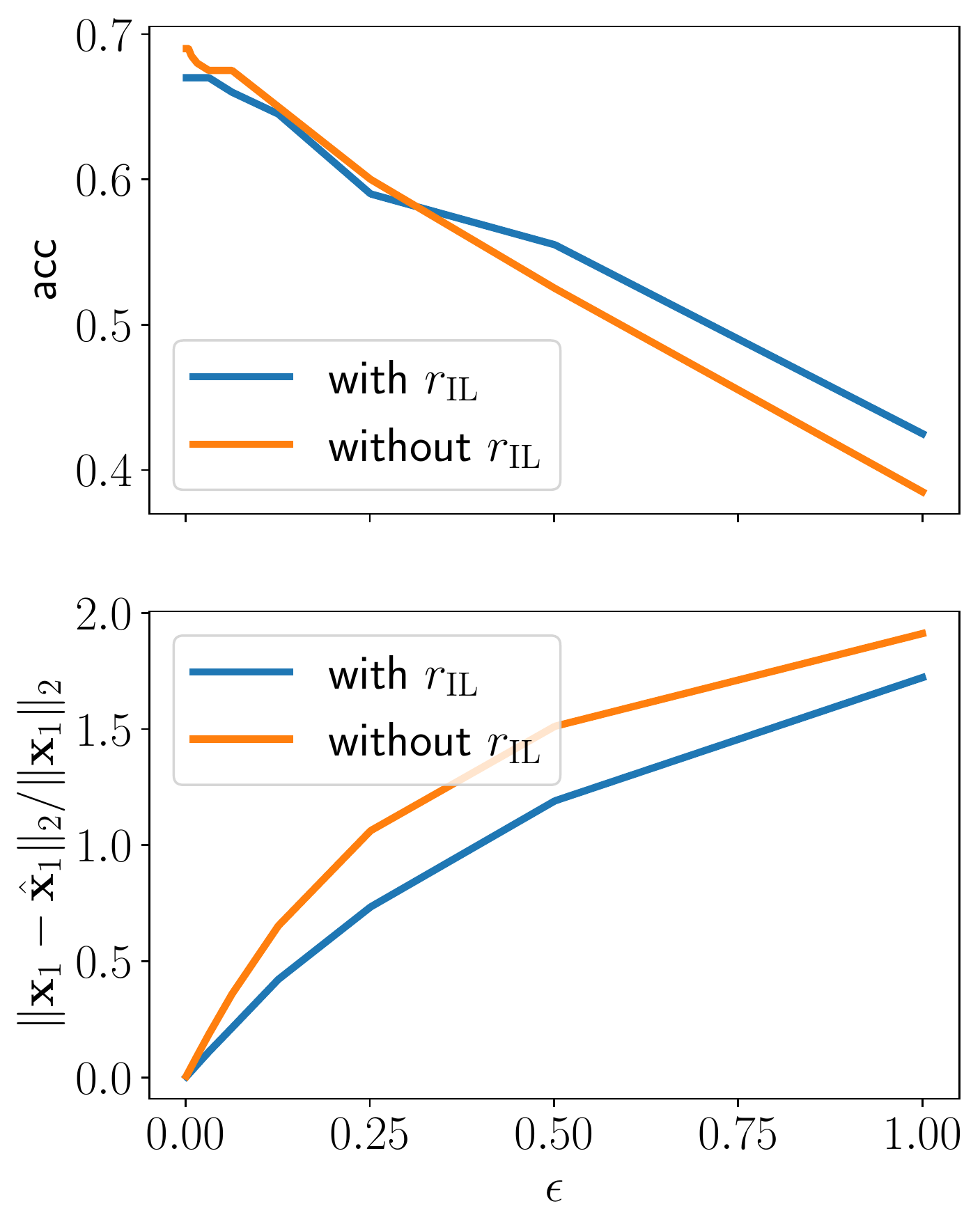}
    \caption{SCNN55}
    \label{fig:pert_acc_dist_scnn55}
  \end{subfigure}
  \caption{Effect of the integral Lipschitz regularizer $r_{\rm{IL}}$ in the task of synthetic trajectory prediction against different levels $\epsilon$ of random perturbations on $\bbL_{1,\rmd}$ and $\bbL_{1,\rmu}$. We show the accuracy (Top row) and the relative distance between the edge output (Bottom row) for different NNs on SCs with and without $r_{\rm{IL}}$. SCNN13 is the SCNN with $T_\rmd=1$ and $T_\rmu=3$.}
  \label{fig:effect_il_regularizer}
\end{figure*}

\begin{figure}[htp!]
  \centering
  \begin{subfigure}{0.246\linewidth}
    \includegraphics[width=\linewidth]{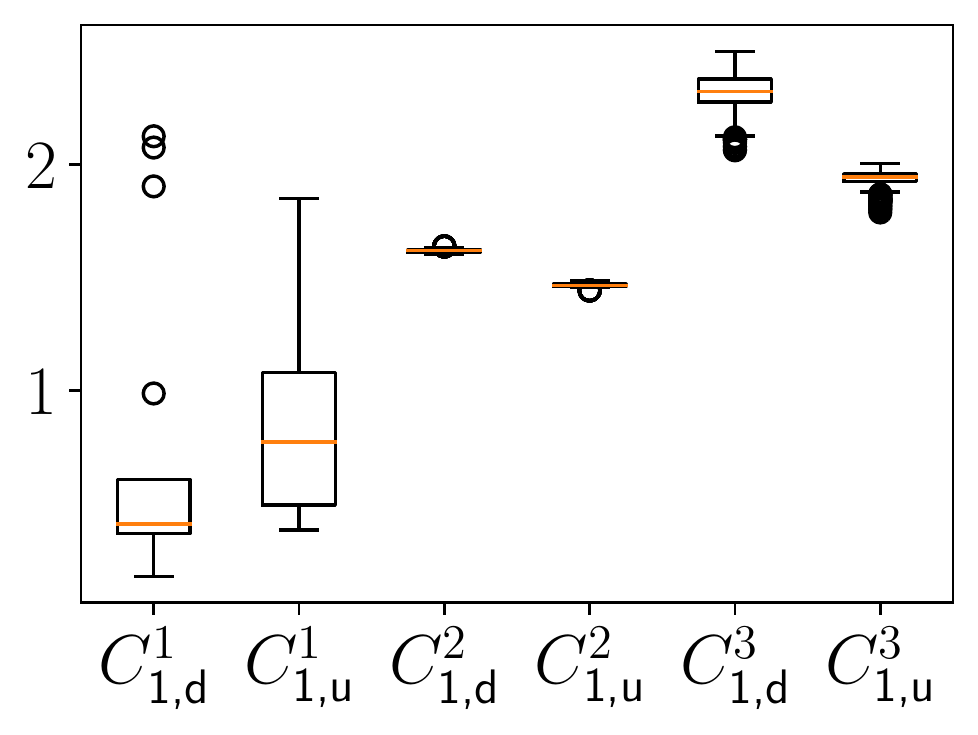}
    \caption{SCNN31, without $r_{\rm{IL}}$}
    \label{fig:il_constant_scnn31_no_il}
  \end{subfigure}
  \begin{subfigure}{0.246\linewidth}
    \includegraphics[width=\linewidth]{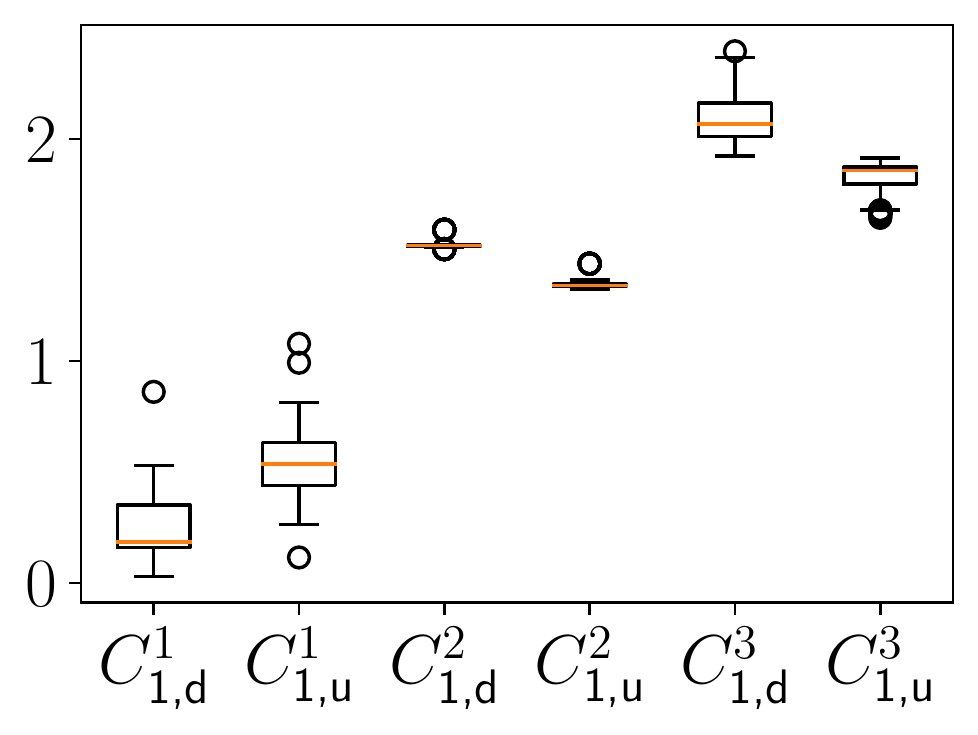}
    \caption{SCNN31, with $r_{\rm{IL}}$}
    \label{fig:il_constant_scnn31_il}
  \end{subfigure}
  \begin{subfigure}{0.246\linewidth}
    \includegraphics[width=\linewidth]{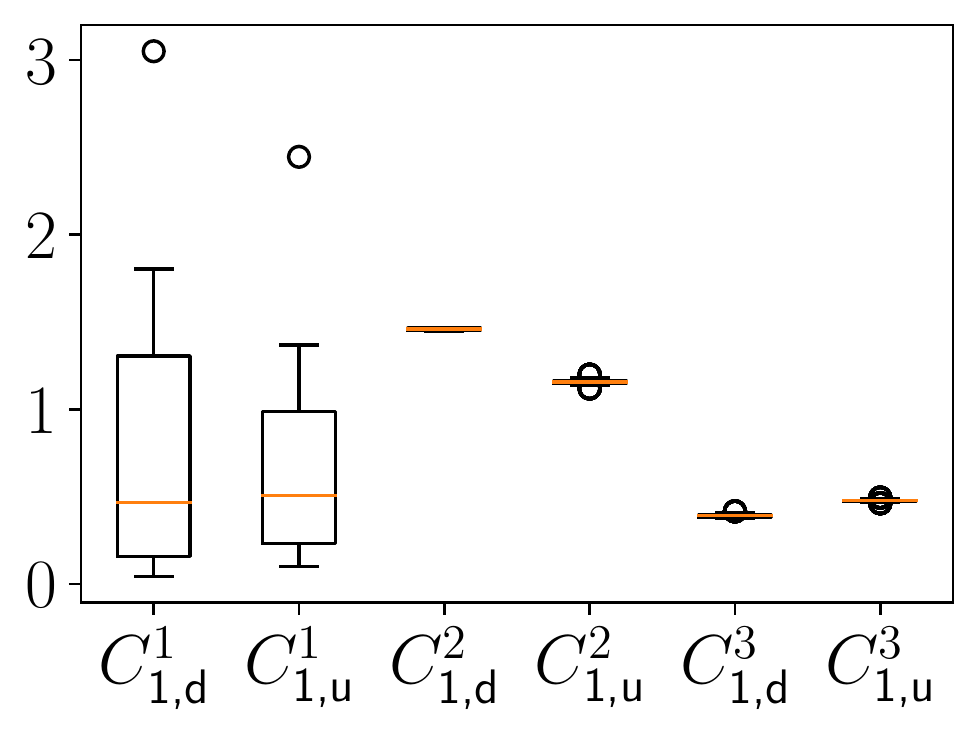}
    \caption{SCNN55, without $r_{\rm{IL}}$}
    \label{fig:il_constant_scnn5_no_il}
  \end{subfigure}
  \begin{subfigure}{0.246\linewidth}
    \includegraphics[width=\linewidth]{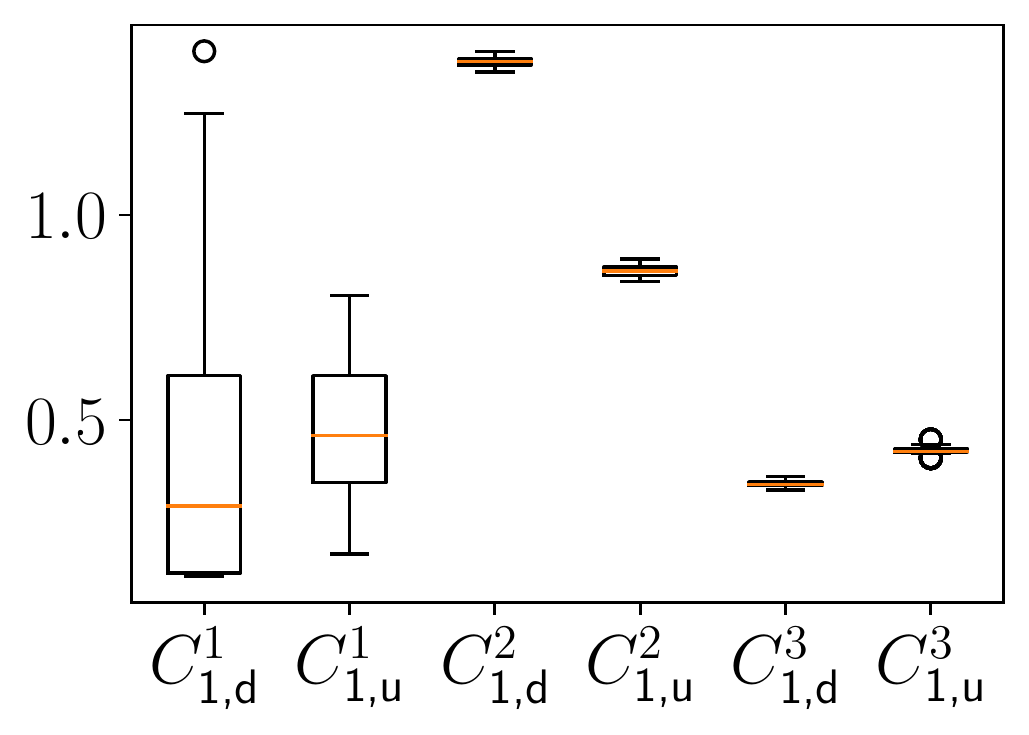}
    \caption{SCNN55, with $r_{\rm{IL}}$}
    \label{fig:il_constant_scnn5_il}
  \end{subfigure}
  \caption{The integral Lipschitz constants of SCFs at each layer of the trained SCNNs with and without the integral Lipschitz regularizer $r_{\rm{IL}}$. We use symbols $C_{k,\rmd}^l$ and $C_{k,\rmu}^l$ to denote the lower and upper integral Lipschitz constants at layer $l$. Regularizer $r_{\rm{IL}}$ promotes the integral Lipschitz property, thus, the stability, especially for NNs with large SCF orders.}
  \label{fig:il_constants_trajectory_prediction}
\end{figure}



\end{document}